\documentclass[lettersize,journal]{IEEEtran}

\usepackage{amsmath}
\usepackage{dsfont}
\usepackage{cite}
\usepackage[pdftex]{graphicx}
\usepackage{amsmath, amsfonts, amssymb, amsthm}
\usepackage{array}
\usepackage[switch]{lineno}
\usepackage{bm}
\usepackage{color}
\usepackage[dvipsnames]{xcolor}
\usepackage{mathrsfs}
\usepackage{epstopdf}
\usepackage{bbding}
\usepackage{pifont}
\definecolor{colorhkust}{RGB}{20,43,140}
\definecolor{colortsinghua}{RGB}{116,52,129}
\definecolor{colorsht}{RGB}{163,11,25}
\definecolor{cha}{HTML}{44a920}
\usepackage[labelformat=simple]{subcaption}
\usepackage{hyperref}
\hypersetup{hypertex=true,colorlinks=true,linkcolor=colorsht,
anchorcolor=colorsht,
citecolor=colorsht,
urlcolor=colorsht}
\usepackage[ruled,linesnumbered]{algorithm2e}
\usepackage{multirow} 
\usepackage{multicol} 
\usepackage{booktabs}
\usepackage{diagbox}
\usepackage{svg}
\newtheorem{thm}{Theorem}
\newtheorem{lem}{Lemma}
\newtheorem{defn}{Definition}
\newtheorem{ass}{Assumption}
\newtheorem{cor}{Corollary}
\newtheorem{Rem}{Remark}
\newcommand{\fedavg}{\textsc{FedAvg}\xspace}
\newcommand{\fedsgd}{\textsc{FedSGD}\xspace}
\newcommand{\airfedavg}{\textsc{AirFedAvg}\xspace}
\newcommand{\airfedavgm}{\textsc{AirFedAvg-M}\xspace}
\newcommand{\airfedavgs}{\textsc{AirFedAvg-S}\xspace}
\newcommand{\airfedmodel}{\textsc{AirFedModel}\xspace}
\begin{document}

\title{Over-the-Air Federated Learning and Optimization}

\author{
                Jingyang~Zhu,~\textit{Graduate Student Member, IEEE},~Yuanming~Shi,~\textit{Senior Member, IEEE},\\~Yong~Zhou,~\textit{Senior Member, IEEE},~Chunxiao~Jiang,~\textit{Senior Member, IEEE},~Wei~Chen,~\textit{Senior Member, IEEE}, and~Khaled~B.~Letaief,~\textit{Fellow, IEEE}
\thanks{Jingyang Zhu, Yuanming Shi, and Yong Zhou are with the School of Information Science and Technology, ShanghaiTech University, Shanghai 201210, China (e-mail: $\{$zhujy2, shiym, zhouyong$\}$@shanghaitech.edu.cn). 
	
			Chunxiao Jiang is with the Tsinghua Space Center, Tsinghua University, Beijing,
			100084, China. (e-mail: jchx@tsinghua.edu.cn).
			
			Wei Chen is with the Department of Electronic Engineering and the Beijing National Research Center for Information Science and Technology, Tsinghua University, Beijing 100084, China (e-mail: wchen@tsinghua.edu.cn).

			Khaled B. Letaief is with the Department of Electronic and Computer Engineering, Hong Kong University of Science and Technology, Clear Water Bay, Hong Kong, and also with Peng Cheng Laboratory, Shenzhen 518066, China (e-mail: eekhaled@ust.hk).}
}

\maketitle

\begin{abstract}
	Federated learning (FL), as an emerging distributed machine learning paradigm, allows a mass of edge devices to collaboratively train a global model while preserving privacy.
	In this tutorial, we focus on FL via over-the-air computation (AirComp), which is proposed to reduce the communication overhead for FL over wireless networks at the cost of compromising in the learning performance due to model aggregation error arising from channel fading and noise. 
	We first provide a comprehensive study on the convergence of AirComp-based \fedavg (\airfedavg) algorithms under both strongly convex and non-convex settings with constant and diminishing learning rates in the presence of data heterogeneity.
	Through convergence and asymptotic analysis, we characterize the impact of aggregation error on the convergence bound and provide insights for system design with convergence guarantees. Then we derive convergence rates for \airfedavg algorithms for strongly convex and non-convex objectives.
	For different types of local updates that can be transmitted by edge devices (i.e., local model, gradient, and model difference), we reveal that transmitting local model in \airfedavg may cause divergence in the training procedure.
	In addition, we consider more practical signal processing schemes to improve the communication efficiency and further extend the convergence analysis to different forms of model aggregation error caused by these signal processing schemes.
	Extensive simulation results under different settings of objective functions, transmitted local information, and communication schemes verify the theoretical conclusions.
\end{abstract}
%
\section{Introduction}
With the ever-growing communication, computation, and caching capabilities of various edge devices (e.g., smart phones, sensors, and Internet of Things (IoT) devices), the sixth generation (6G) wireless communication system 
is undergoing a paradigm transformation from connected everything to connected intelligence, thereby envisioning to support ubiquitous artificial intelligence (AI) services and applications (e.g., immersive extended reality, holographic communication, sensory interconnectivity, and digital twins) \cite{Theroadmap,CEEAI,EdgeAI6G}. A large amount of data collected and generated by edge devices together with their sensing, computation, and communication capabilities have opened up bright avenues for training machine learning (ML) models at the network edge \cite{shi2021mobile}. 
However, in conventional distributed ML for cloud AI, cloud servers need to gather raw data from all devices, resulting in excessive communication overhead and severe privacy concern.
To this end, federated edge learning (FL) \cite{FedAvg,bonawitz2019towards,park2019wireless}, as an emerging distributed ML framework, has been proposed to enable many edge devices to train an ML model under the coordination of an edge server without sharing their local data.
Although promising, supporting FL over wireless networks still encounters multi-faceted determining factors, including transmission schemes, networking, and resource optimization \cite{li2020federated}.

%
To mitigate the communication bottleneck problem, it is essential to design efficient model transmission schemes. The existing studies can be divided into two categories, i.e., digital modulation based FL (digital FL) and over-the-air computation (AirComp)-based FL (AirFL).
The existing studies on digital FL \cite{wang2019adaptive,JointDevice,AJoint,ConvergenceTime,ren2021accelerating,zhang2022communication,EnergyEfficient,luo2020cost,amiria2021Convergence,lim2021dynamic,liu2022joint,wen2022joint,liu2022hierarchical} typically  
adopted the orthogonal multiple access (OMA) scheme to transmit local models, where the frequency and time blocks are orthogonally divided.
In this case, the transmission latency of digital FL will be proportional to the number of edge devices \cite{RN35}, leading to high transmission delay of FL over large-scale wireless networks.
To overcome this challenge, AirComp, as a promising task-oriented scheme that leverages the waveform superposition property of wireless multiple-access channel (MAC)\cite{shi2023task}, was proposed to enable concurrent model transmission over the same time-frequency channel \cite{nazer2007computation}, thus considerably reducing the communication delay \cite{zhu2021aircomp,wang2022over,abari2016over}. Motivated by this advantage, AirFL was proposed to achieve fast analog model aggregation \cite{RN34}, where multiple edge devices transmit their local information concurrently, and the edge server directly obtains the aggregated local information from the superimposed signal. 
Substantial progress has recently been made in the area of AirFL \cite{RN34,xing2021D2D,shi2021over,lin2022distributed,michelusi2022decentralized,wang2022inference,aygun2022over,shao2022bayesian,liu2021reconfigurable,wang2022IRS,yang2022differentially,ni2022multiris,hu2022ris,zhong2022uav,wang2022leo,kim2023beamforming,statisticsaware,RN86,liu2020privacy,xiaowen2021optimized,guo2023dynamic,xiaowen2021transmission,guo2022Joint,zou2023knowledge,guo2021analog,COTAF,gafni2023federated,xu2021LR,wang2022edge,RN37,zhu2020one,wei2021federated,jing2022Federated,RN36,amiri2020federated,sun2021dynamic,su2021data,du2022gradient,fan2021temporal,sifaou2021robust,lin2022relay,RN35,elgabli2021harnessing,Mohamed2021privacy,shao2021misaligned,li2022energy,you2022broadband}, including tackling the learning performance degradation caused by the channel fading and the receiver noise in each model aggregation.

The underlying network architecture is an essential factor that affects both the communication efficiency and the convergence rate of FL systems.
From the perspective of network architectures, existing studies on digital FL cover client-server FL \cite{wang2019adaptive,JointDevice,AJoint,ConvergenceTime,ren2021accelerating,zhang2022communication,EnergyEfficient,luo2020cost,amiria2021Convergence} and hierarchical FL \cite{lim2021dynamic,liu2022joint,wen2022joint,liu2022hierarchical}. 
In addition to the client-server network architecture, AirFL has also been studied in diverse network architectures.
To fully utilize the diversified communication links and geographically dispersed computing resources, device-edge-cloud collaboration has been investigated from a variety of perspectives such as decentralized AirFL by exploiting device-to-device links \cite{xing2021D2D,shi2021over,lin2022distributed,michelusi2022decentralized}, multi-cell AirFL with inter-cell interference management \cite{wang2022inference}, and hierarchical AirFL for alleviating communication costs and accelerating convergence \cite{aygun2022over,shao2022bayesian}.
Besides, space-air-ground integration (e.g., reconfigurable
intelligent surface (RIS)-assisted AirFL \cite{liu2021reconfigurable,wang2022IRS,yang2022differentially,ni2022multiris,hu2022ris}, unmanned aerial vehicle (UAV)-assisted AirFL \cite{zhong2022uav}, and satellite-assisted AirFL \cite{wang2022leo}) provides an integrated information platform that can serve different learning algorithms and topologies in order to establish seamless communication and services between ground facilities and the airspace.

To support communication-efficient FL with limited radio resources, wireless resource management and optimization is essential.
In particular, there have been comprehensive studies on digital FL focusing on communication system design and wireless resource allocation, covering different types of objective functions (e.g., training loss function \cite{wang2019adaptive,JointDevice,AJoint}, latency \cite{ConvergenceTime,ren2021accelerating,zhang2022communication}, and energy consumption \cite{EnergyEfficient,luo2020cost}).
Furthermore, the literature investigating resource optimization in AirFL include transceiver design \cite{RN34,kim2023beamforming,statisticsaware,RN86,liu2020privacy,xiaowen2021optimized,guo2023dynamic,zou2023knowledge,wang2022edge,COTAF,xiaowen2021transmission,guo2022Joint,guo2021analog,xu2021LR} (e.g., beamforming design \cite{RN34,kim2023beamforming}, power control \cite{statisticsaware,RN86,liu2020privacy,xiaowen2021optimized,guo2023dynamic,xiaowen2021transmission,guo2022Joint,zou2023knowledge}, precoding and denoising factor design \cite{COTAF,gafni2023federated}, and hyperparameter optimization \cite{guo2021analog,xu2021LR}), device selection \cite{RN34,kim2023beamforming,sun2021dynamic,wang2022IRS,su2021data,guo2022Joint,du2022gradient}, and privacy protection by utilizing
the inherent channel noise \cite{liu2020privacy,elgabli2021harnessing,Mohamed2021privacy}.
To address the practical challenges for implementing AirFL, some attempts focus on developing synchronization techniques \cite{shao2021misaligned} and digital modulation schemes for deployment in off-the-shelf communication systems (e.g., LTE, Wi-Fi 6, and 5G)\cite{li2022energy,you2022broadband}.

As the communication system design for AirFL highly depends on characterizing the impact of the channel fading and receiver noise on the convergence of the AirFL algorithms, it is important to conduct explicit convergence analysis.
The key factors that affect the convergence analysis can be summarized as follows: the impact of the receiver noise on the convergence bound, the type of local updates (i.e., local model, local gradient, or model difference), the property of the objective functions (e.g., strongly convex or non-convex), the choice of the learning rates (i.e., diminishing or constant), the characterization of data heterogeneity, and the choice of the number of local updates.
These issues arising from the communication systems and learning algorithms are strongly coupled in AirFL.
For instance, the choice of the learning rate relies on the objective functions and the number of local updates, and it also affects the optimality of the objective function.
As another example, in the error-free case, transmitting different types of local updates leads to different model aggregation steps, which are mathematically equivalent. However, such an equivalence is no longer valid in AirFL. The inequivalence is clearly demonstrated in the following two aspects. First, due to different numbers of local updates, transmitting different
types of local updates may lead to different convergence rates and robustness to data heterogeneity, which need to be thoroughly investigated and discussed. Second, different types of transmit signals together
with the corresponding model aggregation steps have different robustness to the receiver noise, affecting the convergence performance as well.
Hence, the coupling between the number of local updates and the type of transmit signals presents a new challenge for designing AirFL algorithms and analyzing their convergence, which needs to be revealed.\\

Despite prior works have achieved encouraging results of analyzing AirComp-based federated stochastic gradient descent (\fedsgd) and federated averaging (\fedavg) \cite{FedAvg} algorithms, a comprehensive understanding of these issues is still lacking, which motivates this comprehensive tutorial on the convergence of AirFL.

\subsection{Main Contributions}
In this paper, we provide a comprehensive understanding of AirComp-based \fedavg (\airfedavg) algorithms.
We first provide convergence analysis with respect to the model aggregation error caused by receiver noise for benchmark algorithm $\textsc{AirFedAvg-Multiple }$ (\airfedavgm, i.e., the model difference is transmitted by an edge device after several local updates in \airfedavg) with a diminishing learning rate in the strongly convex settings. Then we provide asymptotic analysis for the upper bound to characterize the impact of noise by the mean square error (MSE) of the model aggregation, and further arrive at some insightful conclusions for \airfedavgm by the case study of denoising factor. Following the same principle, we extend the asymptotic analysis and case study to the non-convex settings.
We then extend our analysis to variant algorithms: \textsc{AirFedAvg-Single} (\airfedavgs, i.e., edge devices transmit their local gradients after a single local update) and \airfedmodel (i.e., edge devices transmit their local models).
We also figure out the influence of the number of local updates by analyzing the convergence results for \airfedavgm and make comparisons between \airfedavgm and \airfedavgs.
In summary, from a systematic perspective, we provide a detailed taxonomy of the \airfedavg algorithms based on the type of model aggregation error (i.e., unbiased and biased), the transmitted signal (i.e., model difference, local gradient, or local model), the number of local updates (i.e., multiple or single), objective functions (i.e., strongly convex or non-convex), and learning rates (i.e., diminishing or constant), as shown in Fig. \ref{fig: tax}.
Our main contributions can be summarized as follows.
\begin{itemize}
	\item We present a novel convergence analysis framework for \airfedavg with respect to the model aggregation error at the edge server, regardless of whether the transmitted signal can be unbiased estimated by the edge server or not. We first provide convergence and asymptotic analysis for \airfedavgm to characterize the impact of the MSE and then apply the analysis to \airfedavgs and \airfedmodel.
	\item In terms of the type of local updates, we reveal that the convergence of the \airfedavgm and \airfedavgs algorithms (i.e., transmitting model difference or local gradient) can be guaranteed, while transmitting local model may cause the divergence in the training procedure based on the convergence results.
	\item To elaborate the convergence results of the learning algorithms (i.e., \airfedavgm and \airfedavgs) in different scenarios, we provide convergence results for both strongly convex and non-convex objective function under non-IID dataset with different learning rates. The details of the convergence results are listed in Table \ref{tab: scope}.
	\item We provide a case study and demonstrate the optimality gap for strongly convex case and error bound for non-convex case, and characterize the impact of the signal-to-noise ratio (SNR) on the convergence rate. We also derive the maximum number of local updates required for \airfedavgm to preserve convergence for both cases and verify that \airfedavgm can achieve linear speedup with respect to the number of local updates and the number of edge devices.	
\end{itemize}

We conduct substantial simulations to evaluate the performance of the proposed \airfedavgm, \airfedavgs, and \airfedmodel under different system designs, i.e., precoding factors and SNRs. Results confirm that assisted by AirComp, transmitting local model is not a good choice. Besides, simulations also indicate that \airfedavgm can achieve linear speedup, but may suffer from highly non-IID data, and \airfedavgs is more robust to the channel noise and non-IID data despite linear slower convergence compared with \airfedavgm.
\begin{figure*}[tbp]
	\centering
	\includegraphics[width=1\linewidth]{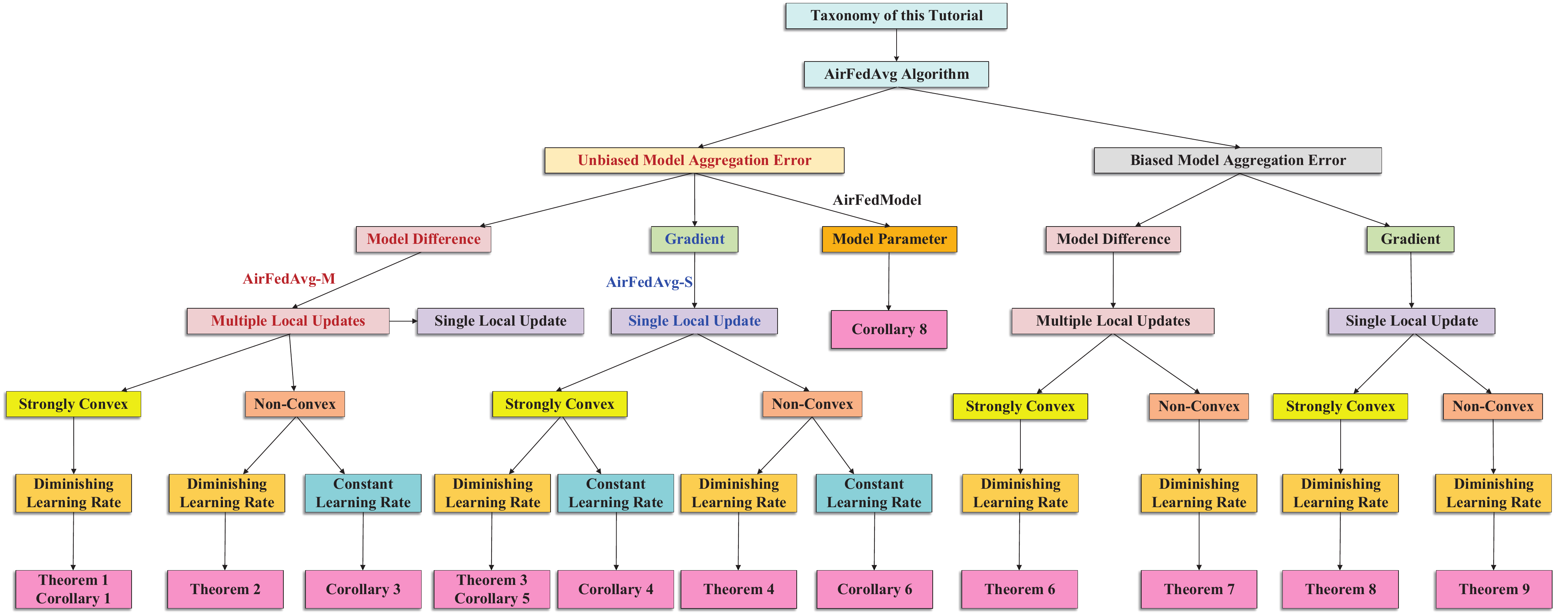}
	\caption{Taxonomy of this tutorial.}
	\label{fig: tax}
\end{figure*}

\begin{table*}[!]	
	\caption{Main Convergence Results of \airfedavg Algorithms in This Tutorial, from the Perspectives of Strongly Convex (SC) and Non-Convex (NC) Objectives, Local Information (LI) to Transmit, Model Aggregation Error (MAE), Local Update (LU), Learning Rate (LR), Convergence Rate (CR), and Dominant Variable (DV).}\label{tab: scope}
	\centering\renewcommand\arraystretch{1.2}
	\begin{tabular}{c|c|cccccccc}
		\hline
		Location&Result &Algorithm & Objective & LI&MAE &LU& LR & CR &DV\\ \hline
		\multirow{4}{*}{Section \ref{sec: convm}}&
		\textbf{Theorem \ref{thm: case3unbiasedconvex}}&\airfedavgm&SC & Difference&Unbiased&$E$&Diminishing&-&MSE\\
		&\textbf{Corollary \ref{cor: case3convex}}&\airfedavgm&SC & Difference&Unbiased&$E$&$\mathcal{O}(1/t)$&$\mathcal{O}(\frac{1}{NET})$&SNR\\
		&\textbf{Theorem \ref{thm: case3unbiasednonconvex}}&\airfedavgm&NC & Difference&Unbiased&$E$&Diminishing&-&MSE\\
		&\textbf{Corollary \ref{cor: case3}}&\airfedavgm&NC & Difference&Unbiased&$E$&$\mathcal{O}(\sqrt{\frac{N}{ET}})$&$\mathcal{O}(\frac{1}{\sqrt{NET}})$&SNR\\
		\hline
		\multirow{6}{*}{Section \ref{sec: var}}&
		\textbf{Theorem \ref{thm: case1unbiasedconvex}}&\airfedavgs&SC & Gradient&Unbiased&$1$&Diminishing&-&MSE\\
		&\textbf{Corollary \ref{cor: case1constant}}&\airfedavgs&SC & Gradient&Unbiased&$1$&Constant& $\mathcal{O}((1-\frac{\mu}{L})^T)$&MSE\\
		&\textbf{Corollary \ref{cor: case1convex}}&\airfedavgs&SC & Gradient&Unbiased&$1$&$\mathcal{O}(1/t)$&$\mathcal{O}(\frac{1}{NT})$&SNR\\
		&\textbf{Theorem \ref{thm: case1unbiasednonconvex}}&\airfedavgs&NC & Gradient&Unbiased&$1$&Diminishing&-&MSE\\
		&\textbf{Corollary \ref{cor: case1}}&\airfedavgs&NC & Gradient&Unbiased&$1$&$\mathcal{O}(\sqrt{\frac{N}{T}})$&$\mathcal{O}(\frac{1}{\sqrt{NT}})$&SNR\\
		&\textbf{Corollary \ref{cor: case2unbiasedconvex}}&\airfedmodel&SC & Model&Unbiased&$1$&Diminishing&-&MSE\\
		\hline
		\multirow{4}{*}{Section \ref{sec: extensions}}&
		\textbf{Theorem \ref{thm: case3biasedconvex}}&\airfedavgm&SC & Difference&Biased&$E$&Diminishing&-&MSE, Bias\\
		&\textbf{Theorem \ref{thm: case3biasednonconvex}}&\airfedavgm&NC & Difference&Biased&$E$&Diminishing&-&MSE, Bias\\
		&\textbf{Theorem \ref{thm: case1biasedconvex}}&\airfedavgs&SC & Gradient&Biased&$1$&Diminishing&-&Bias\\
		&\textbf{Theorem \ref{thm: case1biasednonconvex}}&\airfedavgs&NC & Gradient&Biased&$1$&Diminishing&-&Bias\\
		\hline
	\end{tabular}
\end{table*}

\subsection{Organization and Notation}
The rest of this paper is organized as follows. 
Section \ref{sec: background} provides some preliminaries for learning and optimization theory. 
Section \ref{sec: system model} describes the AirFL system model and \airfedavgm algorithm. Convergence analysis, case study and insights of \airfedavgm are given in Section \ref{sec: convm}. The convergence results of different variants, i.e., \airfedavgs and \airfedmodel, are provided in Section \ref{sec: var}. Further discusses are given in \ref{sec: extensions}. Numerical results are presented in Section \ref{sec: exp} to verify the theoretical findings. In Section \ref{sec: conclusions}, we draw our conclusions.
\\\textbf{Notation:}
The operators $\left\|\cdot \right\|_2$, $(\cdot)^{\sf{T}}$, $(\cdot)^{\sf{H}}$, $\left\langle\cdot\right\rangle$, and $\left|\cdot\right|$ denote $\ell_2$-norm, transpose, Hermitian, inner product, and modulus, respectively.
$\mathbb{E}[\cdot]$ denotes the expectation over a random variable. For a function $f$, $\nabla f$ is its gradient. $[T]$ denotes the set $\{0,1,\dots ,T-1\}$. $\mathcal{O}(\cdot)$ is the big-O notation, which stands for the order of arithmetic operations. The notations $N$, $E$, $T$, $t$, $\mu$, and $L$ denote the number of edge devices, the number of local updates, the number of communication rounds, each communication round, $\mu$ strong convexity, and $L$-smoothness, respectively.

\section{Preliminaries}\label{sec: background}
Before diving into the discussion on FL, we start by introducing some preliminaries of learning and optimization theory \cite{ghadimi2013stochastic,bottou2018optimization} to facilitate better understanding of our analysis. 
\subsection{Learning and Optimization Theory}
We consider the following general empirical risk minimization problem
\begin{equation}\label{eq: background}
	\underset{\bm{x}\in \mathbb{R}^d }{\text { minimize }} ~ f(\bm{x}):=
	\frac{1}{m}\sum\limits_{i=1}^{m} f_{i}\left(\bm x\right), 
\end{equation}
where $d$ is the dimension of vector $\bm{x}$, the objective function $f: \mathbb{R}^d\rightarrow \mathbb{R}$ is continuously differentiable, $m$ is the number of data samples, and $f_i: \mathbb{R}^d\rightarrow \mathbb{R}, i = 0,1,\dots,m$ is the sample-wise loss function.
The optimal solution to problem \eqref{eq: background} is denoted by $ 	\bm x^*:=\operatorname{arg} \underset{\bm{x}\in \mathbb{R}^d}{\min}~f(\bm x)$. 

\subsubsection{Objective Function}
For the objective function, we have the following definition and assumption.
\begin{defn}[$L$-Smoothness] \label{def: L-smooth}
	The objective function $f: \mathbb{R}^d\rightarrow \mathbb{R}$ is continuously differentiable and its gradient, i.e., $\nabla f: \mathbb{R}^d\rightarrow \mathbb{R}^d$, is Lipschitz continuous with constant $L>0$ if the following inequality holds
	\begin{equation}\label{eq: L-smooth}
		\left \|\nabla   f(\bm x)- \nabla f(\bm y)\right\|_2\leq L\left \|\bm x- \bm y\right\|_2,~\forall \bm x, \bm y \in \mathbb{R}^d.
	\end{equation}
	We can equivalently rewrite \eqref{eq: L-smooth} as
	\begin{equation*}
		f(\bm{x}) \leq f(\bm{y}) + (\bm{x} - \bm{y})^{\sf T} \nabla f(\bm{y}) + \frac{L}{2} \| \bm{x} - \bm{y} \|_2^2,~\forall \bm x, \bm y \in \mathbb{R}^d.
	\end{equation*}
\end{defn}
\begin{ass}
	The objective function $f$ is lower bounded by a scalar $f^{\inf}$, i.e., $f(\bm{x})\geq f^{\inf}>-\infty$.
\end{ass}

To solve problem \eqref{eq: background}, the dominant methodology is the vanilla stochastic gradient descent (SGD), which applies the following iterative update rule
\begin{equation}\label{eq: sgd}
	\bm{x}^{t+1} = \bm{x}^{t}- \eta^t \bm{g}(\bm{x}^{t},\bm{\xi}^t),
\end{equation}
where $\eta^t$ is the learning rate or step size in iteration $t$, $\bm{\xi}^t$ is a random variable indicating the selection of data samples, and $ \bm{g}(\bm{x}^{t},\bm{\xi}^t) $ is the output of the stochastic first-order oracle that is a noisy version of gradient $\nabla f(\bm{x}^t)$. The following assumption is made for the stochastic gradient estimation.
\begin{ass}\label{ass: sgd}
	The stochastic gradient estimation $ \bm{g}(\bm{x}^{t},\bm{\xi}^t) $ is unbiased and has a bounded variance in each iteration: 
	\begin{align*}
		&\mathbb{E}\left[ \bm{g}(\bm{x}^{t},\bm{\xi}^t)\right ] = \nabla f(\bm{x}^{t}),\\
		&\mathbb{E} \left[ \| \bm{g}(\bm{x}^{t},\bm{\xi}^t) - \nabla f(\bm{x}^{t}) \|_2^2 \right] \leq \sigma^2,
	\end{align*}
where constant parameter $\sigma^2\geq0$ and the expectation is taken with respect to the distribution of random variable $\bm{\xi}^t$.
\end{ass}

If the objective function is strongly convex, then we have the following definition.
\begin{defn}[Strong Convexity] \label{def: sc}
	The objective function $f: \mathbb{R}^d\rightarrow \mathbb{R}$ is strongly convex if there exists a constant $\mu>0$ such that
	\begin{equation*}
		f(\bm{x}) \geq f(\bm{y}) + (\bm{x} - \bm{y})^{\sf T} \nabla f(\bm{y}) + \frac{\mu}{2} \| \bm{x} - \bm{y} \|_2^2,~\forall \bm x, \bm y \in \mathbb{R}^d.
	\end{equation*}
\end{defn}


For strongly convex objectives, the convergence performance can be evaluated by the optimality gap, which is defined as
\begin{equation*}
		{\sf Gap}^t := f(\bm{x}^t) - f(\bm{x}^{\ast})
\end{equation*}
in iteration $t$. To attain a given accuracy $\epsilon$ within $T$ iterations, the optimality gap is expected to satisfy $\mathbb{E}[{\sf Gap}^T]\leq\epsilon$, where the expectation is taken over $\bm{\xi}^t,~t= 0,1,\dots,T$.
However, to train deep neural networks (DNN) and convolutional neural networks (CNN), the objective function $f$ is likely to be smooth but non-convex.
For non-convex objectives, it is common to consider an upper bound on the weighted average norm of the gradient of the objective function, i.e., 
$$\frac{1}{\sum_{t = 0}^{T-1} \eta^t}\sum_{t = 0}^{T-1} \eta^t\mathbb{E} \left[\left \|\nabla f(\bm{x}^{t})\right \|_2^2\right],$$
for the first $T$ iterations. If the learning rate is a constant, i.e., $\eta^t = \eta,\forall t$, then it reduces to
$\frac{1}{T}\sum_{t = 0}^{T-1} \mathbb{E} \left[\left \|\nabla f(\bm{x}^{t})\right \|_2^2\right].$ Furthermore, we say that solution $\bm{x}^t$ is an $\epsilon$-stationary point if $\left \|\nabla f(\bm{x}^{t})\right \|_2^2\leq\epsilon$.

In addition, some objective functions meet the Polyak-Lojasiewicz (PL) condition, which is defined as follows.
\begin{defn}[PL Condition]\label{def: PL}
	The objective function $f: \mathbb{R}^d\rightarrow \mathbb{R}$ satisfies the PL condition with constant $\mu$, i.e.,
	\begin{align*}\label{eq: PL}
		\|\nabla  f(\bm x^t)\|_2^2\geq 2\mu [f(\bm x^t)-f(\bm x^*)].
	\end{align*}
\end{defn}
Note that the PL condition is weaker than strong convexity, and usually considered in the analysis of non-convex cases.
\subsubsection{Convergence Properties under Different Learning Rates}
To solve problem \eqref{eq: background} with SGD, the choice of learning rate is crucial for the convergence gap and convergence rate\footnote{In this tutorial, convergence gap refers to the optimality gap in the strongly convex case and the upper bound on the weighted average norm of the gradient of the objective function in the non-convex case. Convergence rate is a measure of how fast the difference between the objective function value of the solution and its estimates goes to zero.}. 
\begin{table}[!]	
	\caption{Convergence Results for Strongly Convex (SC) and Non-convex (NC) Objectives}\label{tab: convergence1}
	\centering\renewcommand\arraystretch{1.2}
	\begin{tabular}{cccc}
		\hline
		& Learning Rate & Convergence Gap & Convergence Rate\\ 
		\hline \multirow{2}{*}{SC} & $\mathcal{O}(1)$ & non-diminishing & $\mathcal{O}(\rho^T),\rho\in(0,1)$\\
		& $\mathcal{O}(1/t)$ & diminishing &$\mathcal{O}(1/T)$\\
		\hline \multirow{2}{*}{NC} & $\mathcal{O}(1/\sqrt{T})$ & non-diminishing & $\mathcal{O}(1/\sqrt{T})$\\
		& $\mathcal{O}(1/t)$ & diminishing &-\\\hline
	\end{tabular}
\end{table}
\begin{table*}[!]	
	\caption{Overview of Local SGD and \fedavg Algorithms}\label{tab: convergence2}
	\centering\renewcommand\arraystretch{1.2}
	\begin{tabular}{c|cccccc}
		\hline
		Algorithm & Objective & BG & BGD &Local Update& Learning Rate & Convergence Rate\\ \hline	
		\fedavg \cite{Li2020On}&SC & \textcolor{colorsht}{\checkmark}&\textcolor{cha}{\ding{55}}&$\mathcal{O}(1)$&$\mathcal{O}(1/t)$&$\mathcal{O}(1/NT)$\\
		Local SGD \cite{stich2018local}&SC & \textcolor{colorsht}{\checkmark}&\textcolor{cha}{\ding{55}}&$\mathcal{O}(T^{1/2}N^{-1/2})$&$\mathcal{O}(1/t)$&$\mathcal{O}(1/NT)$\\
		
		\hline
		\fedavg \cite{haddadpour2019convergence}&NC+PL & \textcolor{cha}{\ding{55}}&\textcolor{colorsht}{\checkmark}&$\mathcal{O}(T^{2/3}N^{-1/3})$&$\mathcal{O}(1/t)$&$\mathcal{O}(1/NT)$\\
		Local SGD \cite{Li2020On}&NC+PL & \textcolor{cha}{\ding{55}}&\textcolor{cha}{\ding{55}}&$\mathcal{O}(T^{2/3}N^{-1/3})$&$\mathcal{O}(1/t)$&$\mathcal{O}(1/NT)$\\
		\hline
		\fedavg \cite{haddadpour2019convergence}&NC& \textcolor{cha}{\ding{55}}&\textcolor{colorsht}{\checkmark}&$\mathcal{O}(T^{1/2}N^{-3/2})$&$\mathcal{O}(N^{1/2}T^{-1/2})$&$\mathcal{O}(1/\sqrt{NT})$\\
		Local SGD \cite{yu2019parallel}&NC & \textcolor{colorsht}{\checkmark}&\textcolor{cha}{\ding{55}}&$\mathcal{O}(T^{1/4}N^{-3/4})$&$\mathcal{O}(N^{1/2}T^{-1/2})$&$\mathcal{O}(1/\sqrt{NT})$\\
		Local SGD \cite{wang2021cooperative}&NC & \textcolor{cha}{\ding{55}}&\textcolor{cha}{\ding{55}}&$\mathcal{O}(T^{1/2}N^{-3/2})$&$\mathcal{O}(N^{1/2}T^{-1/2})$&$\mathcal{O}(1/\sqrt{NT})$\\
		\multirow{2}{*}{Local SGD with Momentum\cite{yu2019on}}&NC & \textcolor{cha}{\ding{55}}&\textcolor{cha}{\ding{55}}&$\mathcal{O}(T^{1/2}N^{-3/2})$&$\mathcal{O}(N^{1/2}T^{-1/2})$&$\mathcal{O}(1/\sqrt{NT})$\\
		&NC &\textcolor{cha}{\ding{55}}&\textcolor{colorsht}{\checkmark}&$\mathcal{O}(T^{1/4}N^{-3/4})$&$\mathcal{O}(N^{1/2}T^{-1/2})$&$\mathcal{O}(1/\sqrt{NT})$\\
		VRL-SGD \cite{liang2019variance}&NC & \textcolor{cha}{\ding{55}}&\textcolor{cha}{\ding{55}}&$\mathcal{O}(T^{1/2}N^{-3/2})$&$\mathcal{O}(N^{1/2}T^{-1/2})$&$\mathcal{O}(1/\sqrt{NT})$\\
		\hline
	\end{tabular}
\end{table*}
To be specific, for strongly convex objectives with a constant learning rate, the expected objective function linearly converges to the neighborhood of the optimal value, but with a non-diminishing optimality gap due to the noisy gradient estimation \cite{bottou2018optimization}. In contrast, using a decaying learning rate leads to a diminishing optimality gap at the cost of achieving a sublinear convergence rate\footnote{If an algorithm has a convergence rate of $ \mathcal{O}(1/T) $, then it will take $\mathcal{O}(1/\epsilon)$ iterations to achieve accuracy $\epsilon$.}, where the learning rate satisfies the following conditions:
\begin{equation}\label{eq: learning rate}
	\sum_{t = 0}^{\infty}\eta^t=\infty,~\sum_{t = 0}^{\infty}(\eta^t)^2<\infty.
\end{equation}
Similarly, for non-convex objectives with a constant learning rate, the error bound on the average of squared gradients is non-diminishing as well \cite{bottou2018optimization}. In addition, using a diminishing learning rate leads to the following result:
\begin{equation}\label{eq: non-decay}
	\underset{t\rightarrow\infty}{\lim\inf}~\mathbb{E} \left[\left \|\nabla f(\bm{x}^{t})\right \|_2^2\right] = 0.
\end{equation}
More precisely, \eqref{eq: non-decay} can be written as
\begin{equation}\label{eq: non-decay2}
	\frac{1}{\sum_{t = 0}^{T-1} \eta^t}\sum_{t = 0}^{T-1} \eta^t\mathbb{E} \left[\left \|\nabla f(\bm{x}^{t})\right \|_2^2\right]\stackrel{T\rightarrow\infty}{\longrightarrow} 0.
\end{equation}
Table \ref{tab: convergence1} gives more detailed information about the above statements.

\subsection{Federated Learning}
With the capability of enhancing the data privacy and reducing the communication overhead, FL becomes a research hotspot, where $N$ workers collaboratively train a global model without sharing private data. The objective function can be written as
\begin{equation}\label{eq: fl}
	\underset{\bm{x}\in \mathbb{R}^d }{\text { minimize }} ~ F(\bm{x}):=
	\frac{1}{N}\sum\limits_{n=1}^{N} F_{n}\left(\bm x\right),
\end{equation}
where $F_n$ of each worker $n$ can be viewed as $f$ in \eqref{eq: background}. To solve this problem, various algorithms, including vanilla distributed SGD, paralleled SGD, Local SGD, and \fedavg algorithms, have been proposed in \cite{FedAvg,zhou2018convergence,stich2018local,haddadpour2019trading,yu2019parallel,haddadpour2019convergence,yu2019on,liang2019variance,wang2021cooperative,khaled2020tighter,karimireddy2020scaffold,Li2020On,gorbunov2021local} to achieve linear speed-up in model training with respect to the number of workers.
Specifically, Local SGD is fully-synchronized, where the model averaging step, i.e., taking an average over all workers' local model parameters, is performed once each worker finishes one local update \cite{wang2021cooperative}. To reduce the communication overhead, Local SGD with periodic averaging and \fedavg were proposed, where the frequency of the model averaging step can be chosen \cite{stich2018local}. 
The difference in nomenclature is that Local SGD usually refers to distributed settings with homogeneous data, while \fedavg is  used in FL with heterogeneous data \cite{charles2020outsized}. 
The detailed comparison of these algorithms is listed in Table \ref{tab: convergence2} with two extra assumptions on the gradient, i.e., bounded gradient (BG), and bounded gradient dissimilarity (BGD), which shall be defined in Section \ref{sec: convm}.
%
Note that BGD can be used to measure the deviation of local gradients after multiple local updates as well as the non-IID extent of data at different workers in FL \cite{karimireddy2020scaffold,zhang2020fedpd}. 

The convergence rates of Local SGD, FedAvg and other related algorithms under different assumptions and settings are summarized in Table \ref{tab: convergence2}. 
It can be observed that all studies considering strongly convex objectives (or PL condition) use learning rate decay to eliminate the influence of noisy gradient estimation and achieve a similar convergence rate $\mathcal{O}(1/NT)$. For non-convex objectives, they adopt a similar constant learning rate $\mathcal{O}(\sqrt{N/T})$ and achieve a sublinear convergence rate $\mathcal{O}(1/\sqrt{NT})$. Besides, decentralized SGD algorithms were also studied in \cite{chen2021accelerating,koloskova2020unified}.

All the algorithms listed in Table \ref{tab: convergence2} aggregate local models to update the global model.
In fact, the local information to be aggregated for global model update at the server can be classified into the following three categories.
\begin{itemize}
	\item \textbf{Local Model:} In this case, each worker directly transmits its local model (e.g., $\bm{x}^t$ in \eqref{eq: sgd}) to the central server for global model update \cite{stich2018local,haddadpour2019trading,yu2019parallel,haddadpour2019convergence,yu2019on,liang2019variance,wang2021cooperative,khaled2020tighter,karimireddy2020scaffold,Li2020On,gorbunov2021local,guo2022hybrid}.
	\item \textbf{Model Difference:} Model difference refers to the difference of local models before and after local training in one global iteration (e.g., $\bm{x}^{t+1} - \bm{x}^t$). In this case, the central server receives all workers' model difference and performs a model update step using the average model difference. By applying Lookahead, which is a deep learning optimizer \cite{zhang2019lookahead}, to the update rule in FL, exchanging model difference between the workers and the central server has the following advantages: 1) Model difference exchange is necessary for developing a general local update framework with gradient-based local optimizer \cite{wang2020tackling}; 2) Model difference has decreasing dynamic ranges as iteration proceeds, which facilitates the design of quantization schemes \cite{reisizadeh2020fedpaq,das2020faster} for reducing the communication cost between the workers and server; 3) Exchanging model difference can promote the design of other local update rules in FL, such as \cite{reddi2021adaptive,tran2021feddr,fang2022communication,horvath2022fedshuffle}, because it only needs to exchange the difference instead of transmitting specific local information.
	\item \textbf{Local Gradient:} In this case, all workers transmit their local gradients to the central server, which performs a gradient decent step using the average local gradient. As the local gradients are decreasing as the training continues, the gradient information is suitable for alleviating communication burden by applying gradient quantization, sparsification \cite{rothchild2020fetchsgd,wang2022quantized} and lazy gradient aggregation \cite{LAG}.
\end{itemize}


\section{System Model}\label{sec: system model}
In this section, we elaborate the whole procedure of the AirFL algorithm considered in this tutorial, including the loss functions, the training steps, the communication schemes, and the aggregation steps.
\subsection{FL over Wireless Networks}
We consider a wireless FL system that consists of a single-antenna edge server and $N$ single-antenna edge devices indexed by set $\mathcal{N} = \{1, 2, \ldots, N\}$. Each edge device $n\in\mathcal{N}$ has a local dataset $\mathcal{D}_n=\{\bm{x}_{ni},y_{ni}\}_{i=1}^{D_n}$ with $D_n$ data samples. All edge devices collaboratively learn a shared global model by communicating with the edge server without moving their private raw datasets. For each edge device $n$, $\bm{\theta}_n \in \mathbb{R}^d $ is the local model trained based on its local dataset $\mathcal{D}_n$ and $F_n: \mathbb{R}^d\rightarrow \mathbb{R}$ is the local loss function, defined as 
\begin{equation}\label{eq: local obective}
	F_{n}\left(\bm\theta_n\right) :=\frac{1}{D_n}\sum_{i=1}^{D_n}f_n(\bm\theta_n;\bm{x}_{ni},y_{ni}),
\end{equation}
where $f_n: \mathbb{R}^d\rightarrow \mathbb{R}$ denotes the sample-wise loss function defined by the learning task. For example, the sample-wise loss function of the linear regression problem is given by
$f_n(\bm\theta_n;\bm{x}_{ni},y_{ni}) = \frac{1}{2}|\boldsymbol{x}_{ni}\boldsymbol{\theta}_n- y_{ni}|^2$,
and that of the logistic regression is 
$ f_n(\bm\theta_n;\bm{x}_{ni},y_{ni}) = \log\left(1+e^{-y_{ni}\boldsymbol{x}_{ni}^{\sf T}\boldsymbol{\theta}}\right) $.

Our goal is to learn a shared global model by minimizing the weighted sum of the edge devices' local loss functions. This can be formulated as the following optimization problem:
\begin{equation}\label{eq: problem}
	\mathscr{P}: 
		\underset{\bm{\theta}\in \mathbb{R}^d }{\text { minimize }} ~ F(\bm{\theta}):=
		\sum\limits_{n=1}^{N} p_nF_{n}\left(\bm\theta\right),
\end{equation}
where $\bm{\theta} \in \mathbb{R}^d$ is the global model with dimension $d$, $F(\cdot)$ is the objective function defined by the learning task, and $p_n > 0$ is the aggregation weight satisfying $\sum_{n=1}^{N} p_n = 1$. Generally, the aggregation weight $p_n$ can be simply set as $ D_n / D$, where $D = \sum_{n=1}^{N}D_n$ is the total number of data samples from all edge devices \cite{FedAvg}. Besides, the weights can be designed to promote fairness \cite{Li2020Fair}, personalization \cite{zhang2021personalized}, and efficiency of device sampling \cite{wang2022client}.
The solution to problem \eqref{eq: problem}, i.e., the global minimum $\bm\theta^*$, is defined as
\begin{align}\label{eq: global minimum}
	\bm\theta^*:=\operatorname{arg} \underset{\bm{\theta}\in \mathbb{R}^d}{\min}~F(\bm\theta).
\end{align}
\begin{figure}[tbp]
	\centering
	\includegraphics[width=1\linewidth]{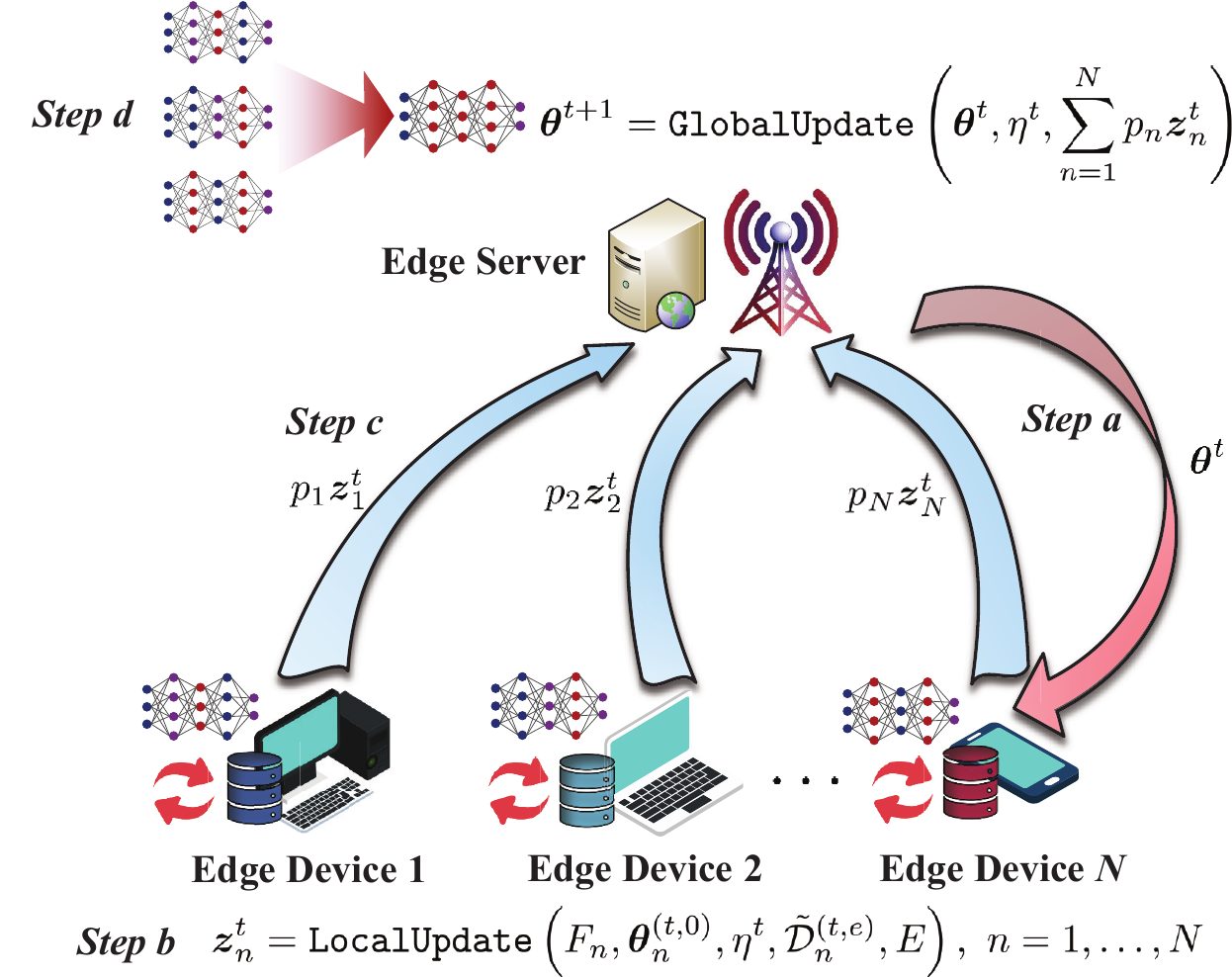}\hfil
	\caption{Illustration of the FL procedure.}
	\label{fig: 4cases}
	\vspace{-0.5cm}
\end{figure}

As illustrated in Fig \ref{fig: 4cases}, during each communication round $t$, the computation and transmission steps between the edge server and edge devices can be summarized as follows:
\begin{itemize}
	\item \textit{Step a}: The edge server broadcasts the global model $\bm\theta^{t}$ to all edge devices;
	\item \textit{Step b}: After receiving the global model $\bm\theta^{t}$, edge device $n\in\mathcal{N}$ initializes its local model by setting $\boldsymbol{\theta}_n^{(t,0)} = \bm\theta^{t}$, and performs $E\geq1$ epoch of local updates with respect to the local mini-batch set $\mathcal{\tilde{D}}_n^{(t,e)}$ with $\tilde{D}_n^{(t,e)}$ samples, $e\in[0,\cdots,E-1]$, via 
	\begin{equation}\label{eq: localupdate}
		\boldsymbol{z}_n^{t}:=\texttt{LocalUpdate}\left(F_{n},\boldsymbol{\theta}_n^{(t,0)} ,\eta^t,\mathcal{\tilde{D}}_n^{(t,e)},E\right),
	\end{equation}
where $\texttt{LocalUpdate}(\cdot): \mathbb{R}^d\rightarrow \mathbb{R}^d$ denotes the local model update operator, $\eta^t$ is the learning rate, and $\boldsymbol{z}_n^{t}\in \mathbb{R}^d$ is the output vector of the operator. If the local batch size equals to the local sample size, i.e., $\tilde{D}_n^{(t,e)}=D_n$, then the mini-batch stochastic gradient becomes a full-batch one.
	\item \textit{Step c}: All edge devices send the local information (e.g., local model, model difference, and local gradient) $p_n\boldsymbol{z}_n^{t}$ to the edge server\footnote{In this paper, we only consider full device participation. The discussion on partial device participation can be found in Section \ref{sec: extensions and future work}};
	\item \textit{Step d}: The edge server receives all the local information and updates the global model $\boldsymbol{\theta}^{t+1}$ by performing
	\begin{align}\label{eq: globalupdate}
		\bm{\theta}^{t+1} := \texttt{GlobalUpdate}\left(\boldsymbol{\theta}^t,\eta^t,\sum_{n=1}^N p_n\boldsymbol{z}_n^{t}\right),
	\end{align}	
	where $\texttt{GlobalUpdate}(\cdot): \mathbb{R}^d\rightarrow \mathbb{R}^d$ denotes the global model update operator, i.e., model aggregation. 
\end{itemize}
The iterative training process continues until a desired model accuracy is attained.

In this article, we aim to analyze the convergence of FL via first-order optimization method. Specifically, we consider the  state-of-the-art algorithm \fedavg \cite{FedAvg} with local vanilla SGD optimizer in our analysis.
The local SGD optimizer of \fedavg is performed as follows
\begin{align}\label{eq: definition gradient}
		\bm\theta_{n}^{(t,e+1)} \triangleq \bm\theta_{n}^{(t,e)} - \eta^{t}\bm g_n^{(t,e)}, ~e\in[0,\cdots,E-1],
\end{align}
where $ \bm g_n^{(t,e)} $ is the abbreviation for $\bm g_n(\boldsymbol{\theta}_n^{(t,e)})=\frac{1}{\tilde{D}_n^{(t,e)}}\sum_{i\in\mathcal{\tilde{D}}_n^{(t,e)}}\nabla  f_n(\bm\theta_n^{(t,e)};\bm{x}_{ni},y_{ni})$ and denotes the mini-batch stochastic gradient estimate of the local loss function $F_n(\bm\theta)$ with respect to $\bm\theta_{n}^{(t,e)}$ and the local mini-batch set $\mathcal{\tilde{D}}_n^{(t,e)}$ in the local iteration $e$ of the communication round $t$.

\begin{Rem}
	It is worth noting that different from the algorithms listed in Table \ref{tab: convergence2}, notation $T$ in our FL model refers to the total number of communication rounds. This means that the number of local computation for each edge device is $E\times T$ in this paper but $T$ in those literature.
\end{Rem}
\begin{Rem}
	The local model update operator \eqref{eq: localupdate} has different forms, resulting in different $ \boldsymbol{z}_n^{t} $, i.e., model difference, local gradient, and local model. For different forms of $\boldsymbol{z}_n^{t} $, the global update should be adjusted accordingly.
\end{Rem}

We first consider the situation that $ \boldsymbol{z}_n^{t} $ in \eqref{eq: localupdate} is the accumulated local gradients (i.e., model difference in other literature) \cite{wang2020tackling}, which can be written as
\begin{align}\label{eq: cumulative gradient}
	\boldsymbol{z}_n^{t} = \Delta\bm{\theta}^{t+1}_n \triangleq \bm\theta_{n}^{(t,E)} - \bm\theta_{n}^{(t,0)}= -\eta^{t}\sum\limits_{e=0}^{E-1}\bm g_{n}^{(t,e)}.
\end{align}
The edge server then aggregates all the accumulated local gradients and updates the global model $\bm\theta^{t+1}$ by performing
\begin{align}\label{eq: global update}
	\bm{\theta}^{t+1} = \bm{\theta}^{t}+\sum_{n=1}^N p_n \boldsymbol{z}_n^{t}.
\end{align}

\subsection{AirComp-based \fedavg (\airfedavg)}
In this paper, we mainly consider the transmission of local information from edge devices to the edge server over multiple-access fading channels in the uplink. The downlink transmission is assumed to be error-free \cite{amiri2020federated,RN36} as the edge server has a greater transmit power than the edge devices. Note that the convergence of FL over a noisy downlink is studied in \cite{amiri2022convergence,wei2021federated,ang2020robust,guo2022Joint}.

With the conventional orthogonal multiple access schemes \cite{OMA}, $N$ edge devices are assigned with $N$ orthogonal frequency/time resource blocks to transmit their local information. The edge server needs to successfully decode the transmitted information of all edge devices for global model update, which may be of low spectrum efficiency. To address this challenge, AirComp \cite{nazer2007computation}, as a non-orthogonal multiple access scheme, was proposed to exploit the waveform superposition property by enabling all edge devices to synchronously transmit their local information over the same channel. 
Moreover, a critical observation made by \cite{RN34} is that the waveform superposition property of AirComp is a perfect fit for model aggregation at the edge server. Prior works \cite{RN34} demonstrate that AirComp enables fast and spectrum-efficient model aggregation for wireless FL.


In this paper, all edge devices adopt AirComp to communicate with the edge server with perfect synchronization.
We assume a block flat-fading channel and each communication block comprises $d$ time slots for a $d$-dimensional local model to be transmitted, during which the channel coefficients remain unchanged.
By denoting $\alpha_n^t$ as the precoder of edge device $n$, the received signal at the edge server in communication round $t$ is given by
\begin{align}\label{eq: received signal}
	\hat{\bm s}^t = \sum_{n=1}^{N}h_n^t  \alpha_n^t p_n \boldsymbol z_n^t+ \boldsymbol w^t,
\end{align}
where $h_n^t\in \mathbb{C}$ is the channel coefficient for edge device $n$ in the $t$-th communication round, and $\bm w^t\in \mathbb{R}^d\sim\mathcal{CN}(0,\sigma_w^2\bm{I_d})$ is the additive white Gaussian noise (AWGN) vector.
In each model aggregation, AirFL suffers from the channel fading and the additive noise. To characterize the detrimental impact of the channel distortions, channel fading causes different magnitudes for different edge devices, and the additive noise of each communication round may stack up and further cause the FL model to diverge. 
These issues inspire us to design power control policies for magnitude alignment and signal processing schemes to reduce the impact of noise, thereby improving the training performance of AirFL.

As in most of the existing literature on AirComp \cite{RN34,RN36,RN37}, we assume that perfect channel state information (CSI) is available at each edge device and the edge sever. Each edge device is required to perform power control to
make the received signal satisfy magnitude alignment for model aggregation. By exploiting the CSI, each edge device can implement channel inversion by multiplying the local information by its inverse channel coefficient to compensate for both amplitude and phase of the fading channel, which is
\begin{align}\label{eq: inversion}
	\alpha_n^t  := \sqrt{\beta^t}\frac{(h_n^t)^{\sf{H}}}{|h_n^t|^2},
\end{align}
where $\beta^t$ is the denoising factor at the edge server. Then the estimated signal at the edge server is\footnote{In this paper, the transmitted signal of edge device $n$ is the local information multiplied by weight $p_n$, i.e., the edge server  receives the weighted sum of the transmitted local information. When $p_1=\dots=p_N=1/N$, the edge devices transmit only the local information without scaling by $1/N$, while the edge server needs to perform global averaging by multiplying $ 1/N $ for post-processing.}
\begin{align}\label{eq: AirFL}
	\hat{\bm y }^t := \frac{ \hat{\bm s}^t}{\sqrt{\beta^t}} =\sum_{n=1}^{N}p_n \bm z_n^t + \tilde{\bm w}^t,
\end{align}
where $\tilde{\bm w}^t \sim \mathcal{N}(0,\frac{\sigma_w^2}{\beta^t}\bm{I_d})$ is the equivalent noise vector. It is clear that the signal estimated by the edge server is distorted by noise, which is different from the error-free case. Such an imperfect estimation causes a model aggregation error (MAE), which refers to the difference between the estimated signal of the edge server and the error-free one, i.e.,
\begin{align}\label{eq: MAE}
	\bm \varepsilon^{t} := \hat{\bm y}^t - \bar{\bm y}^t,
\end{align}
where $\bar{\bm y}^t = \sum_{n=1}^{N}p_n \bm z_n^t $. The MAE is dominated by the additive noise with estimated signal \eqref{eq: AirFL} at the edge server. 
\begin{algorithm}[tbp]
	\caption{\airfedavgm Algorithm}
	\label{algorithm 1}
	\SetAlgoLined
	\SetKwInOut{Input}{Input}
	\SetKwFor{ParFor}{for each}{do in parallel}{end}
	\textbf{Algorithm of the Edge Server:} \\ 
	Initialize the global model and broadcast it to all edge devices. Set $t=0$\;
	\While{$t\leq T$}
	{		
		\textbf{wait} until receiving the perturbed and aggregated signal via AirComp \eqref{eq: received signal}\;
		\textbf{post-process} the received signal by dividing the denoising factor \eqref{eq: AirFL}\;
		\textbf{update} global model via  \eqref{eq: AirFedAvg-Multiple}\;
		\textbf{broadcast} the updated global model to all edge devices\;
		\textbf{set} $t \leftarrow t+1$\;
	}    
	\BlankLine
	\setcounter{AlgoLine}{0}
	\textbf{Algorithm of the $n$-th Edge Device:} \\
	Initialization: local model, $\eta^0, t = 0$\;
	\While{$t\leq T$}{
		\textbf{wait} until receiving global model from edge server\;
		\textbf{update} local model via \eqref{eq: cumulative gradient}\;
		\textbf{send} the pre-processed and precoded accumulated local gradients to edge server synchronously with respect to time and frequency\;
		\textbf{set} $t \leftarrow t+1$.
	}
\end{algorithm}

\begin{figure}[]
	\centering
	\subcaptionbox{Constant learning rate\label{fig: fixlr}}{\includegraphics[width=.49\linewidth]{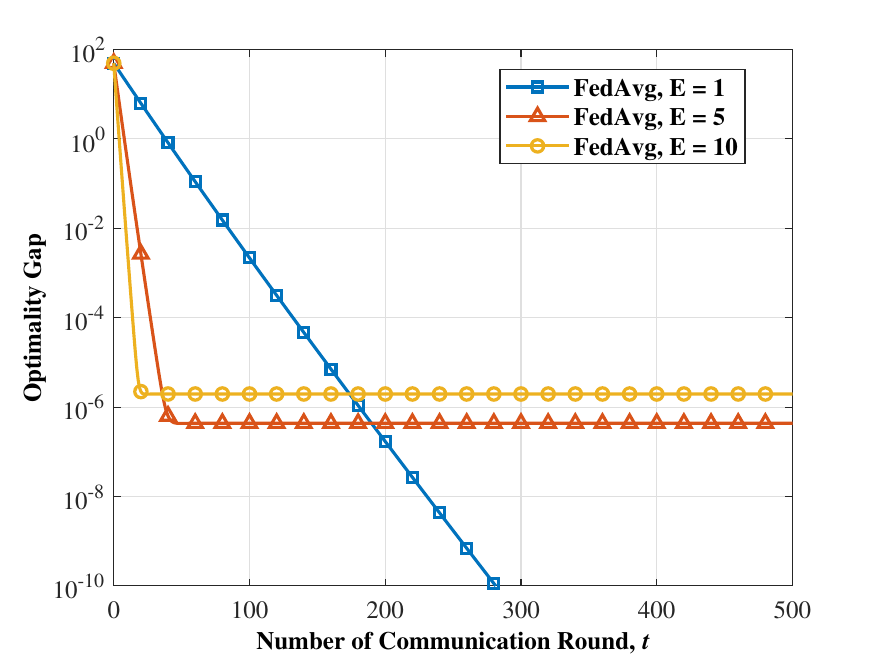}}\hfil
	\subcaptionbox{Diminishing learning rate\label{fig: decaylr}}{\includegraphics[width=.49\linewidth]{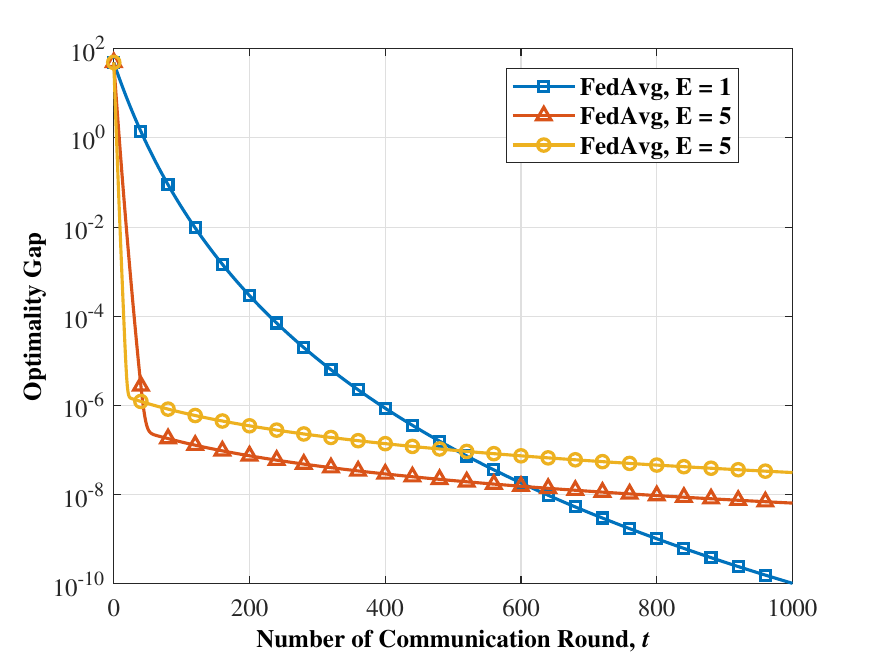}}\hfil
	\caption{\fedavg with $E = 1$, $E = 5$, and $E = 10$ under constant learning rate and learning rate decay for federated linear regression problems.}
	\label{fig: lr}
	\vspace{-0.5cm}
\end{figure}

\begin{Rem}
Besides, the channel inversion method we adopt has some variants. For instance, in \cite{COTAF,xia2020fast,fang2022communication}, the authors considered a threshold for channel gains of edge devices to filter out edge devices in deep fading,
leading to partial device participation. The threshold-based channel inversion can be summarized as
\begin{align}\label{eq: truncated}
	\alpha^t_n := \left\{
	\begin{array}{ll}
		0, &|h^t_n| < \gamma, \\
		\sqrt{\beta^t}\frac{(h_n^t)^{\sf{H}}}{|h_n^t|^2}, &|h^t_n| \geq \gamma,
	\end{array}
	\right.
\end{align}
where $ \gamma $ is the pre-determined threshold.
Under threshold-based channel inversion, full device participation reduces to partial device participation.
\end{Rem}

We first consider the scenario where edge device $n$ performs multiple local updates (i.e., $E>1$) and the accumulated local gradient (i.e., model difference) in \eqref{eq: cumulative gradient} is transmitted. This algorithm is named $\textsc{AirFedAvg-Multiple } (\airfedavgm)$ and is detailed in Algorithm \ref{algorithm 1}, which has the following global update operator
\begin{align}\label{eq: AirFedAvg-Multiple}
	\hat{\bm{\theta}}^{t+1} = \hat{\bm{\theta}}^{t}+\hat{\boldsymbol{y}}^t.
\end{align}
where $\hat{\bm{\theta}}^{t}$ represents the noisy global model.
We will discuss other two forms of $ \boldsymbol{z}_n^{t} $ (i.e., local gradient and local model) in Section \ref{sec: var}.

However, according to \cite{FedSplit,khaled2019first}, \fedavg with multiple local updates cannot optimally converge in some convex problems, e.g., federated least square and logistic regression, under the constant learning rate. As depicted in Fig. \ref{fig: lr}, \fedavg with single local update converges directly to the optimal point under the constant learning rate, while \fedavg with multiple local updates converges to a surrogate fixed point instead of the optimal point, resulting in a non-diminishing optimality gap. This is because multiple local updates may possibly be biased (refer to details in Section \ref{sec: var}), and the decaying learning rate can reduce such a bias in analogous to variance reduction of decaying learning rate on SGD. Using learning rate decay can address this problem \cite{Li2020On,charles2020outsized}, which is verified in Fig. \ref{fig: decaylr} with convergence speed slowed down. In view of this, together with the conclusions in Section \ref{sec: background}, the setting of learning rate $\eta^t$ in \airfedavgm is chosen as follows:
\begin{itemize}
	\item For strongly convex objective functions, we use a diminishing learning rate for algorithm analysis to attain diminishing optimality gap.
	\item For non-convex objective functions, we use a diminishing learning rate for algorithm analysis as well. Besides, we also present convergence results with constant learning rate, in terms of to achieve convergence rate and error bound on the average norm of the global gradient, which is related to the number of edge devices, the number of local updates, and the number of communication rounds,
\end{itemize}

\section{Convergence Analysis for \airfedavgm}\label{sec: convm}
In this section, we start by establishing the convergence analysis of \airfedavgm over a noisy uplink channel.
Our analysis is based on several basic and widely-used assumptions.
\begin{ass} [$L$-Smoothness] \label{ass: smooth}
	$F_{1}, \ldots, F_{N}$ are all $L$-smooth.
\end{ass}

\begin{ass}[$\mu$-Strong Convexity] \label{ass: strong convex}
	$F_{1}, \ldots, F_{N}$ are all $\mu$-strongly convex.
\end{ass}

\begin{ass}[Unbiased Gradient and Bounded Variance] \label{ass: gradient variance}
	Each edge device can query an unbiased stochastic gradient with bounded variance: 
	\begin{align*}
		&\mathbb{E}\left[ \bm{g}_{n}^{t}\right ] = \nabla F_{n}(\bm{\theta}^{t}),\\
		&\mathbb{E} \left[ \| \bm{g}_{n}^{t} - \nabla F_{n}(\bm{\theta}^{t}) \|_2^2 \right] \leq \sigma_{n}^2,
	\end{align*}
 where $\sigma_{n}^2$ is inversely proportional to the local mini-batch size $\tilde{D}_n^{t}$, for $n = 1, \ldots, N$.
\end{ass}

\begin{ass}[Bounded Gradient Dissimilarity (BGD)]\label{ass: bounded gradient dissimilarity}
	There exist constants $\beta_1\geq 1$ and $\beta_2\geq 0$, such that
	\begin{align*}
		\sum_{n=1}^{N}p_n\|\nabla   F_n(\bm x)\|_2^2\leq \beta_2+\beta_1 \left \|\nabla  F(\bm x)\right \|_2^2.
	\end{align*}
	Note that if the local datasets of all devices arte i.i.d., then it holds that $\beta_1 = 1$, and $\beta_2 = 0$.
\end{ass}

\subsection{Strongly Convex Case}\label{sec: conv}
\subsubsection{Main Results and Analysis}
The convergence for \airfedavgm is stated in \textbf{Theorem \ref{thm: case3unbiasedconvex}} for strongly convex objectives.

\begin{thm}[Convergence of \airfedavgm under Strong Convexity with Learning Rate Decay]\label{thm: case3unbiasedconvex}
	With \textbf{Assumptions \ref{ass: smooth}, \ref{ass: strong convex}}, \textbf{\ref{ass: gradient variance}}, and \textbf{\ref{ass: bounded gradient dissimilarity}}, if learning rate $0<\eta^t\leq \min\left\{\frac{1}{L\sqrt{2E(E-1)(2\beta_1+1)}},\frac{1}{2LE}\right\}$, then the upper bound on the cumulative gap after $T$ communication rounds is given by
	\begin{align}\label{eq: case3unbiasedconvex}\nonumber
		\mathbb{E}\left[F(\hat{\bm{\theta}}^{T})\right] - F^{\star}
		&\leq \left[F(\hat{\bm{\theta}}^{0}) - F^{\star}\right] M^0 + \underbrace{C_1 \sum_{t=0}^{T-1} (\eta^t)^3 M^{t+1}}_{(a)}\\
		&+ \underbrace{C_2 \sum_{t=0}^{T-1} (\eta^t)^2 M^{t+1} }_{(b)}
		+ \underbrace{C_3 \sum_{t=0}^{T-1} {\sf{MSE}}^t M^{t+1}}_{(c)},
	\end{align}
	where constants $M^t = \prod _{i=t}^{T-1}(1-\frac{\mu\eta^iE}{2}), M^T=1$, $C_1=\frac{L^2E(E-1)(2\beta_1+1) }{4\beta_1} \left[\sum_{n=1}^{N}p_n\sigma_{n}^2+2E\beta_2\right]$, $C_2=LE\sum_{n=1}^N p_n^2 \sigma_n^2$, and $C_3=\frac{L}{2}$.
	In addition, the MSE is a function of MAE in each communication round, and can be written as
	\begin{equation}\label{eq: MSE}
		{\sf{MSE}}^t(\bm \varepsilon^{t}) := \mathbb{E}\left[\left \| \bm \varepsilon^t\right \|_2^2 \right]  =\mathbb{E}\left[\left \| \tilde{\bm w}^t \right \|_2^2 \right] =  \frac{\sigma_{w}^2}{\beta^t},
	\end{equation}
	where the expectation is taken over the distribution of the random variable $\bm \varepsilon^{t}$.
\end{thm}
\begin{proof}
	Please refer to Appendix \ref{sec: case3unbiased}.
\end{proof}

\begin{Rem}\label{rem: remark}
	The key observations of \textbf{Theorem \ref{thm: case3unbiasedconvex}} include: term (a) is caused by bounded variance, non-IID data, and multiple local updates; term (b) is caused by gradient estimation variance; term (c) reflects the impact of the mean square error (MSE) caused by MAE on the convergence bound, which is dominated by the noise power and the denoising factor.
\end{Rem}
To analyze the convergence of partial sums (a), (b), and (c) in \textbf{Theorem \ref{thm: case3unbiasedconvex}}, we introduce a key lemma to support our theoretical analysis.

\begin{lem}[Lemma 25 in \cite{6172233}]\label{lem: sequence}
	Let sequences $\{s_1(t)\}_{t\geq0}$ and $\{s_2(t)\}_{t\geq0}$ be $$s_1(t)=\frac{a_1}{(t+1)^{\delta_1}},~s_2(t)=\frac{a_2}{(t+1)^{\delta_2}},~t\geq0,$$ 
	with $a_1, a_2, \delta_2 \geq 0$ and $0 \leq \delta_1 \leq 1$. If $\delta_1 =\delta_2$, then there exists a constant $\varphi$, such that for a large enough positive integer $j <T$, 
	\begin{align}
		0\leq \sum_{t=j}^{T-1} \left[s_2(t)\left(\prod _{i=t+1}^{T-1}(1-s_1(i))\right) \right]\leq\varphi.
	\end{align}
	Moreover, if $\delta_1 < \delta_2$, then for an arbitrary $j$, we have
	\begin{align}\label{eq: lemma1conv}
		\lim_{T \to \infty} \sum_{t=j}^{T-1} \left[s_2(t)\left(\prod _{i=t+1}^{T-1}(1-s_1(i))\right) \right]=0.
	\end{align}
\end{lem}
\begin{proof}
	Please refer to \cite[Appendix C]{6172233}.
\end{proof}

We carry out numerical experiments to illustrate the convergence of partial sums with different $(\delta_1,\delta_2)$ pairs, as presented in \textbf{Lemma \ref{lem: sequence}}. The value of sequences against the number of communication rounds in the log-scale is demonstrated in Fig. \ref{fig: sequence} with $(\delta_1,\delta_2)\in\{(1,0),(1,1),(1,2),(1,3)\}$.
\begin{figure}[tbp]
	\centering
	\includegraphics[width=1\linewidth]{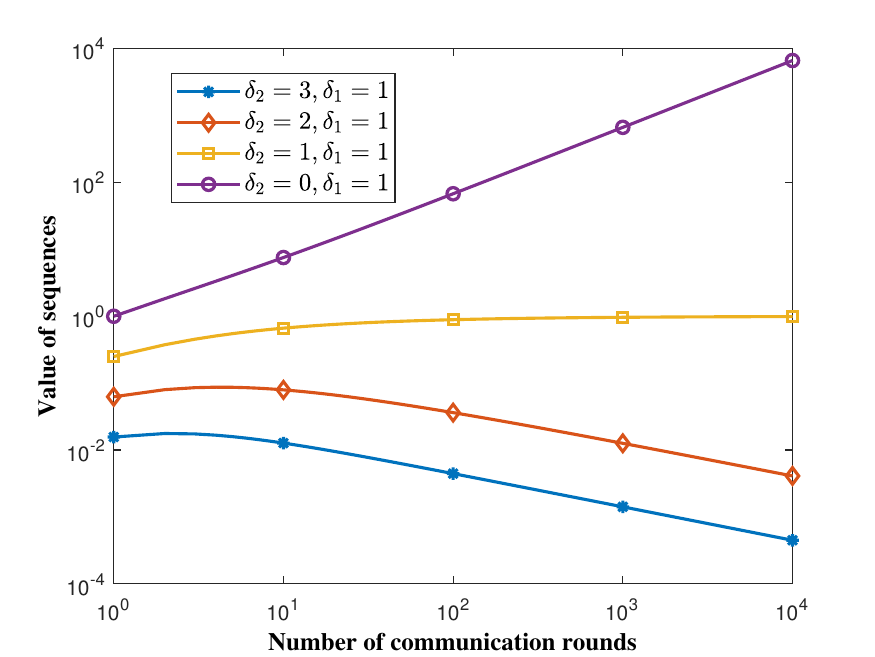}
	\caption{Values of the sequences versus number of communication rounds with different $(\delta_1,\delta_2)$ pairs in \textbf{Lemma \ref{lem: sequence}}.}
	\label{fig: sequence}
	\vspace{-0.5cm}
\end{figure}
Specifically, we set the learning rate as
\begin{equation*}
	\eta^{t} = \frac{\eta^{0}}{ t+ 1},~\eta^{0} = \min\left\{\frac{1}{L\sqrt{2E(E-1)(2\beta_1+1)}},\frac{1}{2LE}\right\},
\end{equation*}
for further analysis, which satisfies condition \eqref{eq: learning rate} in Section \ref{sec: background}.

Term (a) in \textbf{Theorem \ref{thm: case3unbiasedconvex}} can be reformulated as
\begin{equation*}
	C_1 \sum_{t=0}^{T-1} \left[\left(\frac{\eta^{0}}{ t+ 1}\right)^3 \prod _{i=t+1}^{T-1}\left(1-\frac{\mu E\eta^{0}/2}{ i+ 1}\right)\right],
\end{equation*}
which is equivalent to \eqref{eq: lemma1conv} with $j = 0$, $a_1 =\frac{\mu E  \eta^0}{2}$, $\delta_1 = 1$, $a_2 =  (\eta^0)^3$ and $\delta_2 = 3$ in \textbf{Lemma \ref{lem: sequence}}. This indicates that term (a) of \eqref{eq: case3unbiasedconvex} in \textbf{Theorem \ref{thm: case3unbiasedconvex}} is diminishing as $T\rightarrow \infty$ because $\delta_1<\delta_2$. The convergent property is illustrated in Fig. \ref{fig: sequence} with $(\delta_1,\delta_2)=(1,3)$. 

Similarly,  the diminishing property of term (b) of \eqref{eq: case3unbiasedconvex} in \textbf{Theorem \ref{thm: case3unbiasedconvex}} can be verified by setting $j = 0$, $a_1 =\frac{\mu E  \eta^0}{2}$, $\delta_1 = 1$, $a_2 =  (\eta^0)^2$, and $\delta_2 = 2$ in \textbf{Lemma \ref{lem: sequence}}. The diminishing property is also illustrated in Fig. \ref{fig: sequence} with $(\delta_1,\delta_2)=(1,2)$.

However, to analyze the convergence property of term (c), it hinges on whether MSE satisfies the conditions in \textbf{Lemma \ref{lem: sequence}}. Therefore, it is required to provide elaborate classification with respect to the MSE for further discussion.

\subsubsection{Case Study}\label{sec: casestudy1}
According to \eqref{eq: MSE}, the MSE is determined by the noise power and denoising factor. To discuss the influence of the MSE on term (c), we consider the following cases.
\begin{itemize}
	\item \textbf{The MSE approaches to zero in each communication round.} In this case, term (c) in \eqref{eq: case3unbiasedconvex} is eliminated, which means \airfedavgm approximates to error-free \fedavg. In addition, for asymptotic analysis, if ${\sf{MSE}}^t\rightarrow0$, we must have $\sigma_{w}^2\rightarrow 0$ or $\beta^t\rightarrow\infty$, which demands an infinite large SNR. However, neither of the two conditions is practical in real AirComp systems since we cannot control the noise power and the transmit power is limited.
	\item \textbf{The MSE is a positive constant in each communication round.} In this case, we have 
	\begin{equation*}
		{\sf{MSE}}^t\sim (\eta^t)^0=1,
	\end{equation*}
	which means that for the strongly convex case, term (c) is non-diminishing as presented in Fig. \ref{fig: sequence} with $(\delta_1,\delta_2)=(1,0)$.
	\item \textbf{The MSE is proportional to the diminishing learning rate.} For strongly convex case, term (c) converges if and only if the following condition holds
	\begin{equation}\label{eq: MSE condition}
		{\sf{MSE}}^t\sim (\eta^t)^\delta = \left(\frac{\eta^0}{t+1}\right)^\delta,~\delta>1,
	\end{equation}
	without requiring an infinite large SNR. This condition can be directly derived from \textbf{Lemma \ref{lem: sequence}}.
\end{itemize}

In summary, we are unable to obtain a convergent upper bound for strongly convex case of \airfedavgm without properly designing the AirComp system to make the MSE caused by model aggregation satisfy condition \eqref{eq: MSE condition}. Hence, the impact of the MAE caused by the wireless channel needs to be carefully studied.

According to \eqref{eq: MSE}, to achieve convergence guarantee for \airfedavgm, one possible approach is to jointly design the transmit scalar and the denoising factor $\sqrt{\beta^t}$. 
As inspired, a lot of recent efforts applied the precoding and denoising techniques to ensure the decaying property of term (c) and further provided an upper bound for the convergence optimality gap \cite{COTAF}. In this tutorial, we verify the observations by adopting an exemplifying denoising factor design and deriving some insightful conclusions regarding the convergence rate and optimality gap, which are presented as follows.

According to the convergence analysis on the sequences of the upper bound, the convergence of the upper bound can be guaranteed and the impact of the additive noise can be compensated as long as the denoising factor satisfies the following condition:
\begin{align}\label{eq: precoder condition}
	(\beta^t)^{-1}  \sim (\eta^t)^{\delta},~ \delta>1,
\end{align}
which can be directly derived by combining \eqref{eq: MSE} and \eqref{eq: MSE condition}.
The transmit power constraint in communication round $t$ for edge device $n$ is given by
\begin{align*}
	\mathbb{E}\left[\left\| \alpha_n^t p_n \boldsymbol z_n^t\right\|_2^2\right]\leq d\times P_0,
\end{align*}
where $P_0$ is the transmit power, and the SNR can be written as
\begin{equation}\label{SNR}
	{\sf SNR} := \frac{P_0}{\sigma_w^2}.
\end{equation}
Based on the transmit power constraint, we adopt a denoising factor that takes the norm of local information into account to mitigate the detrimental influence caused by the additive noise, i.e.,
\begin{align}\label{eq: precoding}
	\beta^t  := \min_{n\in \mathcal{N}}\frac{|h_n^t|^2dP_0}{ \|p_n\bm z_n^{t}\|_2^2}.
\end{align}
Under this setup, the information recovered by the edge server is an unbiased estimation of the transmit information. However, this setting may be suboptimal especially in deep fading scenarios \cite{xiaowen2021optimized}.
Since $\bm z_n^{t}=-\eta^{t}\sum_{e=0}^{E-1}\bm g_{n}^{(t,e)}$, the denoising factor can be further written as
\begin{align*}
	(\beta^t)^{-1}  = (\eta^t)^2\max_{n\in \mathcal{N}}\frac{\left\|p_n\sum_{e=0}^{E-1}\bm g_{n}^{(t,e)}\right\|_2^2}{ |h_n^t|^2dP_0}\sim (\eta^t)^{2}.
\end{align*}
Condition \eqref{eq: precoder condition} can be proved to hold according to \cite[\textbf{Lemma A.2}]{COTAF} with the BG assumption, which is presented as follows.
\begin{ass}[Bounded Gradient (BG)]\label{ass: BG}
		The expected squared norm of stochastic gradient of each edge device is bounded by
	\begin{align*}
					\mathbb{E}[\|\bm{g}_n^{(t,e)}\|_2^2]\leq G^2,~\forall n,
		\end{align*}
where $G>0$ is a constant.
\end{ass}
The denoising factor can be further specified as
\begin{equation}\label{eq: betabound}
	(\beta^t)^{-1} = \frac{\tilde{G}^2E(\eta^t)^2}{dN^2P_0},
\end{equation}
by \textbf{Assumption \ref{ass: BG}}, where $\tilde{G}^2 = G^2\max_{n\in \mathcal{N}}\frac{v_n^2}{|h_n^t|^2}$ with $v_n = Np_n$, and $\sum_{n=1}^{N}v_n = N$. This indicates that \textbf{Assumption \ref{ass: BG}} is not necessary in error-free transmission scenarios, but usually needed in wireless FL scenarios \cite{COTAF,xiaowen2021optimized,xiaowen2021transmission,fan2021temporal,jing2022Federated}. By substituting \eqref{eq: betabound} into \eqref{eq: MSE}, the MSE of model aggregation can be written as
\begin{equation*}
	{\sf{MSE}}^t = \frac{\sigma_{w}^2\tilde{G}^2E(\eta^t)^2}{dN^2P_0} = \frac{\tilde{G}^2E(\eta^t)^2}{dN^2{\sf SNR}},
\end{equation*}
where the MSE is inversely proportional to SNR.

This transmission scheme leverages the decaying property of the accumulated local gradients to eliminate the impact of receiver noise as the training process proceeds. 
Similarly, the authors in \cite{wei2021federated} proposed an SNR control policy, which requires that the noise power decays with respect to communication rounds, i.e., 
\begin{align*}
	(\sigma_{w}^t)^2 \sim \mathcal{O}\left(\frac{1}{t^2}\right ).
\end{align*}
This scheme coincides with our statement because decaying the power of noise $\tilde{\boldsymbol{w}}^t$ with fixed denoising factor is equivalent to setting $\beta^t$ as in \eqref{eq: precoder condition}.

Subsequently, we have the following corollary for the strongly convex case to derive the diminishing optimality gap and convergence rate for \airfedavgm.

\begin{cor}[Optimality Gap of \airfedavgm under Strong Convexity with Learning Rate Decay]\label{cor: case3convex}
	If the learning rate is set as $\eta^{t} = \frac{6}{E\mu(\tau + t)}$, with $\tau  = \frac{3L}{\mu}$, the optimality gap $\mathbb{E}\left[ F(\hat{\bm \theta}^{T}) \right]  - F^{\star}$ of the strongly convex case of \airfedavgm converges to zero with rate
	\begin{align}\label{eq: optimalitycase3}
		\mathbb{E}\left[ F(\hat{\bm \theta}^{T}) \right]   - F^{\star}\leq	\mathcal{O}\left( \frac{A_1}{E^2\mu^3T^2}\right)+\mathcal{O}\left( \frac{B_1}{NE\mu^2T}\right),
	\end{align}
	where $A_1 = L^2(E-1)(2+1/\beta_1)[\bar{\sigma}^2+E\beta_2]$, $B_1 = L\Sigma+L/dN{\sf SNR}$, $\Sigma = N\sigma^2$, $\bar{\sigma}^2 = \sum_{n=1}^{N}p_n\sigma_{n}^2$, and $\sigma^2 = \sum_{n=1}^{N}p_n^2\sigma_{n}^2$. 

	In order to achieve bound \eqref{eq: optimalitycase3}, the number of local updates is bounded by
	\begin{equation}
		E\leq \frac{6(2\beta_1+1)\bar{\sigma}^2+12\beta_1(\sigma^2+\tilde{G}^2/dN^2{\sf SNR})}{\beta_1\mu (F(\hat{\bm \theta}^{0}) - F^{\star})-12\beta_2(2\beta_1+1)}.
	\end{equation}
\end{cor}
\begin{proof}
	Please refer to Appendix \ref{sec: multiple}.
\end{proof}

\begin{Rem}
	The right hand side of \eqref{eq: optimalitycase3} is dominated by $\mathcal{O}((NET)^{-1})$, and the additional error decays faster with rate $\mathcal{O}((ET)^{-2})$. It is obvious that \airfedavgm in the strongly convex case with proper denoising factor achieves linear speed up in terms of the number of local updates and the number of edge devices.
\end{Rem}

\subsection{Non-Convex Case}
\subsubsection{Main Results and Analysis}
We derive \textbf{Theorem \ref{thm: case3unbiasednonconvex}} for the error bound of the non-convex objectives.

\begin{thm}[Convergence of \airfedavgm under Non-convexity with Learning Rate Decay]\label{thm: case3unbiasednonconvex}
	Let \textbf{Assumptions \ref{ass: smooth}}, \textbf{\ref{ass: gradient variance}}, and \textbf{\ref{ass: bounded gradient dissimilarity}} hold. If the diminishing learning rate satisfies $0<\eta^t\leq \min\left\{\frac{1}{L\sqrt{2E(E-1)(2\beta_1+1)}},\frac{1}{2LE}\right\}$, then the weighted average norm of global gradients after $T$ communication rounds is upper bounded by
	\begin{align}\label{eq: case3unbiasednonconvex}\nonumber
		&\frac{1}{\Phi}\sum\limits_{t = 0}^{T-1} \eta^t\mathbb{E} \left[\left \|\nabla F(\hat{\bm \theta}^{t})\right \|_2^2\right]\\\nonumber
		\leq& \frac{4\left[F(\hat{\bm \theta}^{0}) - F^{\inf}\right] }{E\Phi} 
		+ \underbrace{\frac{4 C_1}{E\Phi}\sum_{t = 0}^{T-1} (\eta^t)^3}_{(a)}
		+ \underbrace{\frac{4 C_2 }{E\Phi}\sum_{t = 0}^{T-1} (\eta^t)^2}_{(b)}\\
		+ & \underbrace{\frac{ 4 C_3}{E\Phi} \sum_{t = 0}^{T-1} {\sf{MSE}}^t}_{(c)}, 
	\end{align}
	where $\Phi=\sum
	\limits_{j = 0}^{T-1} \eta^j$.
\end{thm}
\begin{proof}
	Please refer to Appendix \ref{sec: case3unbiased}.
\end{proof}
Similarly, to analyze the convergence of each sequence in \textbf{Theorem \ref{thm: case3unbiasednonconvex}}, we introduce a key lemma to support our theoretical analysis for the non-convex case. 

\begin{lem}[Stolz–Ces$\grave{a}$ro Theorem]\label{lem: stolz}
	Let $\{a_n\}_{n\geq1}$ and $\{b_n\}_{n\geq1}$ be two sequences of real numbers. Assume that $\{b_n\}_{n\geq1}$ is a strictly monotone and divergent sequence, and the following limit exists:
	$$\lim_{n \to \infty}\frac{a_{n+1}-a_n}{b_{n+1}-b_n} = \ell.$$
	Then we have the following limit
	$$\lim_{n \to \infty}\frac{a_n}{b_n} = \ell.$$
\end{lem}

Using the same diminishing learning rate as the strongly convex case, the partial sum $\Phi=\sum_{j = 0}^{T-1} \eta^j$ diverges due to the property of harmonic series.
For term (a) in \textbf{Theorem \ref{thm: case3unbiasednonconvex}}, its convergent property can be easily verified by \textbf{Lemma \ref{lem: stolz}}, which is a common criterion for proving the convergence of a sequence, as the numerator $a_T = \sum_{t = 0}^{T-1} (\eta^t)^3$ is of higher order than the denominator $b_T = \Phi = \sum_{t = 0}^{T-1} \eta^t$. Similarly, the diminishing property of term (b) in \textbf{Theorem \ref{thm: case3unbiasednonconvex}} can be verified as well.

As in the strongly convex case, to analyze the convergence property of term (c) in \textbf{Theorem \ref{thm: case3unbiasednonconvex}}, it is also required to provide elaborate classification with respect to the MSE. According to \textbf{Lemmas \ref{lem: sequence}} and \textbf{\ref{lem: stolz}}, the properties of terms (a), (b), and (c) in the above two theorems are equivalent. In view of this, the convergence condition of the MSE for non-convex cases with learning rate decay is identical to that in strongly convex cases. 

To conclude, the denoising factor adopted in the previous section can also be applied in the non-convex case, which guarantees the convergence property of term (c) in \textbf{Theorem \ref{thm: case3unbiasednonconvex}}.

\subsubsection{Case Study}
For case study in the non-convex case with denoising factor \eqref{eq: betabound}, we have the following two corollaries for the diminishing error bound and convergence rate.
\begin{cor}[Error Bound of \airfedavgm under Non-Convexity with Learning Rate Decay]\label{cor: case3nonconvex}
Given learning rate $\eta^{t} = \frac{\eta^{0}}{ t+ 1}$, and $\eta^{0} = \min\left\{\frac{1}{L\sqrt{2E(E-1)(2\beta_1+1)}},\frac{1}{2LE}\right\}$, the error bound of \airfedavgm under non-convexity is diminishing, i.e.,
\begin{equation}\label{eq: zero}
	\frac{1}{\Phi}\sum\limits_{t = 0}^{T-1} \eta^t\mathbb{E} \left[\left \|\nabla F(\hat{\bm \theta}^{t})\right \|_2^2\right]\stackrel{T\rightarrow\infty}{\longrightarrow} 0.
\end{equation}
\end{cor}
\begin{proof}
	By combining \eqref{eq: non-decay2}, \eqref{eq: case3unbiasednonconvex} and \eqref{eq: betabound}, we obtain \eqref{eq: zero}.
\end{proof}

\begin{cor}[Error Bound of \airfedavgm under Non-convexity with Constant Learning Rate]\label{cor: case3}
	When the learning rate is a constant satisfying 
	\begin{align*}
		\eta = \frac{1}{L}\sqrt{\frac{N}{ET}}\leq\min\left\{\frac{1}{L\sqrt{2E(E-1)(2\beta_1+1)}},\frac{1}{2LE}\right\},
	\end{align*}
	the minimal gradient norm of the global objective function of \airfedavgm is bounded as follows
	\begin{flalign}\nonumber
		&\min_{t\in[T]} \mathbb{E} \left[\left \|\nabla F(\hat{\bm \theta}^{t})\right \|_2^2\right]\leq \frac{1}{T}\sum\limits_{t = 0}^{T-1} \mathbb{E} \left[\left \|\nabla F(\hat{\bm \theta}^{t})\right \|_2^2\right]\\\nonumber
		\leq&\mathcal{O}\left( \frac{1+\Sigma}{\sqrt{NET}}\right)
		 + \mathcal{O}\left(\frac{\tilde{C}N (\bar{\sigma}^2+E\beta_2)}{ET}\right)\\
		 +& \mathcal{O}\left(\frac{1}{d{\sf SNR}\sqrt{N^3ET}}\right),
	\end{flalign}
	where $\Sigma = N\sigma^2$ and $\tilde{C} = \frac{(E-1)(2\beta_1+1)}{\beta_1}$. The right hand side of the inequality is dominated by $\mathcal{O}\left( \frac{1+\Sigma}{\sqrt{NET}}\right)$.
\end{cor}

It is worth noting that \textbf{Corollary \ref{cor: case3}} recovers the results of the error-free case by setting $\sigma_{w}^2 = 0$, i.e., ${\sf SNR}\rightarrow \infty$ in \cite{wang2020tackling}.

\begin{Rem}\label{rem: airfedavgm noise}
	The impact of receiver noise on the convergence of \airfedavgm is mathematically similar to the gradient estimation noise caused by vanilla SGD. In addition, we observe from \textbf{Corollaries \ref{cor: case3convex}} and \textbf{\ref{cor: case3}} that a lower SNR amplifies the optimality gap for strongly convex cases and the error bound for non-convex objectives. 
\end{Rem}
\begin{Rem}\label{rem: airfedavgm E}
	It is clear that for both strongly convex and non-convex objectives, a greater number of local updates leads to a faster convergence rate at the cost of worsening the impact of non-IID data by enlarging the optimality gap in \textbf{Corollary \ref{cor: case3convex}} and the error bound in \textbf{Corollary \ref{cor: case3}}. Moreover, by setting $E = 1$ for \airfedavgm, the impact caused by non-IID data is eliminated. This means that by setting $E = 1$ and $\sigma_{w}^2 = 0$ in \textbf{Corollary \ref{cor: case3convex}} and \textbf{\ref{cor: case3}}, the result is consistent with the standard convergence rate for vanilla SGD \cite{bottou2018optimization} for both strongly convex and non-convex cases without data heterogeneity.
\end{Rem}

To conclude, \airfedavgm ($E\geq1$) is a typical example of \airfedavg in terms of algorithm analysis. 
Other important results will be discussed in the following section.

\section{Variants of \airfedavg}\label{sec: var}
As mentioned previously, local model, model difference and local gradient are three different forms of local information that can be exchanged between the edge server and edge devices for global model aggregation. In error-free \fedavg, these different forms are equivalent and transferable from each other in the training process.
However, with AirComp \eqref{eq: AirFL}, the variants of
\airfedavg are no longer identical due to the additive noise in each communication round.
In this section, different variants of \airfedavg will be presented, and the insights and results of these variations will be analyzed and discussed. 
\subsection{\airfedavg with Local Gradients}
Although \airfedavgm is an effective algorithm for AirFL \cite{COTAF,wei2021federated,xiaowen2021transmission,xu2021LR,fan2021temporal,sifaou2021robust,shao2021misaligned,lin2022relay,jing2022Federated}, it faces several issues as follows.
\begin{itemize}
    \item In \eqref{eq: cumulative gradient}, if local gradients are not multiplied by the learning rate, then
    \begin{align*}
    	\boldsymbol{z}_n^{t} = \Delta\bm{\theta}^{t+1}_n/\eta^t \triangleq  -\sum\limits_{e=0}^{E-1}\bm g_{n}^{(t,e)}.
    \end{align*}
	In this case, the global update is given by
	\begin{align}
		\hat{\bm{\theta}}^{t+1} \overset{\eqref{eq: MAE}}{=}\hat{\bm{\theta}}^{t}+\eta^t\left(\sum_{n=1}^{N}p_n \bm z_n^t+\boldsymbol{\varepsilon}^t\right).
	\end{align}
	If the learning rate changes every local update in one communication round, i.e., $\eta^t\rightarrow\eta^{(t,e)}$, then this scheme cannot be well generalized because it is difficult to select $\eta^{(t,e)}$ for obtaining $\bm{z}_n^t$ and model aggregation. 
	\item From \textbf{Remark \ref{rem: airfedavgm E}}, we know that a larger number of local updates enlarges the aggregation bias due to non-IID data \cite{karimireddy2020scaffold}. To be more specific, for $E = 1$ and
$ 		\bm{\theta}_n^{(t,1)} = \bm{\theta}_n^{(t,0)} - \eta^t \bm{g}_n^{(t,0)}$, 
	$\bm{\theta}_n^{(t,1)}-\bm{\theta}_n^{(t,0)}$ is an unbiased estimator of $-\eta^t \nabla F(\bm{\theta}^{t})$. However, for $E = 2$, we have
		\begin{equation*}
		\bm{\theta}_n^{(t,2)} = \bm{\theta}_n^{(t,0)} - \eta^t \bm{g}_n^{(t,0)}- \eta^t \bm{g}_n^{(t,1)},
	\end{equation*}
	and $\bm{\theta}_n^{(t,2)}-\bm{\theta}_n^{(t,0)}$ is neither the unbiased estimator of $-\eta^t \nabla F(\bm{\theta}^{t})$ nor that of $-\eta^t \nabla F(\bm{\theta}^{t}-\eta^t\nabla F(\bm{\theta}^{t}))$ \cite{reisizadeh2020fedpaq}. This reflects the advantage of single local update.
	\item From the results of \airfedavgm, we know that the requirements for the MSE in \airfedavgm is quite stringent, i.e., the upper bound of \airfedavgm cannot converge with an arbitrary MSE. Simply setting $E = 1$ in \airfedavgm still requires strict restrictions on the model aggregation MSE.
\end{itemize}

To tackle these challenges, we propose to exploit an important variant of \airfedavg. Specially, when $E = 1$, the local computation can directly be local gradient estimation without performing a decent step, i.e.,
\begin{align}\label{eq: lgoutput}
	\boldsymbol{z}_n^{t} = \bm g_{n}^{(t,0)}.
\end{align}
We name this algorithm as \airfedavg with single local gradient, \airfedavgs. The global update of \airfedavgs is 
\begin{align}\label{eq: AirFedAvg-Single}
	\hat{\bm{\theta}}^{t+1} = \hat{\bm{\theta}}^{t}-\eta^t \hat{\boldsymbol{y}}^t\overset{\eqref{eq: MAE}}{=}\hat{\bm{\theta}}^{t}-\eta^t(\bar{\boldsymbol{y}}^t+\boldsymbol{\varepsilon}^t),
\end{align}
where $\hat{\boldsymbol{y}}^t$ is defined in \eqref{eq: AirFL} and $\bar{\bm y}^t = \sum_{n=1}^{N}p_n \bm z_n^t $.

As to the convergence analysis for \airfedavgs, the model aggregation of this algorithm is special. Compared with the global model update in \airfedavgm ($E = 1$) \eqref{eq: AirFedAvg-Multiple}, it is observed that the MAE of \airfedavgs is naturally multiplied by the learning rate for model aggregation while \airfedavgm does not have this property, which implies that the convergence of these two algorithms may not be identical and simply setting $E = 1$ in the convergence results of \airfedavgm is improper. Assume that this property is prone to relax the requirements on the model aggregation MSE as discussed in the previous section. We will prove this conjecture. 

The setting of learning rate $\eta^t$ in \airfedavgs can be chosen as follows:
\begin{itemize}
	\item For strongly convex objective functions, we unify to use a diminishing learning rate to attain diminishing optimality gap. The results for constant learning rate are presented as corollaries for non-diminishing optimality gap.
	\item For non-convex objective functions, we use a diminishing learning rate for convergence analysis as well. In addition, we will present convergence results with constant learning rate in terms of the convergence rate and the error bound.
\end{itemize}

\subsubsection{Main Results and Case Study}
To analyze the impact of receive noise in \airfedavgs, we summarize the convergence results with respect to the MSE in \textbf{Theorem \ref{thm: case1unbiasedconvex}} for strongly convex objectives.

\begin{thm}[Convergence of \airfedavgs under Strong Convexity with Learning Rate Decay]\label{thm: case1unbiasedconvex}
	Let \textbf{Assumption \ref{ass: smooth}, \ref{ass: strong convex}} and \textbf{\ref{ass: gradient variance}} hold, where $L, \mu, \sigma_{n}$ are defined. If the decaying learning rate satisfies $0<\eta^t\leq \frac{1}{L}$, then the upper bound on the optimality gap after $T$ communication rounds is given by
	\begin{align}\label{eq: case1unbiasedconvex}\nonumber
		&\mathbb{E}\left[ F(\hat{\bm \theta}^{T}) \right] - F^{\star} \\
		\leq& \left[F(\hat{\bm \theta}^{0}) - F^{\star}\right] P^0 + \frac{L}{2}\sum_{t=0}^{T-1}(\eta^t)^2 (\sigma^2+{\sf{MSE}}^t)P^{t+1},
	\end{align}
	where $P^t = \prod _{i=t}^{T-1}(1-\mu\eta^i), P^T=1$, and ${\sf{MSE}}^t$ is defined in \eqref{eq: MSE}.
\end{thm}

\begin{proof}
	Please refer to Appendix \ref{sec: case1unbiased}.
\end{proof}

Using \textbf{Lemma \ref{lem: sequence}}, we can easily prove that for strongly convex objective functions, the last term of \eqref{eq: case1unbiasedconvex} is a convergent sequence for all three conditions of the MSE listed in Section \ref{sec: casestudy1}.
The following two corollaries for the strongly convex cases can be obtained based on \textbf{Theorem \ref{thm: case1unbiasedconvex}}.
\begin{cor}[Optimality Gap of \airfedavgs under Convexity with Constant Learning Rate]\label{cor: case1constant}
	If the learning rate is a constant satisfying $\eta^{t}  = \eta= \frac{1}{L}$, the optimality gap $\mathbb{E}\left[ F(\hat{\bm \theta}^{T}) \right]  - F^{\star}$ of the convex case of \airfedavgs can be bounded by
	\begin{align}\nonumber
		\mathbb{E}\left[ F(\hat{\bm \theta}^{T}) \right]   - F^{\star}
		&\leq \left(1-\frac{\mu}{L}\right)^T\left[F(\hat{\bm \theta}^{0}) - F^{\star}\right]\\
		& +\left(1-\frac{\mu}{L}\right)^{T-t-1} \frac{1}{2L}\sum_{t=0}^{T-1} (\sigma^2+{\sf{MSE}}^t).
	\end{align}
	
\end{cor}

The second term of the right hand side is a non-diminishing term. Note that this is the most frequently-used conclusion in wireless FL literature with respect to \airfedavgs. Subsequently, to provide case study with respect to the denoising factor \eqref{eq: precoding}, we have the following corollary for the strongly convex case.
\begin{cor}[Optimality Gap of \airfedavgs under Convexity with Learning Rate Decay]\label{cor: case1convex}
	If the precoding factor is a constant, and the learning rate is set as $\eta^{t} = \frac{2}{\mu(\tau + t)}$, with $\tau  = \frac{2L}{\mu}$, then the optimality gap $\mathbb{E}\left[ F(\hat{\bm \theta}^{T}) \right]  - F^{\star}$ of the convex case of \airfedavgs converges to zero with rate
	\begin{align}
		\mathbb{E}\left[ F(\hat{\bm \theta}^{T}) \right]  - F^{\star}\leq \mathcal{O}\left( \frac{\tilde{B}}{\mu^2NT}\right),
	\end{align}
	where $\tilde{B} = \max\left\{2L(\Sigma+\frac{\tilde{G}^2}{dN{\sf SNR}}),\mu^2N\tau(F(\hat{\bm{\theta}}^{0})  - F^{\star})\right\}$. 
\end{cor}
\begin{proof}
	Please refer to Appendix \ref{sec: single}.
\end{proof}

Note that this result is equivalent to the statements in \cite{Li2020On,guo2021analog} with respect to the convergence rate.
According to \textbf{Corollary \ref{cor: case1constant}} and \textbf{Corollary \ref{cor: case1convex}}, we can conclude that in the strongly convex case, the expected objective values converge linearly to a neighborhood of the optimal value for \airfedavgs with a constant learning rate. However the impact of MAE, i.e., the model aggregation MSE, hinders further convergence to the optimal solution. On the other hand, it can converge to the optimal point at the expense of achieving a slower sublinear convergence speed, which is in the order of $\mathcal{O}\left( 1/NT\right)$.

Besides, the convergence results for non-convex objectives with diminishing learning rate can be found in \textbf{Theorem \ref{thm: case1unbiasednonconvex}}.

\begin{thm}[Convergence of \airfedavgs under Non-convexity with Learning Rate Decay]\label{thm: case1unbiasednonconvex}
	Let \textbf{Assumptions \ref{ass: smooth}} and \textbf{\ref{ass: gradient variance}} hold, and the learning rate satisfy $0<\eta^t\leq \frac{1}{L}$. The weighted average norm of global gradients after $T$ communication rounds is upper bounded by
	\begin{align}\label{eq: case1unbiasednonconvex}\nonumber
		&\frac{1}{\Phi}\sum\limits_{t = 0}^{T-1} \eta^t\mathbb{E} \left[\left \|\nabla F(\hat{\bm \theta}^{t})\right \|_2^2\right]\\
		\leq& \frac{2}{\Phi} \left[F(\hat{\bm \theta}^{0}) - F^{\inf}\right] 
		+  \frac{L}{\Phi}\sum_{t = 0}^{T-1} (\eta^t)^2(\sigma^2+{\sf{MSE}}^t).
	\end{align}
\end{thm}
\begin{proof}
	Please refer to Appendix \ref{sec: case1unbiased}.
\end{proof}

Similar to the non-convex objective functions, the last term of \textbf{Theorem \ref{thm: case1unbiasednonconvex}} is a convergent sequence for all three conditions of the MSE listed in Section \ref{sec: casestudy1} by \textbf{Lemma \ref{lem: stolz}}.
In addition, the following corollary for non-convex case and denoising factor \eqref{eq: precoding} can be obtained based on \textbf{Theorem \ref{thm: case1unbiasednonconvex}}.

\begin{cor}[Error Bound of \airfedavgs under Non-convexity with Constant Learning Rate]\label{cor: case1}
	When the constant learning rate is given by
	\begin{equation*}
		\eta = \frac{1}{L}\sqrt{\frac{N}{T}}\leq\frac{1}{L},
	\end{equation*}
	the minimal gradient norm of the global objective function of the unbiased \airfedavgs will be bounded by
	\begin{align}\label{eq: Case1unbiasednonconvex}\nonumber
		&\min_{t\in[T]} \mathbb{E} \left[\left \|\nabla F(\hat{\bm \theta}^{t})\right \|_2^2\right]\leq \frac{1}{T}\sum\limits_{t = 0}^{T-1} \mathbb{E} \left[\left \|\nabla F(\hat{\bm \theta}^{t})\right \|_2^2\right]\\
		\leq&\mathcal{O}\left( \frac{1+\Sigma}{\sqrt{NT}}\right) + \mathcal{O}\left(\frac{1}{d{\sf SNR}\sqrt{N^3T}}\right),
	\end{align}
	where the right hand side is dominated by $\mathcal{O}\left( \frac{1+\Sigma}{\sqrt{NT}}\right)$.
\end{cor}

\subsubsection{Comparison with \airfedavgm}
Compared with \airfedavgm ($E = 1$), we have the following corollary.
\begin{cor}[Equivalence of \airfedavgm ($E = 1$) and \airfedavgs]\label{cor: equ}
	When the denoising factor design satisfies \eqref{eq: precoding}, \airfedavgm ($E = 1$) and \airfedavgs are identical regardless of the strongly convex or non-convex objectives.
\end{cor}

\begin{proof}
	For \airfedavgm ($E = 1$), we have
	\begin{equation*}
		{\sf{MSE}}_M^t = \frac{\sigma_{w}^2}{\beta^t} = (\eta^t)^2\max_{n\in \mathcal{N}}\frac{\left\|p_n\bm g_{n}^{(t,0)}\right\|_2^2}{ |h_n^t|^2d{\sf SNR}}.
	\end{equation*}
	For \airfedavgs, we have
	\begin{equation*}
		{\sf{MSE}}_S^t = \frac{\sigma_{w}^2}{\beta^t} = \max_{n\in \mathcal{N}}\frac{\left\|p_n\bm g_{n}^{(t,0)}\right\|_2^2}{ |h_n^t|^2d{\sf SNR}}.
	\end{equation*}
	Obviously, we have 
	\begin{equation*}
		{\sf{MSE}}_M^t = (\eta^t)^2{\sf{MSE}}_S^t.
	\end{equation*}
It is worth noting that \textbf{Corollary \ref{cor: equ}} can be derived by comparing \textbf{Corollary \ref{cor: case3convex}} with \textbf{Corollary \ref{cor: case1convex}} for strongly convex case, and \textbf{Corollary \ref{cor: case3}} with \textbf{Corollary \ref{cor: case1}} for non-convex case.
\end{proof}

To conclude, for both strongly convex and non-convex cases, \airfedavgs can be generalized by \airfedavgm with denoising factor in \eqref{eq: precoding}. With a constant MSE, \airfedavgs is more robust to receiver noise than \airfedavgm. This is due to the fact that MSE is multiplied by the learning rate, which is either decaying or a small real number, compared with \airfedavgm ($E = 1$).

Furthermore, the comparisons between \airfedavgm and \airfedavgs are discussed in the following.
\begin{Rem}\label{rem: remark1}
	In \airfedavgs, the detrimental impact of the additive noise on the optimality gap for strongly convex cases and the error bound for non-convex cases is relatively small compared with \airfedavgm, as the last terms of \eqref{eq: case1unbiasedconvex} and \eqref{eq: case1unbiasednonconvex} converge with arbitrary model aggregation MSE. 
\end{Rem}

\begin{Rem}\label{rem: hetero}
	Note that \textbf{Theorems \ref{thm: case1unbiasedconvex}} and \textbf{\ref{thm: case1unbiasednonconvex}}, together with \airfedavgm ($E = 1$) in \textbf{Theorems \ref{thm: case3unbiasedconvex}} and \textbf{\ref{thm: case3unbiasednonconvex}}, do not depend on \textbf{Assumption \ref{ass: bounded gradient dissimilarity}}, which indicates that performing only one local update has the potential to avoid the impact of heterogeneous data to some extent in comparison to the periodic averaging algorithms such as \airfedavgm. This confirms the discussion in \textbf{Remark \ref{rem: airfedavgm E}}. This phenomenon is also illustrated in Fig. \ref{fig: hetero}.
\end{Rem}

\begin{Rem}\label{rem: 13comparison}
	To compare the performance of \airfedavgm and \airfedavgs, we take the case study of non-convex case for example, where the denoising factor satisfies \eqref{eq: precoder condition}. Comparing \textbf{Corollary \ref{cor: case3nonconvex}} with \textbf{Corollary \ref{cor: case1}}, it is clear that the number of communication rounds in \airfedavgm can be reduced by the number of local updates $E$, whose maximum is in the order of $\min \left\{\mathcal{O}\left(\frac{T+(2\beta_1+1)N}{(2\beta_1+1)N}\right),\mathcal{O}\left(\frac{T}{N}\right)\right \}$. In summary, compared with \airfedavgs, \airfedavgm reduces the heavy communication cost, while guaranteeing faster convergence. However, \airfedavgm has more stringent requirements for the model aggregation MSE, while those requirements are quite loose in \airfedavgs. In addition, \airfedavgm may suffer from aggregation bias caused by non-IID data, while \airfedavgs is more robust to data heterogeneity at the cost of slower convergence speed.
\end{Rem}

\begin{figure}[tbp]
	\centering
	\includegraphics[width=0.8\linewidth]{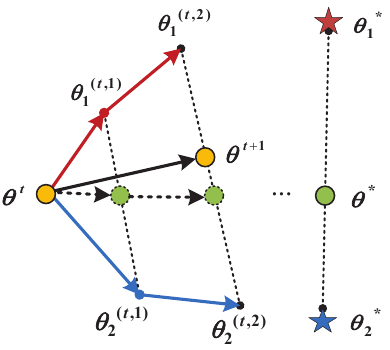}
	\caption{Model aggregation bias caused by data heterogeneity of \airfedavgm for 2 edge devices with 2 local updates, i.e., $N = 2$, and $E = 2$. The stars represent the local optima of two edge devices while the circle in green is the global optima. The circles in yellow represent the global model for \airfedavgm. The dotted circles in green illustrate the global model for fully synchronized \airfedavgs.}
	\label{fig: hetero}
	\vspace{-0.5cm}
\end{figure}

\begin{figure*}[tbp]
	\centering
	\includegraphics[width=1\linewidth]{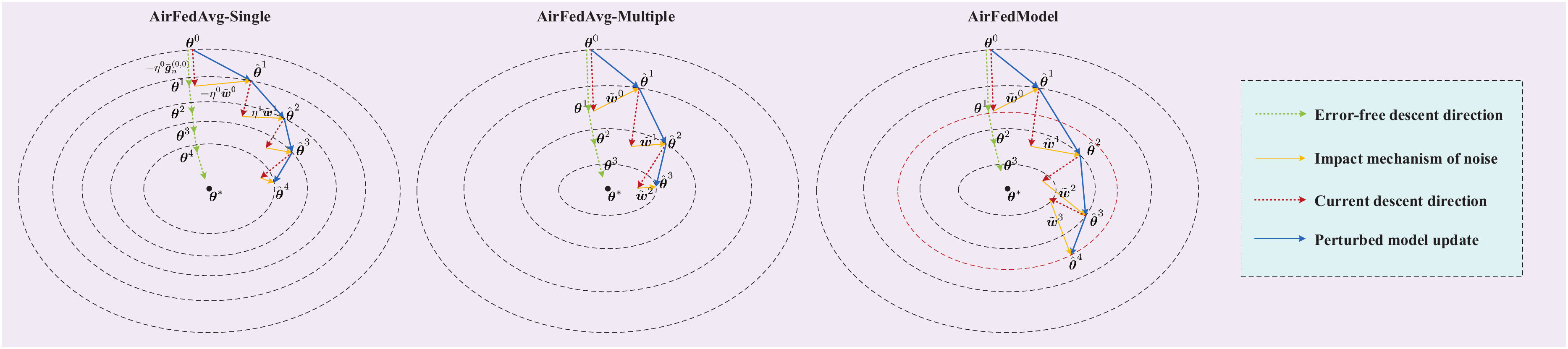}
	\caption{The impact of the additive noise on the convergence for \airfedavgs, \airfedavgm and \airfedmodel. The tighter the ellipse set is, the slower the convergence rate it represents. The dotted lines in green represent the error-free iterates of \fedavg and the dotted lines in red represent the current descent direction from the perturbed global model $\hat{\boldsymbol{\theta}}$. The yellow arrows are the impact mechanisms of the additive noise, $\tilde{\boldsymbol{w}}^t$ for \airfedavgm and \airfedmodel and  $-\eta^t\tilde{\boldsymbol{w}}^t$ for \airfedavgs. The noise vector leads to an offset on the descent direction of the initialization point and that of every global model hereafter. For \airfedavgm and \airfedavgs, existing approaches can eliminate the impact of noise as the training procedure proceeds, but not for \airfedmodel.}
	\label{fig: noiseimpact}
	\vspace{-0.5cm}
\end{figure*}

In a nutshell, we have the following conclusions for algorithm selection.
\begin{itemize}
	\item For the relatively simple learning task with strongly convex objectives, to support lightweight algorithm design, \airfedavgs is preferred with either constant or decaying learning rate to avoid non-diminishing optimality gap caused by multiple local updates, as in most recent works \cite{statisticsaware,RN37,RN86,zhu2020one,RN36,amiri2020federated,liu2020privacy,xiaowen2021optimized,sun2021dynamic,guo2021analog,wang2022IRS,liu2021reconfigurable,su2021data,guo2023dynamic}.
	\item For training non-convex DNNs and CNNs, the communication overhead is a critical issue due to the high dimensional model parameters. Since reducing the number of communication rounds between the edge devices and edge server alleviates the communication overhead, \airfedavgm algorithm is strongly recommended \cite{COTAF,wei2021federated,xiaowen2021transmission,xu2021LR,fan2021temporal,sifaou2021robust,shao2021misaligned,lin2022relay,jing2022Federated} based on the analysis above. 
	\airfedavgs is also a practical algorithm to train non-convex DNNs in some extreme scenarios, for example, highly non-IID data distribution across edge devices.
\end{itemize} 

\subsection{\airfedavg with Local Model}
As an equivalent variant of \fedavg, if the output of the local update operator is the local model parameter, i.e.,
\begin{align}\label{eq: model}
	\boldsymbol{z}_n^{t} = \bm \theta_{n}^{(t,E)},
\end{align}
the global model update becomes
\begin{align}\label{eq: AirFedModel}
	\hat{\bm{\theta}}^{t+1} = \hat{\boldsymbol{y}}^t,
\end{align}
for \airfedavg with local model (\airfedmodel). Besides, the impact mechanism of the additive noise for \airfedmodel and \airfedavgm is the same, i.e., adding the noise directly to the model aggregation step without multiplying the learning rate.

In this section, we will present the convergence results in \textbf{Corollary \ref{cor: case2unbiasedconvex}} for \airfedmodel by generalizing the analysis from Section \ref{sec: convm}.

\begin{thm}[Succession from \airfedavgm to \airfedmodel]
	In \textbf{Theorems \ref{thm: case3unbiasedconvex}} and \textbf{\ref{thm: case3unbiasednonconvex}}, we present the convergence results of \airfedavgm with strongly convex and non-convex objective functions, respectively. Due to the fact that the impact mechanism of the additive noise is identical for \airfedavgm and \airfedmodel, the convergence results of \airfedmodel is mathematically the same as that of \airfedavgm with the output of the local update operator being set as $\bm z_n^t = \bm \theta_{n}^{(t,E)}$, which reflects the flexibility and extensibility of our convergence analysis framework. 
\end{thm}

Here we take the strongly convex case with $E = 1$ and a diminishing learning rate for case study.
\begin{cor}[Convergence of \airfedmodel ($E=1$) under Strong Convexity]\label{cor: case2unbiasedconvex}
	With \textbf{Assumptions \ref{ass: smooth}, \ref{ass: strong convex}}, \textbf{\ref{ass: gradient variance}}, and \textbf{\ref{ass: bounded gradient dissimilarity}}, if the decaying learning satisfies $0<\eta^t\leq \frac{1}{L}$, then the upper bound of the optimality gap after $T$ communication rounds is given by
	\begin{align}\label{eq: case2unbiasedconvex}\nonumber
		\mathbb{E}\left[F(\hat{\bm \theta}^{T})\right] - F^{\star}&\leq \left[F(\hat{\bm \theta}^{0}) - F^{\star}\right] Q^0 + C_{21} \sum_{t=0}^{T-1} (\eta^t)^2 Q^{t+1} \\
		&+ C_3 \sum_{t=0}^{T-1} {\sf{MSE}}^tQ^{t+1},
	\end{align}
	where $Q^t = \prod _{i=t}^{T-1}(1-\frac{\mu\eta^i}{2}), Q^T=1$, $C_{21}=L\sigma^2$, and $C_3=\frac{L}{2}$.
\end{cor}

However, the optimality gap for strongly convex case and error bound for non-convex case are non-diminishing unless the MSE approaches to zero in each communication round when the model parameter is transmitted in \airfedmodel according to the analysis and results in Section \ref{sec: casestudy1}. 
In order to support this conclusion, suppose that each local model parameter is bounded by a constant, i.e.,
\begin{align*}
	\mathbb{E}[\|\bm{\theta}_n^{(t,E)}\|_2^2]\leq \Theta^2,
\end{align*}
where $\Theta>0$ is a constant, then by \eqref{eq: precoding}, the model aggregation MSE can be written as
\begin{equation*}
	{\sf{MSE}}^t = \frac{\sigma_{w}^2\tilde{\Theta}^2}{dN^2P_0} = \frac{\tilde{\Theta}^2}{dN^2{\sf SNR}},
\end{equation*}
where $\tilde{\Theta}^2 = \Theta^2\max_{n\in \mathcal{N}}\frac{v_n^2}{|h_n^t|^2}$.
This indicates that the local model has neither decaying property like the accumulated local gradients (proportional to $\eta^2$) nor has the property as \airfedavgs, where the receiver noise is multiplied by the learning rate, as depicted in Fig. \ref{fig: noiseimpact}.

In this case, if the model parameters to be aggregated are perturbed, and the aggregation bias caused by stacked perturbations cannot be eliminated over time. This results in the model parameter being off the track, and the entire training procedure will be ruined.
Then our FL model is very likely to fail to converge, unless ${\sf{MSE}}^t\rightarrow0$, which approximates to the error-free case or under an infinite large SNR. 

For strongly convex cases, the non-convergence property of \airfedmodel is reflected in the non-diminishing optimality gap. However, for non-convex DNNs and CNNs, highly precise model parameter is required in the training process, which means the training loss of \airfedmodel may diverge after only a few communication rounds.

In fact, there are few works considering transmitting local model parameter via AirComp \cite{RN35,wei2021federated,wang2022edge,guo2022Joint}. We believe that \airfedmodel is not an appropriate choice for the aforementioned reasons. We will verify our findings in Section \ref{sec: exp}.

\section{Discussions}\label{sec: extensions}
In this section, we further extend the analysis of \airfedavgm and \airfedavgs with unbiased MAE to the biased MAE cases for practical implementations. Other extensions are also discussed.

\subsection{Signal Processing Schemes}
In Section \ref{sec: system model}, we adopt channel inversion \eqref{eq: inversion} with denoising factor \eqref{eq: precoding} for signal pre-processing and post-processing. However, there are also efforts that consider other signal pre-processing and post-processing methods and precoding (denoising) factors to improve the communication efficiency.

%
\subsubsection{Normalization Methods}
To facilitate power control, the transmitted symbols can be normalized to have zero mean and unit variance. The normalization and de-normalization procedure is summarized in \cite{RN35}. This signal pre-processing scheme facilitates the MSE minimization problem considered in \cite{RN34,wang2022IRS,xu2021LR}.

\subsubsection{Multi-antenna Scenarios}

In order to make full use of the spatial multiplexing, effectively resist the influence of multipath fading, and improve communication quality, multi-antenna technology has been widely applied. Although our work mainly focuses on the single-input-single-output (SISO) scenario, it can also be extended to the single-input-multiple-output (SIMO), multiple-input-single-output (MISO), and multiple-input-multiple-output (MIMO) scenarios for AirFL \cite{xiao2023otafl}. For each time slot, the estimated signal can be generalized as 
\begin{align}
	y := \frac{1}{\sqrt{\beta}}\sum_{n=1}^{N} b_n p_n z_n + \frac{1}{\sqrt{\beta}} w.
\end{align}
With $N_t$ antennas at edge devices and $N_r$ antennas at the edge server for SIMO, MISO and MIMO systems, the settings of $P_n$ and $w$ are shown in Table \ref{tab: table3}, where $\boldsymbol{m}\in \mathbb{C}^{N_r}$ denotes the receive beamforming vector. Then, a uniform-forcing based precoder can be designed based on Table \ref{tab: table3} \cite{chen2018auniform}. 

\begin{table}[h]	
	\caption{Parameter Settings for Transceivers with Different Number of Antennas}\label{tab: table3}
	\centering\renewcommand\arraystretch{1.2}
	\begin{tabular}{c|cccc}
		\hline
		\diagbox {Value}{Scenario} & SISO & MISO & SIMO & MIMO \\ 
		\hline
		$b_n$ & $h_n\alpha_n$ & $\boldsymbol{h}_n^{\sf T}\boldsymbol{\alpha}_n$ &$\boldsymbol{m}^{\sf H}\boldsymbol{h}_n\alpha_n$ &$\boldsymbol{m}^{\sf H}\boldsymbol{H}_n\boldsymbol{\alpha}_n$ \\ \hline
		$w$ & $w$ & $ w $ &$ \boldsymbol{m}^{\sf H}\boldsymbol{w} $ & $ \boldsymbol{m}^{\sf H}\boldsymbol{w} $\\ \hline
	\end{tabular}
\end{table}

For instance, for the SIMO scenario, the precoding factor can be set as
\begin{align}
	\alpha_n = \sqrt{\beta}\frac{(\boldsymbol{m}^{\sf H}\boldsymbol{h}_n)^{\sf H}}{\|\boldsymbol{m}^{\sf H}\boldsymbol{h}_n\|_2^2},
\end{align}
where the denoising factor is set based on the case study in the previous sections
\begin{align}
	\beta  := \min_{n\in \mathcal{N}}\frac{P_0\|\boldsymbol{m}^{\sf H}\boldsymbol{h}_n\|_2^2}{ |p_n z_n|_2^2}.
\end{align}
The MSE is thus given by
\begin{align}
	{\sf MSE} := \frac{\|\bm m\|_2^2\sigma_{w}^2}{\beta} = \frac{\|\bm m\|_2^2}{{\sf SNR}}\max_{n\in \mathcal{N}}\frac{|p_n z_n|_2^2}{\|\boldsymbol{m}^{\sf H}\boldsymbol{h}_n\|_2^2}.
\end{align}
This illustrates that our convergence analysis framework can be directly applied in multi-antenna scenarios, where the receive beamforming vector design \cite{RN34} and learning rate optimization \cite{xu2021LR} by MSE minimization were considered. 

\subsubsection{Compressive Sensing Scheme}
The communication overhead is enormous as the model dimensions of the DNNs and CNNs are typically high-dimensional.
To reduce the dimension of the gradient information to be transmitted by the edge devices, compressive sensing was considered in \cite{RN36,fan2021temporal,jeon2021compressive}. In the following, we take \airfedavgm with compressive sensing for a case study. In particular, the sparsification procedure is given by
\begin{align}
	\Delta\bm{\theta}_{n,t+1}^{sp} = \operatorname{Top}_k(\Delta\bm{\theta}^{t+1}_n),
\end{align}
where operator $\operatorname{Top}_k(\cdot)$ sets all elements of $\Delta\bm{\theta}^{t+1}_n$ to zero but the first $k$ largest ones. Subsequently, a consensus mapping matrix $\mathbf{A}$ is used for compression of the sparsified vector $\Delta\bm{\theta}_{n,t+1}^{sp}$, i.e.,
\begin{align}
	\Delta\bm{\theta}_{n,t+1}^{cp} = \mathbf{A}\Delta\bm{\theta}_{n,t+1}^{sp},
\end{align}
and the local information to be transmitted is $\bm z_n^t = \Delta\bm{\theta}_{n,t+1}^{cp}$.
In this case, together with the AirComp model in Section \ref{sec: system model}, the estimated signal at the edge server is given by
\begin{align}\label{eq: estimate signal 2}\nonumber
	\hat{\bm y }_{cp}^t=&\sum_{n=1}^{N}p_n \boldsymbol z_n^t+ \tilde{\boldsymbol w}^t\\
	=&\mathbf{A}\sum_{n=1}^{N}p_n \Delta\bm{\theta}_{n,t+1}^{sp} + \tilde{\bm w}^t = \mathbf{A}\bm{x}_{sp} + \tilde{\bm w}^t.
\end{align}
The edge server can recover the signal for model aggregation by applying the approximate message passing (AMP) algorithm \cite{donoho2009message}. In view of this, the model aggregation MSE can still be used for analyzing the impact of signal estimation error on the convergence of AirFL. 
Besides, quantization was also considered in \cite{zhu2020one}
to further reduce communication overhead. When the channel inversion precoder of the edge devices is imperfect due to the hardware limitation and inaccurate channel estimation along with the imperfect synchronization, the misaligned transmit scheme proposed in \cite{shao2021misaligned} can be adopted.

\subsection{Unbiased and Biased MAE}
As various communication schemes can be applied in AirFL, different schemes cause different forms of MAE. 
According to the statistical characteristics of MAE, we can divide the communication model of AirFL into the following two categories:
\begin{itemize}
	\item \textbf{Unbiased AirFL}: The MAE is unbiased from the statistical point view, where the error satisfies \begin{equation*}
		{\sf{Bias}}^t:= \mathbb{E}\left[ \bm \varepsilon^t \right]=0, ~{\sf{MSE}}^t:= \mathbb{E}\left[\left \| \bm \varepsilon^t\right \|_2^2 \right] \neq 0,
	\end{equation*}
	This indicates that the edge server receives an unbiased estimation of the transmitted signal;
	\item \textbf{Biased AirFL}: The MAE is biased with the error satisfying \begin{equation*}
		{\sf{Bias}}^t\neq0, ~{\sf{MSE}}^t\neq 0.
	\end{equation*}
	Opposite to the unbiased case, the edge server receives a biased estimation of the transmitted signal.
\end{itemize}
We next provide evidence for the classification.
With the channel inversion model in \eqref{eq: inversion} and \eqref{eq: AirFL}, it is guaranteed that the signal estimated by the edge server is only perturbed by the Gaussian noise. In other words, the MAE is given by
\begin{align}\label{eq: error fad}
	\bm \varepsilon^{t} = \tilde{\bm w}^t \sim \mathcal{N}(0,\frac{\sigma_w^2}{\beta^t}\bm{I}_d).
\end{align}
As such, by applying channel inversion to compensate for channel fading, the signal after post-processing at the edge server is an unbiased estimation of the transmitted one, which is a special case of unbiased AirFL.

Recently, only compensating for the phase of the fading channel first and then optimizing the power control scheme in AirFL was put forward and applied in \cite{RN37,xiaowen2021optimized,paul2021accelerated,xiaowen2021transmission,wang2022inference,guo2023dynamic}. The basic idea is that the amplitude alignment among the transmitted signals of edge devices is no longer required, e.g., the precoding factor can be set as
\begin{align}\label{eq: biased factor}
	\alpha_n^t  := \sqrt{\beta_n^t}\frac{(h_n^t)^{\sf{H}}}{|h_n^t|},
\end{align}
where $\sqrt{\beta_n^t}$ is the power control factor. The estimated signal at the edge server is given by
\begin{align}\label{eq: biased signal}\nonumber
	\hat{\bm y }^t:= \frac{ \hat{\bm s}^t}{\sqrt{\beta^t}}  \overset{\eqref{eq: received signal}}{=} &\frac{1}{\sqrt{\beta^t}}\sum_{n=1}^{N}h_n^t  \alpha_n^t p_n \boldsymbol z_n^t+ \tilde{\boldsymbol w}^t\\
	=&\frac{1}{\sqrt{\beta^t}}\sum_{n=1}^{N}|h_n^t|\sqrt{\beta_n^t}p_n \bm z_n^t + \tilde{\bm w}^t.
\end{align}
Compared with the channel inversion approach that compensates for both the amplitude and phase of the channel coefficient, this scenario only regards the magnitude of the channel coefficient as aggregation weights.

For the case without magnitude alignment, as presented in \eqref{eq: biased factor} and \eqref{eq: biased signal}, the MAE is given by
\begin{align}\label{eq: error biased}
	\bm \varepsilon^{t} = \sum_{n=1}^{N} \left(\frac{|h_n^t|\sqrt{\beta_n^t}}{\sqrt{\beta^t}}-1\right) p_n \bm z_n^t + \tilde{\bm w}^t.
\end{align}
In this case, the MAE is not guaranteed to have the zero-mean, which implies the MAE is biased.
For instance, the authors in \cite{xiaowen2021optimized} designed transmit power control policies in the fading scenario to minimize the optimality gap caused by the biased MAE. 

To summarize, the biasedness of the MAE depends on the transceiver design for AirFL. 
As the convergence analysis of unbiased MAE has been conducted in previous sections, we shall present the influence caused by the biased MAE on the model convergence in the following.
In this paper, we do not compare the learning performance under biased MAE with that of unbiased MAE, but point out the dominant terms of the convergence results under biased MAE for extension.

The results for \airfedavgm under biased MAE with error $\bm \varepsilon^{t}$ are summarized in \textbf{Theorem \ref{thm: case3biasedconvex}} for strongly convex cases and \textbf{Theorem \ref{thm: case3biasednonconvex}} for non-convex cases with a diminishing learning rate, respectively.
\begin{thm}[Convergence of \airfedavgm with Biased MAE under Strong Convexity]\label{thm: case3biasedconvex}
	Let \textbf{Assumptions \ref{ass: smooth}, \ref{ass: strong convex}}, \textbf{\ref{ass: gradient variance}}, and \textbf{\ref{ass: bounded gradient dissimilarity}} hold. By setting the decaying learning rate as $0<\eta^t\leq \min\left\{\frac{1}{L\sqrt{2E(E-1)(4\beta_1+1)}},\frac{1}{4LE}\right\}$, the upper bound on the optimality gap after $T$ communication rounds is given by
	\begin{align}\label{eq: theorem4}\nonumber
		&\mathbb{E}\left[F(\hat{\bm \theta}^{T})\right] - F^{\star}\\\nonumber
		\leq& \left[F(\hat{\bm \theta}^{0}) - F^{\star}\right] J^0 + C_1^{'} \sum_{t=0}^{T-1} (\eta^t)^3 J^{t+1}\\\nonumber
		+& C_2^{'} \sum_{t=0}^{T-1} (\eta^t)^2 J^{t+1} 
		+  \underbrace{\sum_{t=0}^{T-1} \eta^t \left(C_3^{'}+2L^2E {\sf{MSE}}^t \right)J^{t+1}}_{(d)} \\
		+&  \underbrace{C_4^{'} \sum_{t=0}^{T-1} {\sf{MSE}}^t  J^{t+1}}_{(e)}
		+ \underbrace{C_5^{'}\sum_{t=0}^{T-1} \frac{\left \| {\sf{Bias}}^t \right \|_2^2J^{t+1}}{\eta^t}}_{(f)},
	\end{align}
	where $J^t = \prod _{i=t}^{T-1}(1-\frac{\mu\eta^iE}{4})$, $J^T=1$, $C_1^{'}=\frac{L^2E(E-1)(4\beta_1+1) }{8\beta_1} \left[\bar{\sigma}^2 + 2E\beta_2\right]$, $C_2^{'}=LE\sigma^2$, $C_3^{'}=\frac{1}{4}\sigma^2$, $C_4^{'}=\frac{L}{2}$ and $C_5^{'}=\frac{1}{E}$.
\end{thm}
\begin{proof}
	Please refer to Appendix \ref{sec: case3biased} for details.
\end{proof}
\begin{thm}[Convergence of \airfedavgm with Biased MAE under Non-convexity]\label{thm: case3biasednonconvex}
	With \textbf{Assumptions \ref{ass: smooth}}, \textbf{\ref{ass: gradient variance}}, and \textbf{\ref{ass: bounded gradient dissimilarity}}, if the decaying learning rate satisfies $0<\eta^t\leq \min\left\{\frac{1}{L\sqrt{2E(E-1)(4\beta_1+1)}},\frac{1}{4LE}\right\}$, then the weighted average norm of global gradients after $T$ communication rounds is upper bounded by
	\begin{align}\label{eq: theorem5}\nonumber
		&\frac{1}{\Phi}\sum\limits_{t = 0}^{T-1} \eta^t\mathbb{E} \left[\left \|\nabla F(\hat{\bm \theta}^{t})\right \|_2^2\right]\\\nonumber
		\leq& \frac{8 \left[F(\hat{\bm \theta}^{0}) - F^{\inf}\right] }{E\Phi}
		+  \frac{8 C_1^{'}}{E\Phi}\sum_{t = 0}^{T-1} (\eta^t)^3
		+  \frac{8 C_2^{'} }{E\Phi}\sum_{t = 0}^{T-1} (\eta^t)^2\\\nonumber
		+& \underbrace{ \frac{8  }{E\Phi}\sum_{t = 0}^{T-1} \eta^t \left(C_3^{'}+2L^2E{\sf{MSE}}^t\right)}_{(d)}+ \underbrace{\frac{ 8 C_4^{'}}{E\Phi} \sum_{t=0}^{T-1} {\sf{MSE}}^t}_{(e)}\\
		+& \underbrace{\frac{8C_5^{'}}{E\Phi}\sum_{t=0}^{T-1} \frac{\left \|{\sf{Bias}}^t\right \|_2^2}{\eta^t }}_{(f)}.
	\end{align}
\end{thm}
\begin{proof}
	Please refer to Appendix \ref{sec: case3biased}.
\end{proof}

\begin{Rem}\label{rem: remarkbiased}
	\textbf{Theorems \ref{thm: case3biasedconvex}} and \textbf{\ref{thm: case3biasednonconvex}} focus on the convergence results of the strongly convex and non-convex loss functions, respectively. We have the following key observations: term (d) is caused by gradient estimation variance and the model aggregation MSE; term (e) is caused by the model aggregation MSE; term (f) results from the biased MAE. The upper bound is a function of both the model aggregation bias and MSE.
\end{Rem}

According to \textbf{Lemmas \ref{lem: sequence}} and \textbf{\ref{lem: stolz}}, for any communication model with biased MAE, both the squared norm of bias, i.e., $\left \|{\sf{Bias}}^t\right \|_2^2$, and $ {\sf{MSE}}^t $ need to be designed or minimized to improve the model training performance of \airfedavgm for strongly convex and non-convex loss functions.
One possible solution is to jointly minimize $\left \|{\sf{Bias}}^t\right \|_2^2$ and $ {\sf{MSE}}^t $. Without loss of generality, the optimization problem can be established as
\begin{align}\label{eq: problem biased3}
	\mathscr{P}_1: 
	\begin{array}{ll}
		\text { minimize } &\lambda {\sf{MSE}}^t+(1-\lambda)\left \|{\sf{Bias}}^t\right \|_2^2\\
		\text { subject~to } & 0\leq\beta_n^t\leq P^{\max}_n, \forall n \in \mathcal{N},  t \in [T],
	\end{array}
\end{align}
where $P^{\max}_n$ is the maximum transmit power for edge device $n$.
This key observation is in line with the work in \cite{xiaowen2021transmission}. Another possible solution is to ensure these two variables satisfy the following condition with learning rate decay:
\begin{equation}
	{\sf{MSE}}^t\sim (\eta^t)^{\rho_1},~{\sf{Bias}}^t\sim (\eta^t)^{\rho_2},~\rho_1,\rho_2>1.
\end{equation}
However, this method may be difficult to be applied in practical systems.

To conclude, we mainly discuss the convergence results and intuitions for \airfedavgm with biased MAE, which are corresponding to the analysis in Section \ref{sec: convm}. In comparison, when the MAE is unbiased, the upper bound is a function of the model aggregation MSE. Otherwise it is related to both the squared norm of bias and the MSE.

Besides, the convergence results for \airfedavgs with biased MAE are summarized in \textbf{Theorem \ref{thm: case1biasedconvex}} for strongly convex cases and \textbf{Theorem \ref{thm: case1biasednonconvex}} for non-convex cases, respectively.

\begin{thm}[Convergence of \airfedavgs with Biased MAE under Strong Convexity]\label{thm: case1biasedconvex}
	With \textbf{Assumptions \ref{ass: smooth}, \ref{ass: strong convex}} and \textbf{\ref{ass: gradient variance}}, if the decaying learning rate with $0<\eta^t\leq\frac{1}{4L}$, then the upper bound on the cumulative gap after $T$ communication rounds is given by
	\begin{align}\nonumber
		&\mathbb{E}\left[F(\hat{\bm \theta}^{T})\right] - F^{\star}\\\nonumber
		\leq& \left[F(\hat{\bm \theta}^{0}) - F^{\star}\right] Q^0 + \frac{1}{2} \sum_{t=0}^{T-1} \eta^t \left \|{\sf{Bias}}^t\right \|_2^2 Q^{t+1}\\
		+& \frac{L}{2} \sum_{t=0}^{T-1} (\eta^t)^2\left[ {\sf{MSE}}^t + \left \|{\sf{Bias}}^t\right \|_2^2+ \sigma^2 \right] Q^{t+1}.
	\end{align}
\end{thm}
\begin{proof}
	Please refer to Appendix \ref{sec: case3biased}.
\end{proof}
\begin{thm}[Convergence of \airfedavgs with Biased MAE under Non-convexity]\label{thm: case1biasednonconvex}
	With \textbf{Assumptions \ref{ass: smooth}}, \textbf{\ref{ass: gradient variance}}, and \textbf{\ref{ass: bounded gradient dissimilarity}}, if the decaying learning rate satisfies $0<\eta^t\leq \frac{1}{4L}$, then the weighted average norm of global gradients after $T$ communication rounds is upper bounded by
	\begin{align}\nonumber
		&\frac{1}{\Phi}\sum\limits_{t = 0}^{T-1} \eta^t\mathbb{E} \left[\left \|\nabla F(\hat{\bm \theta}^{t})\right \|_2^2\right]\\\nonumber
		\leq& \frac{4}{\Phi} \left[F(\hat{\bm \theta}^{0}) - F^{\inf}\right] +\frac{2}{\Phi}\sum_{t=0}^{T-1} \eta^t\left \|{\sf{Bias}}^t\right \|_2^2\\
		+&\frac{2L}{\Phi} \sum_{t=0}^{T-1} (\eta^t)^2\left[{\sf{MSE}}^t + \left \| {\sf{Bias}}^t \right \|_2^2+ \sigma^2 \right].
	\end{align}
\end{thm}
\begin{proof}
	Please refer to Appendix \ref{sec: case3biased}.
\end{proof}

Combining the above two theorems, we notice that the upper bound is dominated by the square norm of bias, i.e., $\left \|{\sf{Bias}}^t\right \|_2^2$, which converges to an error floor if $\left \|{\sf{Bias}}^t\right \|_2^2$ is a constant. Compared to \airfedavgs with unbiased MAE, the main challenge is to minimize $\left \| {\sf{Bias}}^t \right \|_2^2$ for achieving better training performance of \airfedavgs with biased MAE for both strongly convex and non-convex objectives. Our observation is in line with the work in \cite{xiaowen2021optimized}, which focuses on \airfedavgs with biased MAE. In addition, another possible solution with learning rate decay is that
\begin{equation}
	{\sf{Bias}}^t\sim (\eta^t)^{\rho},~\rho>0.
\end{equation}
It is clear that with biased MAE, \airfedavgs is different from \airfedavgm in the dominant term of the optimality gap for strongly convex objectives and the error bound for non-convex ones. 

In a nutshell, we discuss the practical signal processing approaches and their impact on MAE in this section. Then we provide convergence results for \airfedavgm and \airfedavgs with biased MAE.

\subsection{Extension to Other Optimization Methods}
\subsubsection{Different Local Optimizers}
This paper mainly considers first-order optimization methods for local model updating, i.e., vanilla SGD. In fact, the local optimizer can be extended to other gradient-based methods, e.g., SGD with proximal gradients, momentum \cite{wang2020tackling,paul2021accelerated}.
Moreover, 
ideas and insights obtained in this paper can also be applied in AirFL with zeroth-order optimization \cite{fang2022communication} and second-order optimization \cite{hua2019secondorder,yang2022secondorder}.
\subsubsection{Different Optimization Methods}
The aforementioned optimization methods are all primal. AirFL with Peaceman-Rachford splitting was first considered in \cite{xia2020fast}, which can be classified as a primal-dual method, to address the non-optimality of \airfedavgm with a constant learning rate. By applying the similar convergence analysis approach, the illustration of bounding the MAE is given in \cite[\textbf{Theorem 2}]{xia2020fast}. Without gradient information, it is worth investigating that whether transmitting local model parameter via AirComp leads to convergence after multiple communication rounds.

\subsection{Extension to Other Network Architectures}
\subsubsection{Decentralized Networks}
This paper considers a star-topology network consisting of a central edge server, and $N$ distributed edge devices, where the device-to-server (D2S) communication links are dominated. To alleviate the communication overhead caused by D2S links, decentralized networks exploiting device-to-device (D2D) links with an arbitrary topology has recently attracted a lot of attention \cite{jiang2022graph}, e.g., decentralized (gossip) SGD \cite{chen2021accelerating,koloskova2020unified} and hybrid SGD \cite{guo2022hybrid}. Besides, AirFL in wireless D2D networks was studied in \cite{xing2021D2D}.
\subsubsection{Multi-cell Networks}
Due to the mobility of the edge devices and the rapid development in connected intelligence, the co-existence of different FL tasks in multi-cell networks will become a new paradigm in the future. \cite{wang2022inference} extends AirFL to multi-cell networks with different learning tasks for inter-cell interference management. In addition, End-Edge-Cloud hierarchical FL architecture has been proposed to exploit more data while reducing the costly communication with the cloud \cite{liu2020client,abad2020hierarchical}. Convergence analysis for hierarchical AirFL has been studied in \cite{aygun2022over}.

\subsubsection{Other Networks}
To overcome the detrimental effect of channel fading in wireless networks, RIS has been proposed to reconfigure the propagation channels and thus support reliable model aggregation of AirFL \cite{2020yangRIS,liu2021reconfigurable,wang2022IRS,yang2022differentially}. In particular, the convergence of RIS-assisted AirFL can still be covered in this tutorial, which only leads to different expressions for MSE or bias.
Moreover, unlike static networks considered in this tutorial, UAV or satellite aided FL systems \cite{fu2022UAV,fu2022uavmulti,fu2022federated,razmi2022ground,zhou2023towards}, which consider mobile edge server or edge devices, lead to new challenges in terms of communication schemes and algorithm design. There are also open problems and interesting directions for future study.

\begin{figure*}[tbp]
	\centering
	\subcaptionbox{Fixed SNR $=5$dB\label{fig: result1a}}{\includegraphics[width=0.49\linewidth]{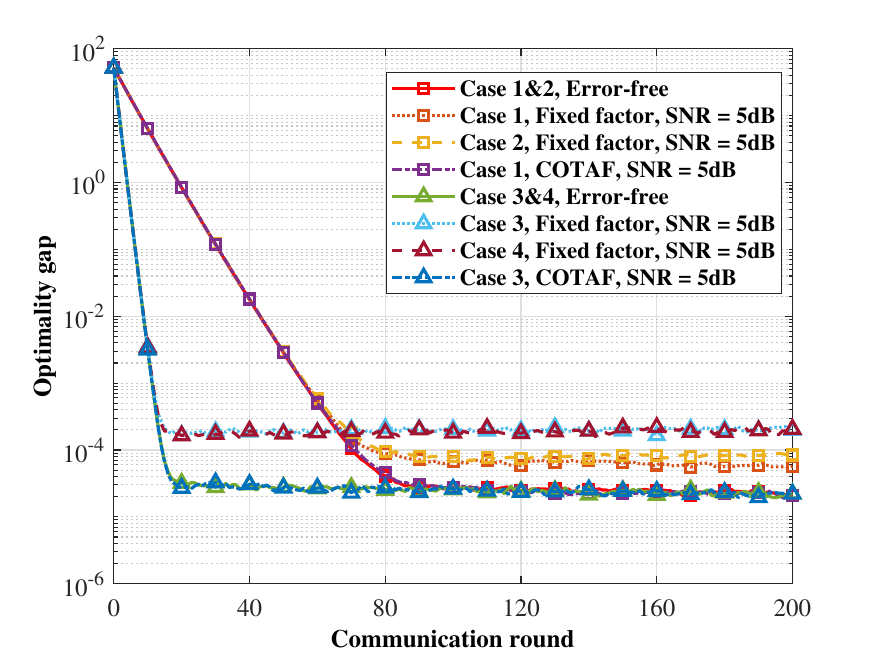}}\hfil
	\subcaptionbox{Fixed SNR $=0$dB \label{fig: result1b}}{\includegraphics[width=0.49\linewidth]{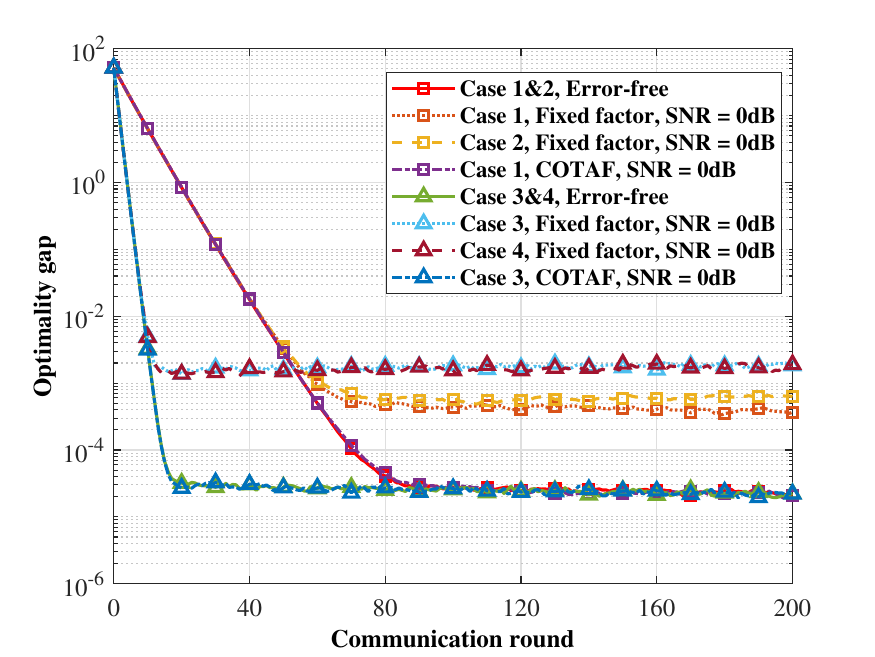}}\hfil
	\subcaptionbox{Case 1 $\&$ 2\label{fig: result1c}}{\includegraphics[width=0.49\linewidth]{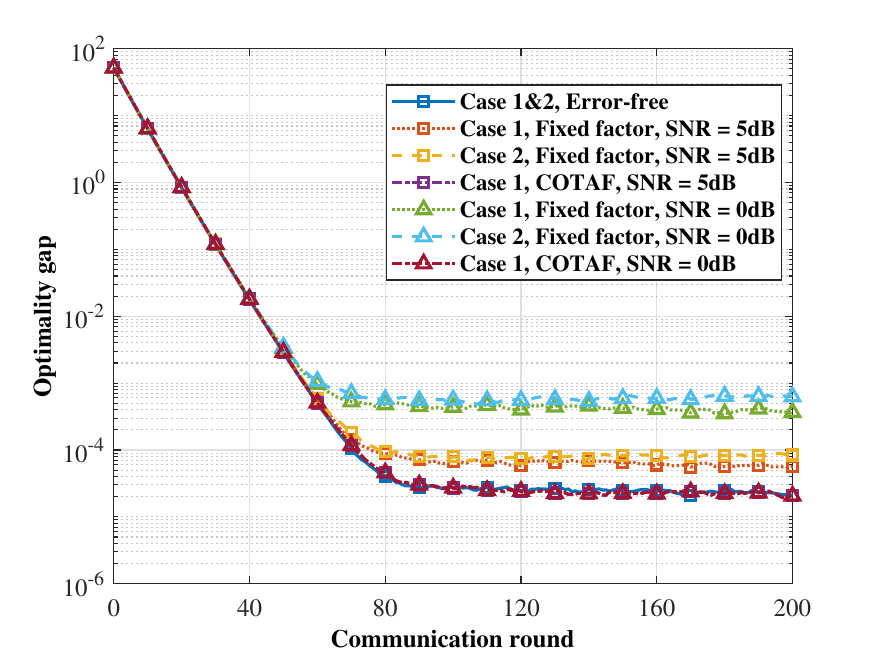}}\hfil
	\subcaptionbox{Case 3 $\&$ 4 \label{fig: result1d}}{\includegraphics[width=0.49\linewidth]{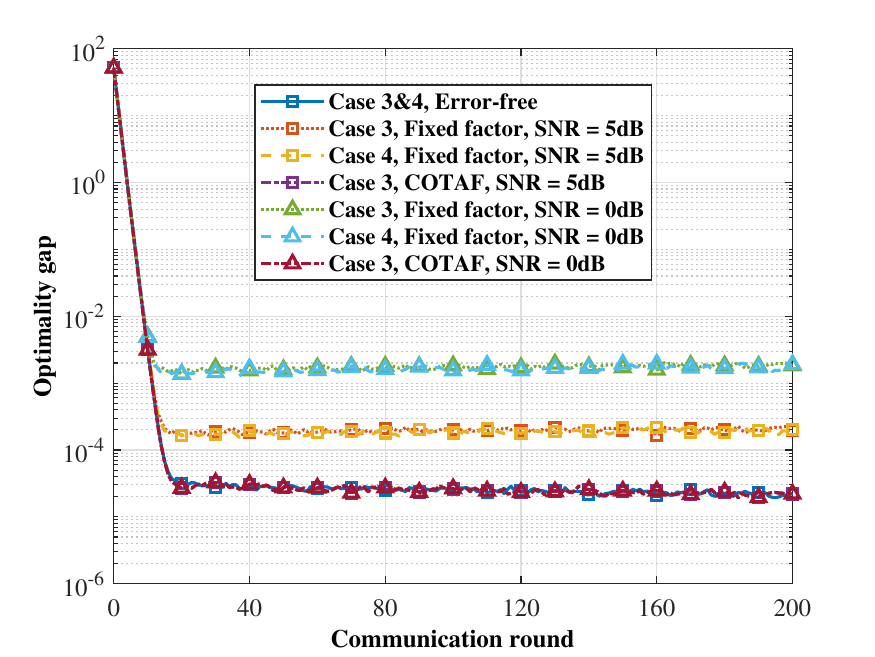}}\hfil
	\caption{Simulation results for linear regression on the synthetic dataset.}
	\label{fig: result1}
	\vspace{-0.5cm}
\end{figure*}
\subsection{Extension to Partial Device Participation}\label{sec: extensions and future work}
This paper mainly considers the convergence of \airfedavg with full device participation. However, due to the limited communication resources and the heterogeneous channel conditions, only a part of edge devices may participate in the model training. For instance, in \cite{COTAF,xia2020fast,fang2022communication}, the authors adopted a threshold-based method as described in \eqref{eq: truncated} to filter out edge devices with poor channel conditions, leading to partial device participation. The convergence of \airfedavg with partial device participation is worth further study.

Partial device participation can be divided into the following three cases. 
\begin{itemize}
	\item Non-uniform sampling of edge devices. The edge server actively schedules a part of edge devices before the global model dissemination non-uniformly. Different scheduling policies based on edge devices' computation capability, channel condition, norm of local model, beamforming, or contribution can be applied to select a subset of edge devices \cite{yang2019scheduling,ren2020scheduling,RN34,ConvergenceTime,AJoint,JointDevice,guo2022Joint,ali2022matching}.
	\item Uniform sampling of edge devices. In this category, the probability of each edge device participating in the training process in each communication round is uniform \cite{Mohamed2021privacy}.
	\item Arbitrary participation of edge devices. Edge devices participate in the model training in an arbitrary way, i.e., arriving and leaving the system without grant of edge server, which makes the convergence of FL algorithm more complicated \cite{ruan2021towards,gu_fast_2021,yang2022anarchic}.
\end{itemize}

\begin{figure*}[tbp]
	\centering
	\subcaptionbox{Loss of \airfedavgs and \airfedmodel ($E$ = 1)\label{fig: result2a}}{\includegraphics[width=0.49\linewidth]{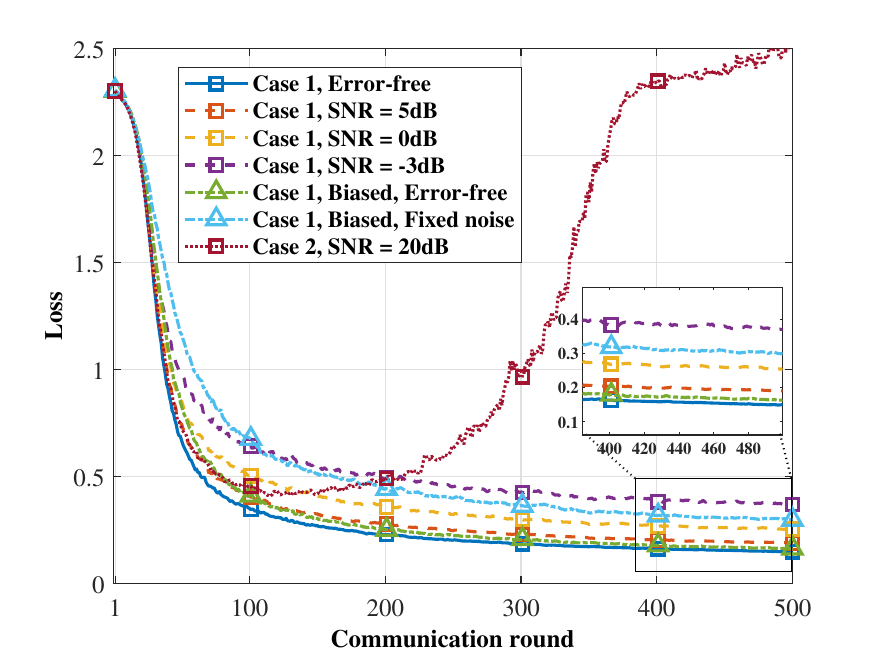}}\hfil
	\subcaptionbox{Test accuracy of \airfedavgs and \airfedmodel ($E$ = 1)\label{fig: result2b}}{\includegraphics[width=0.49\linewidth]{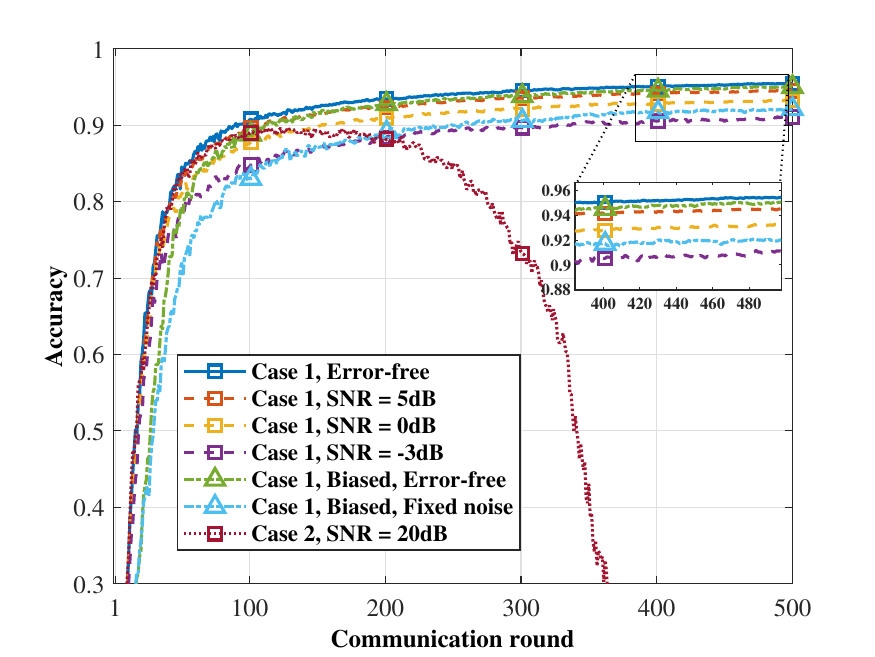}}\hfil
	\subcaptionbox{Loss of \airfedavgm and \airfedmodel ($E$ = 5)\label{fig: result2c}}{\includegraphics[width=0.49\linewidth]{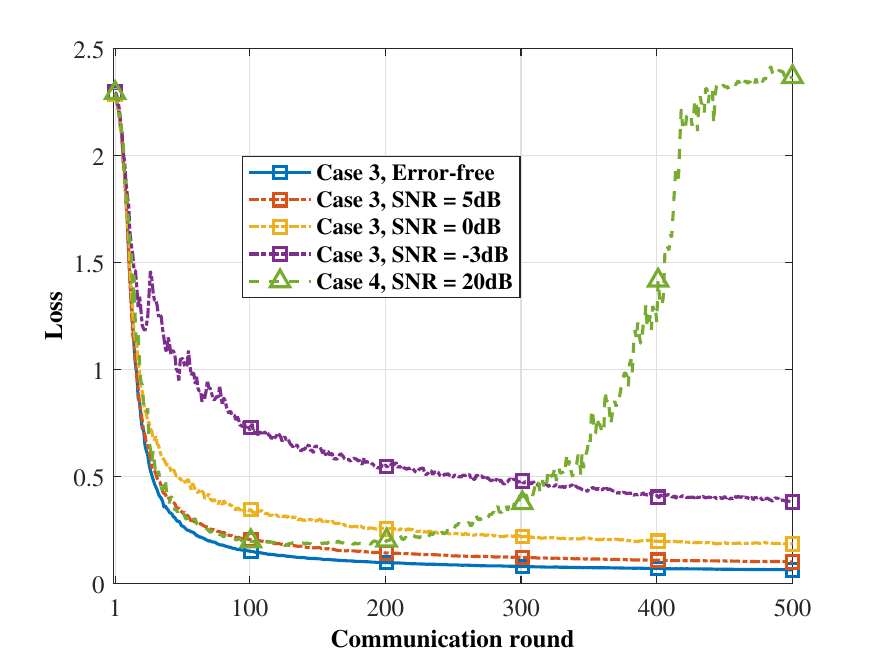}}\hfil
	\subcaptionbox{Test accuracy of \airfedavgm and \airfedmodel ($E$ = 5)\label{fig: result2d}}{\includegraphics[width=0.49\linewidth]{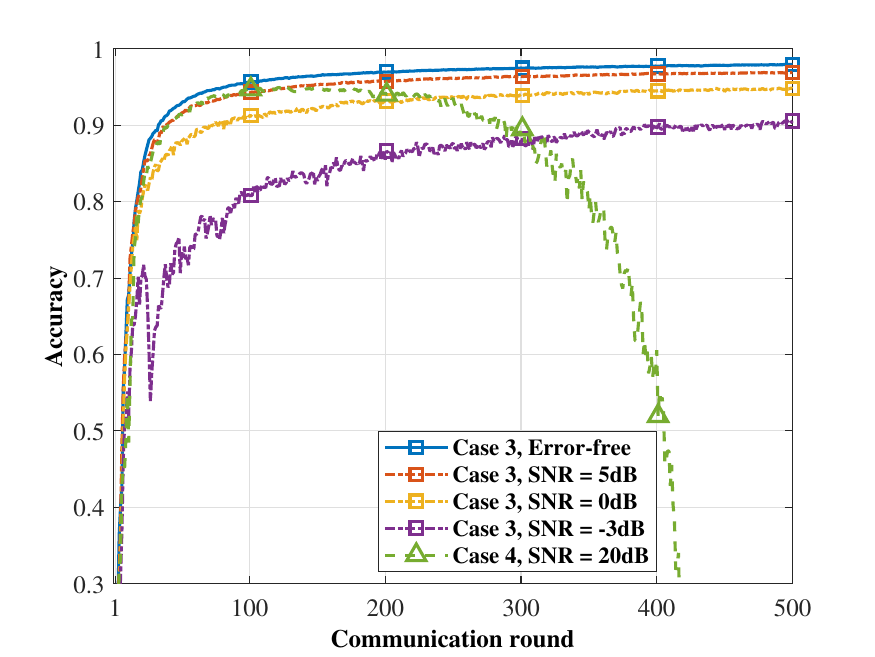}}\hfil
	\caption{Simulation results for the CNN model on the MNIST dataset over AWGN channel.}
	\label{fig: result2}
	\vspace{-0.5cm}
\end{figure*}

\subsection{Other Potential Extensions}
Firstly, although AirFedModel is not convergent in low SNR scenarios from our analysis, a potential research direction is to provide convergence guarantees for this algorithm.


Secondly, while this paper considers that all edge devices perform an equal number of local updates, the heterogeneity in computation capabilities leads to variations in the number of local epochs across edge devices in each communication round. Furthermore, each edge device's number of local epochs can also vary across different communication rounds. This brings new challenges to the convergence analysis and system design of AirFL \cite{wang2020tackling,mao2022charles,yang2022over}, which is an attractive extension.

Finally, data heterogeneity may severely degrade the overall learning performance \cite{zhao2018federated}. To handle non-IID data, personalization is a leading approach for each edge device to train a personalized local model, which mainly consists of multi-task FL \cite{li2021ditto} and meta-learning \cite{fallah2020personalized}. This technique has been applied in AirFL \cite{sami2022air}. The performance of other solutions such as \cite{FedProx}, which introduces a proximal term to the local objectives to minimize the model divergence, also worth investigating in the presence of wireless fading channels.

\section{Numerical Evaluations}\label{sec: exp}
In this section, we evaluate the performance of \airfedavgm, \airfedavgs, and \airfedmodel with both strongly convex and non-convex loss functions. All experiments are averaged over 5 independent runs. For simplicity, we use Case 1$-$4 to represent \airfedavgs, \airfedmodel ($E=1$), \airfedavgm and \airfedmodel ($E>1$), respectively.

\subsection{Strongly Convex Case: Linear Regression on Synthetic Dataset}
We consider a linear regression problem for the strongly convex case on the synthetic dataset.
In particular, the dataset of each edge device $\mathcal{D}_n$ is randomly generated by
\begin{align}\label{eq: linear model}
	\bm{b}_{n} = \bm{A}_{n}\bm x_{0}+v_{n},~\forall n \in \mathcal{N},
\end{align}
where $\bm x_{0} \in \mathbb{R}^d$ is a random parameter with standard Gaussian distribution $\mathcal{N}(0, \boldsymbol{I}_{d})$ and the sample-wise noise vectors are independently generated as $v_{n}\overset{ind.}{\sim}\mathcal{N}(0, \sigma^2 \bm{I}_{D_{n}})$. Random matrices $\bm{A}_n \in \mathbb{R}^{D_{n} \times d}$ is generated by $(\bm{A}_n )_{uv} \overset{i.i.d}{\sim} \mathcal{N}(0, 1)$, for all $n \in \mathcal{N}$, $u \in [D_{n}]$ and $v \in [d]$. 

The local loss function for each edge device $n$ is given by
\begin{align}\label{eq: loss function}
	F_{n}(\bm\theta) = \frac{1}{2D_n}\left\|\bm{A}_n \bm\theta - \bm{b}_n \right \|^2_2,
\end{align}
which is strongly convex.

\begin{figure*}[tbp]
	\centering
	\subcaptionbox{Loss\label{fig: result3a}}{\includegraphics[width=0.49\linewidth]{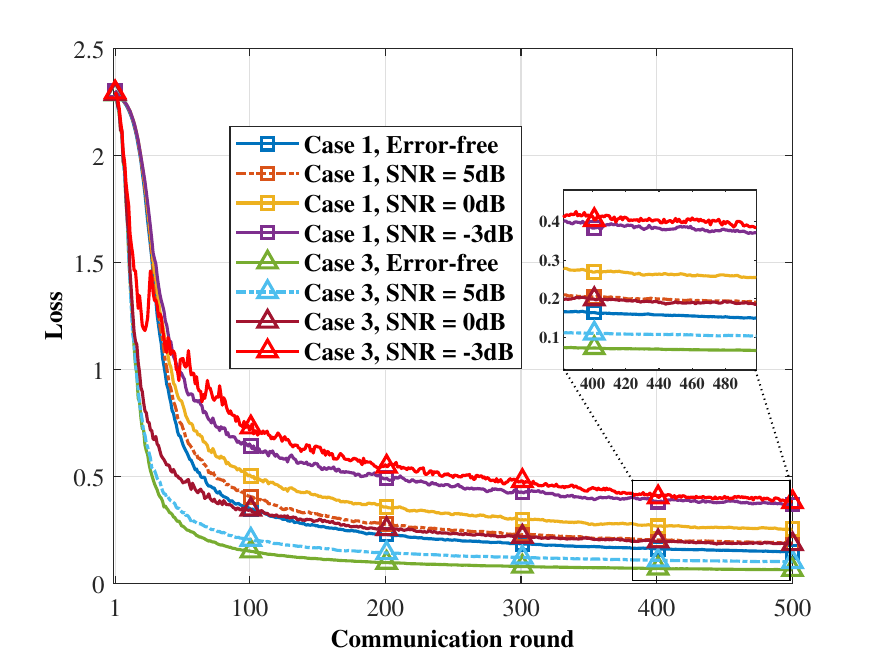}}\hfil
	\subcaptionbox{Test accuracy\label{fig: result3b}}{\includegraphics[width=0.49\linewidth]{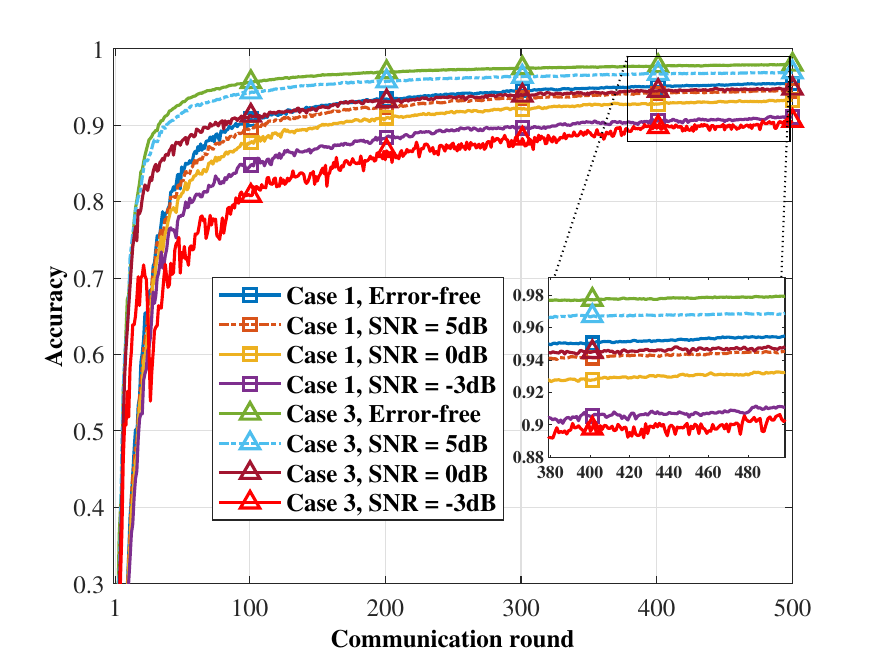}}\hfil
	\caption{Comparison of \airfedavgm and \airfedavgs for the CNN model on the MNIST dataset over AWGN channel.}
	\label{fig: result3}
	\vspace{-0.5cm}
\end{figure*}

\begin{figure*}[tbp]
	\centering
	\subcaptionbox{Loss\label{fig: result4a}}{\includegraphics[width=0.49\linewidth]{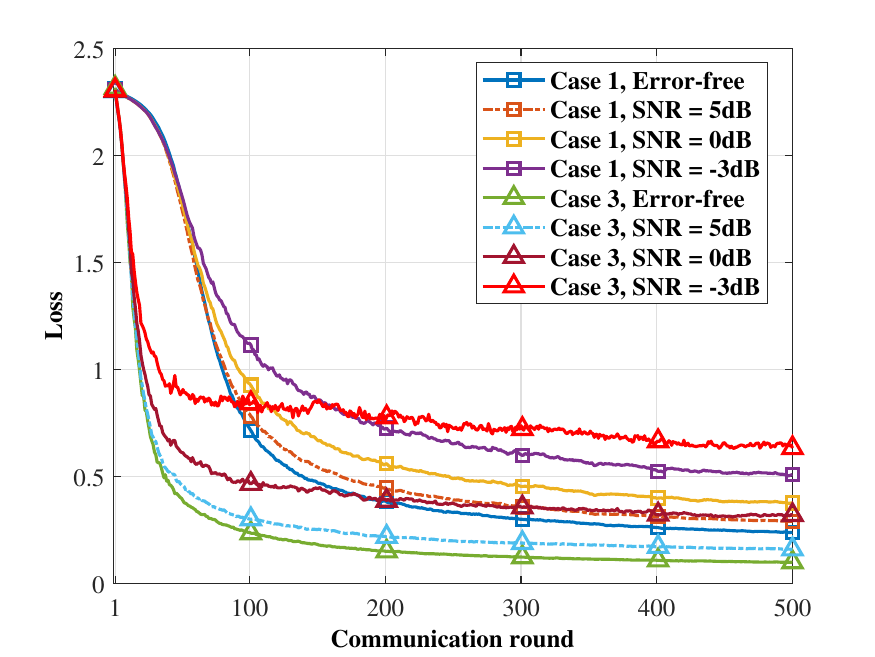}}\hfil
	\subcaptionbox{Test accuracy\label{fig: result4b}}{\includegraphics[width=0.49\linewidth]{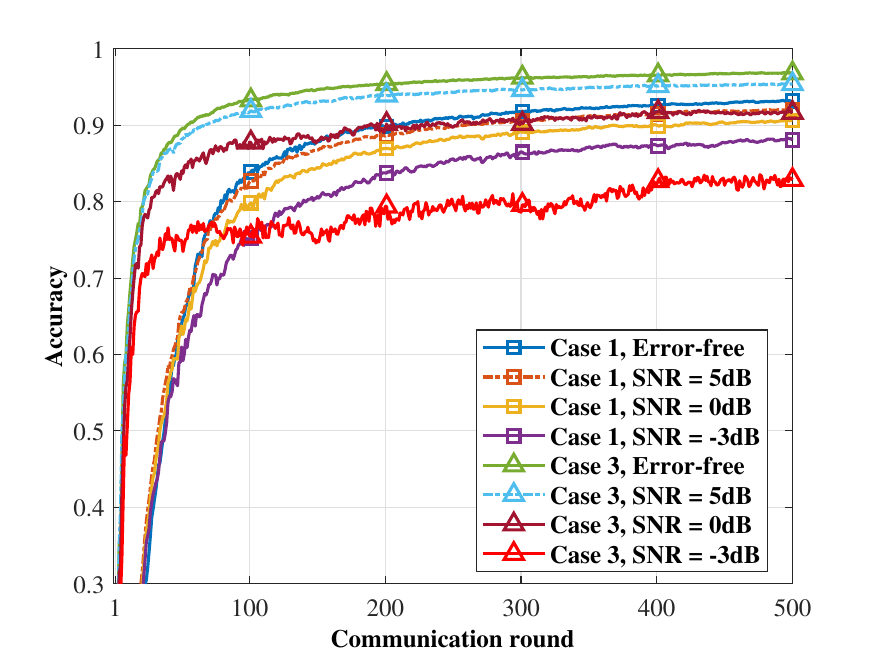}}\hfil
	\caption{Comparison of \airfedavgm and \airfedavgs for the CNN model on the MNIST dataset over Rayleigh fading channel.}
	\label{fig: result4}
	\vspace{-0.5cm}
\end{figure*}

We numerically evaluate the optimality gap $\mathbb{E}[F(\hat{\bm \theta})] - F(\bm{\theta}^{\ast})$ with respect to $T=200$ communication rounds in the following benchmarks in AWGN channel for comparison:
\begin{itemize}
	\item \textbf{Error-free \airfedavgm and \airfedavgs:} Conventional FL without channel fading and noise, where the number of local updates is $E=5$ for \airfedavgm.
	\item \textbf{AirFedAvg with Fixed Precoding Factor:} The precoding factor is fixed since the starting round by 
	\begin{align}
		\alpha_n^t = \frac{\sqrt{dP_0}}{\max_{n\in \mathcal{N}}\|p_n\bm{z}_n^0\|_2},~t\in[T].
	\end{align}
	\item \textbf{\airfedavgm and \airfedavgs with Designed Precoding Factor:} The precoding factor is designed as 
	\begin{align}\label{eq: beta}
		\alpha_n^t = \frac{\sqrt{dP_0}}{\max_{n\in \mathcal{N}}\|p_n\bm{z}_n^t\|_2},~t\in[T],
	\end{align}
which is named COTAF for abbreviation \cite{COTAF}.
\end{itemize}

Consider there are $N=25$ edge devices and the model with dimension $d = 100$. The size of local dataset satisfies $D_n\in[D_{\min},D_{\max}]$ with mean $\bar{D}=500$, where $D_{\min}=300$ and $D_{\max}=1200$. The local batch size is 128 and we evaluate the convergence performance under different SNR, where $\text{SNR}=P_0/\sigma_{w}^2$.

Here we set the parameters as
\begin{align*}
	\begin{array}{ll}
		\sigma^2 = 0.20, &\text{SNR}\in \{5\text{dB}, 0\text{dB}\},\\
		\eta^0 = 0.1, & \eta^t = \frac{\eta^0}{1+0.002\times t}.
	\end{array}
\end{align*}

Experimental results are illustrated in Fig. \ref{fig: result1}. The observations made from Fig. \ref{fig: result1} can be concluded as follows:
\begin{itemize}
	\item From Figs. \ref{fig: result1a} and \ref{fig: result1b}, it is clear that under fixed SNR, \airfedavgs and \airfedavgm with COTAF precoder have the same performance as the error-free setting. 
	\item For fixed SNR and fixed precoding factor, the performance is worse compared with designed COTAF precoder for all cases. Meanwhile, \airfedavgs has a smaller optimality gap than \airfedmodel and \airfedavgm, which is in line with our analysis in \textbf{Remark \ref{rem: remark1}} that the \airfedavgs is less sensitive to the model aggregation MSE, i.e., the denoising factor design and \airfedavgm is on the contrary.
	\item From Figs. \ref{fig: result1c} and \ref{fig: result1d}, it is obvious that the lower the SNR is, the bigger the optimality gap for all cases but COTAF, which means this precoding scheme is more robust to noisy channel and is capable of eliminating the impact of receiver noise through the training process. 
	\item Subsequently, \airfedavgm converges faster with $E=5$ local updates than \airfedavgs to achieve linear speedup for fewer communication rounds as analyzed in \textbf{Remark \ref{rem: 13comparison}}.
\end{itemize}

\subsection{Non-Convex Case I: CNN on MNIST Dataset}
Next we validate our theorems and remarks by considering handwritten digits classification on the MNIST dataset \cite{MNIST}, which consists of 60,000 training images and 10,000 images for testing with $28\times28$ pixels of 10 digits. The training dataset is uniformly distributed over $N = 50$ edge devices. We assume non-IID data distribution, where the original training dataset is first sorted by labels, and then randomly assigned to all edge devices. Each edge device can access only two sorts of labels.

\begin{figure*}[tbp]
	\centering
	\subcaptionbox{Loss\label{fig: result5a}}{\includegraphics[width=0.49\linewidth]{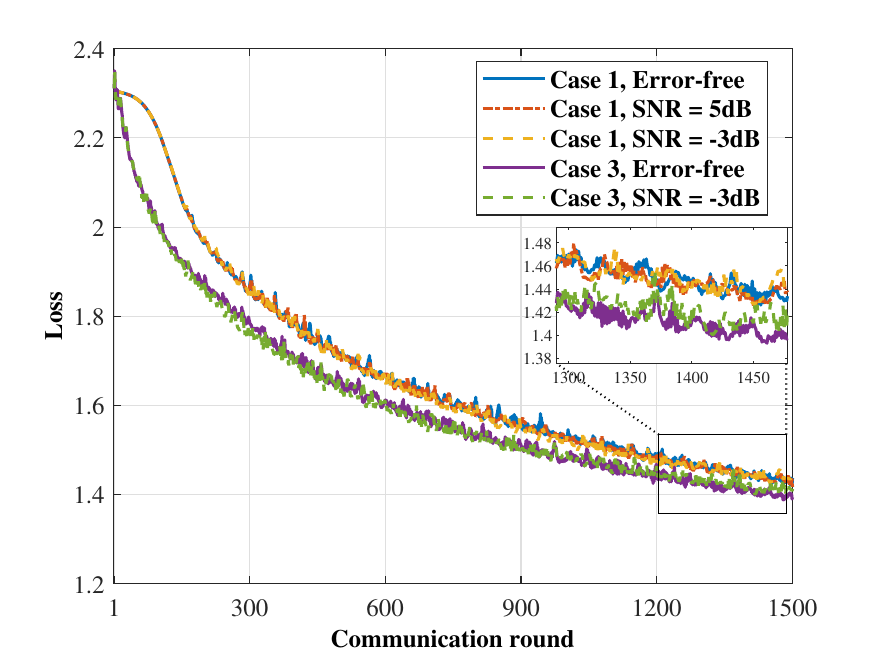}}\hfil
	\subcaptionbox{Test accuracy\label{fig: result5b}}{\includegraphics[width=0.49\linewidth]{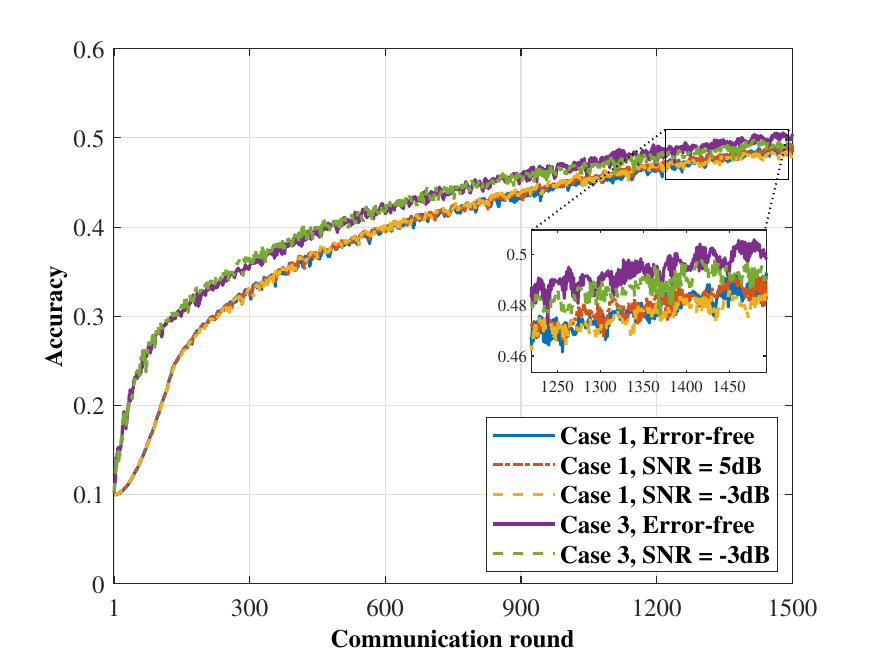}}\hfil
	\caption{Comparison of \airfedavgm and \airfedavgs for the CNN model on the CIFAR-10 dataset over Rayleigh fading channel.}
	\label{fig: result5}
	\vspace{-0.5cm}
\end{figure*}

For the learning model, we consider the non-convex CNN model, which has two $5\times5$ convolution layers, a full connected layers with 320 units and ReLU activation, and a final output layer with softmax. The first convolution layer has 10 channels and the second one has 20 channels, which are both followed by ReLU activation and $2\times2$ max pooling. The total number of model parameters is $d=21840$ and the number of local updates for \airfedavg is set as $E=5$ unless specified otherwise. The local batch size is 10 and the learning rate is initially set as $\eta^0 = 0.1$ for AWGN channel and $\eta^0 = 0.05$ for Rayleigh fading channel and decays with $ \eta^t = \frac{\eta^0}{1+0.005\times t} $ in each communication round and the SNR is selected from the set $\{-3\text{dB},0\text{dB},5\text{dB}\}$. We run the experiments on PyTorch version 1.10.2 with python version 3.6.

We consider the loss and test accuracy over $T=500$ communication rounds of the following benchmarks for comparison:
\begin{itemize}
	\item \textbf{Error-free \airfedavgm and \airfedavgs:} Conventional FL without channel fading and noise.
	\item \textbf{\airfedavgm and \airfedavgs with Designed Precoding Factor:} The precoding factor is described in \eqref{eq: beta} in AWGN channel and \eqref{eq: precoding} for Rayleigh fading channel.
	\item \textbf{\airfedmodel over High SNR with Designed Precoding Factor:}
	For AWGN channel, the SNR is set to be up to 20dB to verify the divergence of Case 2 and 4.
	\item \textbf{Biased MAE for \airfedavgs:} The power control policy proposed in \cite{xiaowen2021optimized} for gradient transmission with noise power $\sigma_w^2 = 0.01$ for fixed noise case and $\sigma_w^2 = 0$ for error-free case.
\end{itemize}

\begin{table}[h]	
	\caption{Highest achievable test Accuracy for AWGN channel}\label{tab: table4}
	\centering\renewcommand\arraystretch{1.1}
	\begin{tabular}{ccccc}
		\hline MAE & Scheme & \text Case 1  &Case 3 ($E$=5) & Case 3 ($E$=10)\\
		\hline  \multirow{4}{*}{Unbiased} & Error-free  & 95.5\% & 98.0\% & 98.5\% \\
		& SNR = 5dB  & 94.5\% & 96.9\% &  97.6\%\\
		& SNR = 0dB  & 93.3\% & 94.9\% &  96.2\%\\
		& SNR = -3dB & 91.1\% & 94.4\% &  94.7\%\\
		\hline \multirow{2}{*}{Biased} & Error-free & 95.1\% & $-$ & $-$\\
		& \text {Fixed Noise} & 92.1\% & $-$ & $-$\\
		\hline
	\end{tabular}
\end{table}

Experimental results are illustrated in Figs. \ref{fig: result2} and \ref{fig: result3}.
From Figs. \ref{fig: result2a} and \ref{fig: result2c}, it is clear that although under a high SNR, \airfedmodel still cannot converge in the presence of complicated neural networks, let alone a lower SNR. A further increase in the SNR, up to 30dB, will not make it converge. This is because in neural networks, the model parameter should be highly precise, which is of very low tolerance in perturbation. However, even in very low SNR regime, the convergence of \airfedavgm and \airfedavgs is guaranteed, which confirms our earlier reasoning that transmitting local model parameter via AirComp is not a good choice.
Similar to the convex case, the higher the SNR is, the better the performance of \airfedavg will be for all cases.
In addition, \airfedavgm converges faster with $E=5$ local updates than \airfedavgs to achieve linear speedup, as analyzed in \textbf{Remark \ref{rem: airfedavgm E}}. But with a low SNR, using the same learning rate as \airfedavgm will degrade the learning performance. As analyzed in \textbf{Theorem \ref{thm: case3unbiasednonconvex}}, the initial learning rate should be smaller. After the tuning on the learning rate, the highest achievable accuracy on non-IID MNIST dataset of different cases is summarized in Table \ref{tab: table4}.
From Table \ref{tab: table4}, more local updates lead to better learning performance under the same precoding factor design, which may be caused by a better stationary point for the non-convex objective function, but requiring smaller learning rate.
The experiments concerning \airfedavgs with biased MAE are only demonstrated for validating the convergence property.
As depicted in Fig. \ref{fig: result4}, the training performance is degraded in fading channel compared with that of AWGN channel. Although channel inversion is exploited to compensate for the fading channel, it still has huge influence reflected by the precoding factor. Last but not least, with a high SNR, \airfedavgm has better performance and faster convergence speed than \airfedavgs. However, it is clear that \airfedavgs is more robust to the fading channel even in a very low SNR due to the impact mechanism of the noise discussed in \textbf{Remark \ref{rem: 13comparison}}.

\subsection{Non-Convex Case II: CNN on CIFAR-10 Dataset}
Finally we validate our theorems and remarks by considering image classification on the CIFAR-10 dataset \cite{cifar}, which consists of 50,000 training images and 10,000 images for testing with $32\times32\times3 = 3072$ features. The training dataset is uniformly distributed over $N = 50$ edge devices. We consider non-IID data distribution, where the original training dataset is first sorted by labels, and then randomly assigned to all edge devices. Specially, the situation of extreme non-IID data is considered in this experiments, where each edge device can access only one shard of label.

For the learning model, we consider the non-convex CNN model and the architecture of this model is listed in the following table. The total number of model parameters is $d=62006$ and the number of local updates for \airfedavg is set as $E=5$ unless specified otherwise. The local batch size is 50 and the learning rate is initially set as $\eta^0 = 0.05$ and decays with $ \eta^t = \frac{\eta^0}{1+0.001\times t} $ in each communication round and the SNR is selected from the set $\{-3\text{dB},5\text{dB}\}$.
\begin{table}[h]	
	\caption{CNN Architecture for CIFAR-10}\label{tab: CNN}
	\centering
	\begin{tabular}{cc}
		\hline
		Layer Type & Size \\ 
		\hline
		Convolution + ReLU & 5$ \times $5$\times$6\\
		Max Pooling & 2$\times$2\\
		Convolution + ReLU & 5$ \times $5$\times$16\\
		Max Pooling & 2$\times$2\\
		Fully Connected + ReLU&400$\times$120\\
		Fully Connected + ReLU&120$\times$84\\
		Fully Connected&84$\times$10\\\hline
	\end{tabular}
\end{table}

We consider the loss and test accuracy over $T=1500$ communication rounds.
Experimental results are illustrated in Fig. \ref{fig: result5}. 
For extreme non-IID case of CIFAR-10 dataset, where each edge device can only access one category of data, the model aggregation MSE is not dominant any more, but the non-IID data degrading the learning performance.
In this case, as analyzed in \textbf{Remark \ref{rem: hetero}}, \airfedavgm suffers from highly non-IID data, and the linear speedup advantage over \airfedavgs is no longer obvious. In addition, greater receiver noise for model aggregation may have the potential to help the vanilla SGD method escape from saddle points \cite{chi2017how,zhang2022turning}.

\section{Concluding Remarks}\label{sec: conclusions}
In this tutorial, we investigated the first-order optimization algorithm \airfedavg, where MAE is introduced in each communication round due to the receiver noise. We first adopted channel inversion scheme, and provided convergence analysis for \airfedavgm algorithm and obtained the results for strongly convex objectives and non-convex ones with diminishing learning rate. Based on these results, we provided mathematical convergence guarantees for the optimality gap of the strongly convex case and the error bound of the non-convex case, and derived the convergence rate for both cases. We further extended the algorithm to \airfedavgs and \airfedmodel. Last but not least, we discussed the communication schemes in AirFL for practical implementations and extended our analysis to biased MAE. Our simulations were consistent with the theoretical analysis for the impact of noise and non-IID data, and convergence rate.

Several insightful conclusions can be obtained as follows.
\begin{itemize}
	\item Firstly, the convergence of the optimality gap of \airfedavgm in the strongly convex case and error bound in the non-convex case hinge on the model aggregation MSE, which is dominated by the noise power and denoising factor. In addition, the MSE should satisfy a certain condition to guarantee the convergence. For strongly convex objective function, the convergence rate of \airfedavgm is $\mathcal{O}\left(\frac{1}{NET}\right)$ with diminishing learning rate while that of non-convex case is $\mathcal{O}\left(\frac{1}{\sqrt{NET}}\right)$ with designed constant learning rate, which achieves linear speedup in terms of the number of local updates and the number of edge devices.
	\item Secondly, it may not be a good choice to transmit local model in \airfedmodel, which may cause divergence of the algorithm unless the MSE approaches to zero.
	\item Thirdly, the convergence of the optimality gap of \airfedavgs in strongly convex case and error bound in non-convex hinge on the model aggregation MSE as well, but the requirements on the MSE is not stringent. Under the same precoding and denoising factor design as \airfedavgm, \airfedavgs witnesses convergence rate $\mathcal{O}\left(\frac{1}{NT}\right)$ in strongly convex case and $\mathcal{O}\left(\frac{1}{\sqrt{NT}}\right)$ in non-convex case, respectively.
	\item In terms of the comparison between \airfedavgm and \airfedavgs, \airfedavgm converges faster and is more communication-efficient, but is more sensitive to the receiver noise and non-IID data. At the cost of convergence speed, \airfedavgs is observed to be more robust to receiver noise. There is an interesting trade-off between convergence rate and algorithm robustness. Besides, the consensus for \airfedavgm and \airfedavgs is that the lower SNR it is, the worse the learning performance will be. 
	\item In the end, when the MAE is biased, the convergence bound is dominated by both MSE and the squared norm of bias for \airfedavgm, and the squared norm of bias for \airfedavgs.
\end{itemize}

There are many open directions to extend this tutorial. For instance, different federated optimization algorithms and network architectures can be applied in AirFL for further analysis. More practical scenarios such as multi-antennas system, RIS and UAV assisted AirFL, and FL over satellite networks are promising directions to be investigated.

\appendices
\allowdisplaybreaks
\section{Convergence Analysis for MAE in \airfedavgm}\label{sec: case3}
This section summarizes the method of convergence analysis for MAE in \airfedavgm. In terms of perturbation in GD and SGD algorithms, \cite{Friedlander2012} is served as a guidance to this paper
\subsection{Preliminaries}
We denote $\bm d_n^t=\frac{1}{E}\sum\limits_{e=0}^{E-1}\bm g_n^{(t,e)}$ and $\bm h_n^t=\frac{1}{E}\sum\limits_{e=0}^{E-1}\nabla  F_n(\bm\theta_n^{(t,e)})$. It is clear that $\mathbb{E}[\bm d_n^t-\bm h_n^t]=0,\forall n$ and $\mathbb{E}[\left<\bm d_i^t-\bm h_i^t,\bm d_j^t-\bm h_j^t\right>]=0,\forall i\neq j$.

\begin{lem}\label{lem: Lemma 0}
	Suppose  $ \left\{A_{k}\right\}_{k = 1}^{T} $  is a sequence of random matrices and  $ \mathbb{E}\left[A_{k} \mid A_{k-1}, A_{k-2}, \ldots, A_{1}\right] = \mathbf{0}, \forall k $. Then we have
	\begin{align}
		\mathbb{E}\left[\left\|\sum_{k  = 1}^{T} A_{k}\right\|_{F}^{2}\right]  = \sum_{k = 1}^{T} \mathbb{E}\left[\left\|A_{k}\right\|_{F}^{2}\right].
	\end{align}
\end{lem}

\begin{lem}\label{lem: Lemma 1}
	Let \textbf{Assumption} \textbf{\ref{ass: gradient variance}} hold and we have
	\begin{align}
		\mathbb{E}\left[\left \|\sum_{n=1}^N p_n \bm d_n^t \right \|_2^2\right]\leq\frac{2}{E}\sum_{n=1}^{N}p_n^2\sigma_{n}^2+2\mathbb{E}\left[\left \|\sum_{n=1}^N p_n \bm h_n^t \right \|_2^2\right].
	\end{align}
\end{lem}

\begin{proof}
	By applying the fact that $\|\bm a+ \bm b\|_2^2\leq2\|\bm a\|_2^2+2\|\bm b\|_2^2$, we have
	\begin{align}\nonumber
		&\mathbb{E}\left[\left \|\sum_{n=1}^N p_n \bm d_n^t \right \|_2^2\right]\\\nonumber
		\leq & 2\mathbb{E}\left[\left \|\sum_{n=1}^N p_n \left(\bm d_n^t-\bm h_n^t\right) \right \|_2^2\right]+2\mathbb{E}\left[\left \|\sum_{n=1}^N p_n \bm h_n^t \right \|_2^2\right]\\\nonumber
		=&2\sum_{n=1}^{N}p_n^2 \mathbb{E}\left[\left \| \bm d_n^t-\bm h_n^t \right \|_2^2\right]+2\mathbb{E}\left[\left \|\sum_{n=1}^N p_n \bm h_n^t \right \|_2^2\right]\\\nonumber
		=& \frac{2}{E^2}\sum_{n=1}^{N}p_n^2\sum_{e=0}^{E-1} \mathbb{E}\left[\left \| \bm g_n^{(t,e)}-\nabla F_n(\bm\theta_{n}^{(t,e)}) \right \|_2^2\right]\\\nonumber
		+&2\mathbb{E}\left[\left \|\sum_{n=1}^N p_n \bm h_n^t \right \|_2^2\right]\\
		\leq&\frac{2}{E}\sum_{n=1}^{N}p_n^2\sigma_{n}^2+2\mathbb{E}\left[\left \|\sum_{n=1}^N p_n \bm h_n^t \right \|_2^2\right].
	\end{align}
where the second equality follows from \textbf{Lemma \ref{lem: Lemma 0}}, and the last inequality holds by \textbf{Assumption \ref{ass: gradient variance}}.
\end{proof}

By applying \textbf{Assumption} \textbf{\ref{ass: smooth}} for global loss function, we obtain
\begin{align}\label{eq: primalhat0}\nonumber
	F(\hat{\bm{\theta}}^{t+1}) - F(\hat{\bm{\theta}}^{t})
&	\leq  \left<\boldsymbol{\varepsilon}^{t}-\eta^tE\sum_{n=1}^N p_n \bm d_n^t, \nabla  F(\hat{\bm{\theta}}^{t})\right>\\
&	+\frac{L}{2}\left \|\bm \varepsilon^{t}- \eta^tE\sum_{n=1}^N p_n \bm d_n^t \right \|_2^2.
\end{align}
Taking an expectation on the mini-batch stochastic gradient for both sides of \eqref{eq: primalhat0}, we obtain
\begin{align}\nonumber\label{eq: expouter}
	&\mathbb{E}\left[F(\hat{\bm{\theta}}^{t+1}) - F(\hat{\bm{\theta}}^{t})\right]\\\nonumber
	\leq& \mathbb{E} \left[\left <\bm \varepsilon^{t}-\eta^tE\sum_{n=1}^N p_n \bm d_n^t, \nabla  F(\hat{\bm{\theta}}^{t})\right>\right]\\\nonumber
	+&\frac{L}{2}\mathbb{E}\left[\left \|\bm \varepsilon^{t}-\eta^tE\sum_{n=1}^N p_n \bm d_n^t \right \|_2^2\right]\\\nonumber
	=&\left[\nabla  F(\hat{\bm{\theta}}^{t})\right]^{\sf T} \mathbb{E}\left[\bm \varepsilon^{t}\right]-E\eta^t \left [\nabla  F(\hat{\bm{\theta}}^{t})\right]^{\sf T}\mathbb{E}\left[ \sum_{n=1}^N p_n \bm d_n^t \right]\\\nonumber
	+&\frac{L}{2}\mathbb{E}\left[\left\| \bm \varepsilon^{t}\right \| _2^2 \right]
	+\frac{L(\eta^t)^2E^2}{2} \mathbb{E}\left[\left \|\sum_{n=1}^N p_n \bm d_n^t \right \|_2^2\right]\\
	- & L\eta^tE\mathbb{E}\left[ (\bm \varepsilon^{t})^{\sf T}\sum_{n=1}^N p_n \bm d_n^t \right].
\end{align}

\subsection{Unbiased MAE}\label{sec: case3unbiased}
Obviously we have $\mathbb{E}\left[ \bm \varepsilon^t \right]=0$. We suppose $\mathbb{E}\left[\left \| \bm \varepsilon^t\right \|_2^2 \right]  = \frac{\sigma_{w}^2}{\beta^t} \neq 0$, then \eqref{eq: expouter} can be reformulated as
\begin{align}\nonumber\label{eq: unbiased}
	&\mathbb{E}\left[F(\hat{\bm \theta}^{t+1}) - F(\hat{\bm \theta}^{t})\right]\\\nonumber
	=&-\frac{\eta^tE }{2}\left \|\nabla  F(\hat{\bm \theta}^{t})\right \|_2^2 -\frac{\eta^tE }{2}\mathbb{E}\left[\left \|\sum_{n=1}^N p_n \bm h_n^t \right \|_2^2\right] \\
	+& \frac{\eta^tE }{2}\mathbb{E}\left[\left \|\sum_{n=1}^N p_n \bm h_n^t -\nabla  F(\hat{\bm \theta}^{t}) \right \|_2^2\right] \\\nonumber
	+& \frac{L(\eta^t)^2E^2}{2}\mathbb{E}\left[\left \|\sum_{n=1}^N p_n \bm d_n^t \right \|_2^2\right]+\frac{L}{2}\frac{\sigma_{w}^2}{\beta^t},
\end{align}
where the equality holds for the fact that $\mathbb{E}[\bm d_n^t-\bm h_n^t]=0,\forall n$, and that $2\left<\bm a,\bm b\right> = \|\bm a\|_2^2+\|\bm b\|_2^2-\|\bm a-\bm b\|_2^2$. 

Using \textbf{Lemma \ref{lem: Lemma 1}}, we have
\begin{flalign}\nonumber\label{eq: unbiased1}
	&\mathbb{E}\left[F(\hat{\bm \theta}^{t+1}) - F(\hat{\bm \theta}^{t})\right]\\\nonumber
	\leq&-\frac{\eta^tE }{2}\left \|\nabla  F(\hat{\bm \theta}^{t})\right \|_2^2 \\\nonumber
	+&\frac{\eta^tE }{2}(2\eta^tEL-1)\mathbb{E}\left[\left \|\sum_{n=1}^N p_n \bm h_n^t \right \|_2^2\right]+\frac{L}{2}\frac{\sigma_{w}^2}{\beta^t}\\\nonumber 
	+&\frac{\eta^tE }{2}\mathbb{E}\left[\left \|\sum_{n=1}^N p_n \bm h_n^t -\nabla  F(\hat{\bm \theta}^{t}) \right \|_2^2\right] +(\eta^t)^2LE\sum_{n=1}^N p_n^2 \sigma_n^2\\\nonumber
	\leq& -\frac{\eta^tE }{2}\left \|\nabla  F(\hat{\bm \theta}^{t})\right \|_2^2+(\eta^t)^2LE\sum_{n=1}^N p_n^2 \sigma_n^2+\frac{L}{2}\frac{\sigma_{w}^2}{\beta^t}\\
	+&\frac{\eta^tE }{2}\underbrace{\mathbb{E}\left[\left \|\sum_{n=1}^N p_n \bm h_n^t -\nabla  F(\hat{\bm \theta}^{t}) \right \|_2^2\right]}_{T_1},
\end{flalign}
where the second inequality holds by $\eta^t\leq\frac{1}{2LE}$.

To give an upper bound of $T_1$, since $\hat{\bm \theta}_n^{(t,0)} = \hat{\bm \theta}^{t}$, we have
\begin{align}\nonumber\label{eq: T1}
	T_1&=\mathbb{E}\left[\left \|\sum_{n=1}^N p_n \left[\bm h_n^t -\nabla  F_n(\hat{\bm \theta}_n^{(t,0)})\right] \right \|_2^2\right]\\\nonumber
	&\leq \sum_{n=1}^{N}p_n\mathbb{E}\left[\left \|\bm h_n^t-\nabla  F_n(\hat{\bm \theta}_n^{(t,0)})\right \|_2^2\right]\\\nonumber
	&=	\sum_{n=1}^{N}p_n\mathbb{E}\left[\left \|\nabla F_n(\hat{\bm \theta}_n^{(t,0)})- \frac{1}{E}\sum\limits_{e=0}^{E-1}\nabla  F_n(\hat{\bm \theta}_n^{(t,e)})\right \|_2^2\right]\\\nonumber
	&= \sum_{n=1}^{N}p_n\mathbb{E}\left[\left \|\frac{1}{E}\sum\limits_{e=0}^{E-1}\left[ \nabla F_n(\hat{\bm \theta}_n^{(t,0)})- \nabla  F_n(\hat{\bm \theta}_n^{(t,e)})\right]\right \|_2^2\right]\\\nonumber
	&\leq \frac{1}{E} \sum_{n=1}^{N}p_n \sum\limits_{e=0}^{E-1}\mathbb{E}\left[\left \|\nabla F_n(\hat{\bm \theta}_n^{(t,0)})- \nabla  F_n(\hat{\bm \theta}_n^{(t,e)})\right \|_2^2\right]\\
	&\leq  \sum_{n=1}^{N}p_n \left\{\frac{L^2}{E}\sum\limits_{e=0}^{E-1}\underbrace{\mathbb{E}\left[\left \|\hat{\bm \theta}_n^{(t,0)}- \hat{\bm \theta}_n^{(t,e)}\right \|_2^2\right]}_{T_2}\right\},
\end{align}
where the first and the second inequality holds by Jensen's Inequality, which is $\|\sum_{n=1}^{N}p_n\bm x_n\|_2^2\leq\sum_{n=1}^{N}p_n\|\bm x_n\|_2^2$, while the last inequality follows \textbf{Assumption} \textbf{\ref{ass: smooth}}.

Next, we give an upper bound of the term $T_2$. From \eqref{eq: definition gradient} and the fact that $\|\bm a+ \bm b\|_2^2\leq2\|\bm a\|_2^2+2\|\bm b\|_2^2$, we have
\begin{align}\nonumber
	T_2&=(\eta^t)^2 \mathbb{E}\left[\left \|\sum_{k=0}^{e-1}\bm g_n^{(t,k)}\right \|_2^2\right]\\\nonumber
	&\leq 2(\eta^t)^2\mathbb{E}\left[\left \|\sum_{k=0}^{e-1}\left[ \bm g_n^{(t,k)}-\nabla F_n(\hat{\bm \theta}_n^{(t,k)}) \right] \right \|_2^2\right]\\\nonumber 
	&+2(\eta^t)^2\mathbb{E}\left[\left \|\sum_{k=0}^{e-1}\nabla F_n(\hat{\bm \theta}_n^{(t,k)})\right \|_2^2\right]\\\nonumber
	&=  2(\eta^t)^2\sum_{k=0}^{e-1}\mathbb{E}\left[\left \| \bm g_n^{(t,k)}-\nabla F_n(\hat{\bm \theta}_n^{(t,k)})  \right \|_2^2\right]\\\nonumber
	&+2(\eta^t)^2\mathbb{E}\left[\left \|\sum_{k=0}^{e-1}\nabla F_n(\hat{\bm \theta}_n^{(t,k)})\right \|_2^2\right]\\\nonumber
	&\leq  2e\sigma_{n}^2(\eta^t)^2+ 2(\eta^t)^2e\sum_{k=0}^{e-1}\mathbb{E}\left[\left \|\nabla F_n(\hat{\bm \theta}_n^{(t,k)})\right \|_2^2\right]\\\nonumber
	&\leq  2e\sigma_{n}^2(\eta^t)^2\\\nonumber
	&+ 2(\eta^t)^2e\sum_{k=0}^{E-1}\left\{2\mathbb{E}\left[\left \|\nabla F_n(\hat{\bm \theta}_n^{(t,k)})-\nabla F_n(\hat{\bm \theta}_n^{(t,0)})\right \|_2^2\right]\right.\\\nonumber
	\phantom{=\;\;}&+\left. 2\mathbb{E}\left[\left \|\nabla F_n(\hat{\bm \theta}_n^{(t,0)})\right \|_2^2\right]\right\}\\\nonumber
	& \leq 2e\sigma_{n}^2(\eta^t)^2+ 4(\eta^t)^2e\sum_{k=0}^{E-1}\left\{L^2\mathbb{E}\left[\left \|\hat{\bm \theta}_n^{(t,k)}-\hat{\bm \theta}_n^{(t,0)}\right \|_2^2\right]\right.\\
	\phantom{=\;\;}&+\left. \mathbb{E}\left[\left \|\nabla F_n(\hat{\bm \theta}_n^{(t,0)})\right \|_2^2\right]\right\},
\end{align}
where the second inequality follows Jensen's Inequality and \textbf{Assumption} \textbf{\ref{ass: gradient variance}}, and the last inequality holds by \textbf{Assumption} \textbf{\ref{ass: smooth}}. In this case, we have
\begin{align}\nonumber
	\frac{L^2}{E}\sum\limits_{e=0}^{E-1}T_2
	&\leq  \frac{L^2}{E}\sum\limits_{e=0}^{E-1}2e\sigma_{n}^2(\eta^t)^2\\\nonumber
	&+ \frac{L^2}{E}\sum\limits_{e=0}^{E-1} 4(\eta^t)^2e\sum_{k=0}^{E-1}\left\{  L^2\mathbb{E}\left[\left \|\hat{\bm \theta}_n^{(t,k)}-\hat{\bm \theta}_n^{(t,0)}\right \|_2^2\right]\right.\\\nonumber
	\phantom{=\;\;}&+\left.\mathbb{E}\left[\left \|\nabla F_n(\hat{\bm \theta}_n^{(t,0)})\right \|_2^2\right]\right\}\\\nonumber
	&\leq L^2\sigma_{n}^2(\eta^t)^2(E-1)+H^t\frac{L^2}{E}\sum\limits_{e=0}^{E-1}T_2\\
	&+H^t\mathbb{E}\left[\left \|\nabla F_n(\hat{\bm \theta}_n^{(t,0)})\right \|_2^2\right],
\end{align}
where $H^t=2L^2(\eta^t)^2E(E-1)<\frac{E-1}{2E}<1$. After rearranging, we obtain
\begin{align}\label{eq: T2}
	\frac{L^2}{E}\sum\limits_{e=0}^{E-1}T_2
	\leq \frac{L^2\sigma_{n}^2(\eta^t)^2(E-1)}{1-H^t} +\frac{H^t\mathbb{E}\left[\left \|\nabla F_n(\hat{\bm \theta}_n^{(t,0)})\right \|_2^2\right]}{1-H^t}.
\end{align}
By substituting \eqref{eq: T2} into \eqref{eq: T1}, we have
\begin{flalign}\nonumber\label{eq: T11}
	T_1
	&\leq\sum_{n=1}^{N}p_n\frac{L^2\sigma_{n}^2(\eta^t)^2(E-1)}{1-H^t}\\\nonumber
	&+\frac{H^t}{1-H^t}\sum_{n=1}^{N}p_n \mathbb{E}\left[\left \|\nabla F_n(\hat{\bm \theta}_n^{(t,0)})\right \|_2^2\right]\\\nonumber
	&\leq\frac{L^2(\eta^t)^2(E-1)}{1-H^t}\sum_{n=1}^{N}p_n\sigma_{n}^2\\
	&+\frac{H^t}{1-H^t}\left[\beta_1\left \|\nabla F(\hat{\bm \theta}^{t})\right \|_2^2+\beta_2\right].
\end{flalign}
By substituting \eqref{eq: T11} into \eqref{eq: unbiased1}, we have
\begin{align}\label{eq: expectation4}\nonumber
	&\mathbb{E}\left[F(\hat{\bm \theta}^{t+1}) - F(\hat{\bm \theta}^{t})\right]\\\nonumber
	\leq&-\frac{\eta^tE }{2}\left(1-\frac{H^t\beta_1}{1-H^t}\right)\left \|\nabla  F(\hat{\bm \theta}^{t})\right \|_2^2\\\nonumber
	+&\frac{(\eta^t)^3L^2E(E-1) }{2(1-H^t)} \sum_{n=1}^{N}p_n\sigma_{n}^2+(\eta^t)^2LE\sum_{n=1}^N p_n^2 \sigma_n^2\\
	+&\frac{\eta^tH^tE\beta_2}{2(1-H^t)}+\frac{L}{2}\frac{\sigma_{w}^2}{\beta^t}.
\end{align}
If $H^t\leq\frac{1}{2\beta_1+1}$, then it follows that $1-\frac{H^t\beta_1}{1-H^t}\geq\frac{1}{2}$, which requires $\eta^t\leq\frac{1}{L\sqrt{2E(E-1)(2\beta_1+1)}}$. Then \eqref{eq: expectation4} can be simplified as 
\begin{align}\label{eq: expectation5}\nonumber
	&\mathbb{E}\left[F(\hat{\bm \theta}^{t+1}) - F(\hat{\bm \theta}^{t})\right]\\
	\leq&-\frac{\eta^tE}{4}\mathbb{E}\left[\left \|\nabla  F(\hat{\bm \theta}^{t})\right \|_2^2\right]
	+(\eta^t)^3C_1 + (\eta^t)^2C_2
	+\frac{C_3}{\beta^t},
\end{align}
where $C_1=\frac{L^2E(E-1)(2\beta_1+1) }{4\beta_1}\left[\sum_{n=1}^{N}p_n\sigma_{n}^2+2E\beta_2\right] $, $C_2=LE\sum_{n=1}^N p_n^2 \sigma_n^2$,  and $C_3=\frac{L}{2}\sigma_{w}^2$.

\subsubsection{Strongly Convex Case}
When \textbf{Assumption} \textbf{\ref{ass: strong convex}} holds, i.e., the strongly convex case, we have the following results:
\begin{align}\label{eq: expectation6}\nonumber
	&\mathbb{E}\left[F(\hat{\bm \theta}^{t+1}) - F(\hat{\bm \theta}^{t})\right]\\
	\leq&-\frac{\mu\eta^tE}{2}\mathbb{E}\left[F(\hat{\bm \theta}^{t}) - F^{\star}\right]+(\eta^t)^3C_1+(\eta^t)^2C_2+\frac{C_3}{\beta^t}.
\end{align}

By rearranging \eqref{eq: expectation6}, we obtain 
\begin{align}\label{eq: expectation7}\nonumber
	\mathbb{E}\left[F(\hat{\bm \theta}^{t+1}) - F^{\star}\right]
	&\leq\left(1-\frac{\mu\eta^tE}{2}\right)\mathbb{E}\left[F(\hat{\bm \theta}^{t}) - F^{\star}\right]\\
	&+(\eta^t)^3C_1	+(\eta^t)^2C_2
	+\frac{C_3}{\beta^t}.
\end{align}

Applying \eqref{eq: expectation7} recursively for $t$ from $0$ to $T-1$, we
can obtain the cumulative upper bound after $T$ communication rounds shown as follows
\begin{align}\label{eq: expectation8}\nonumber
	&\mathbb{E}\left[F(\hat{\bm \theta}^{T})\right] - F^{\star}\\\nonumber
	\leq &\prod _{t=0}^{T-1}\rho^t\left[F(\bm \theta^{0}) - F^{\star}\right]\\
	+& \sum_{t=0}^{T-1}\left\{\left[(\eta^t)^3C_1
	+(\eta^t)^2C_2+\frac{C_3}{\beta^t}\right]\prod _{i=t+1}^{T-1}\rho^i \right\},
\end{align}
where $\rho^t = 1-\frac{\mu\eta^tE}{2}\geq0$. At last, the constraint on the learning rate satisfies 
\begin{align}\nonumber
	0<\eta^t&\leq \min\left\{\frac{1}{L\sqrt{2E(E-1)(2\beta_1+1)}},\frac{1}{2LE},\frac{2}{\mu E}\right\}\\
	&=\min\left\{\frac{1}{L\sqrt{2E(E-1)(2\beta_1+1)}},\frac{1}{2LE}\right\}.
\end{align}

\subsubsection{Non-Convex Case} When \textbf{Assumption} \textbf{\ref{ass: strong convex}} does not hold, i.e., the non-convex case, we have the corresponding results.

By rearranging \eqref{eq: expectation5}, we obtain 
\begin{align}\label{eq: expectation9}\nonumber
	\eta^t\mathbb{E}\left[ \left \|\nabla  F(\hat{\bm \theta}^{t})\right \|_2^2\right]
	&\leq \frac{4\mathbb{E}\left[F(\hat{\bm \theta}^{t}) - F(\hat{\bm \theta}^{t+1})\right]}{E} \\
	&+ \frac{4C_1}{E}  (\eta^t)^3+ \frac{4C_2}{E}  (\eta^t)^2 +  \frac{4C_3}{E\beta^t}.
\end{align}
Then, summing both sides of \eqref{eq: expectation9} for $t$ from $0$ to $T-1$, we
have
\begin{align}\label{eq: expectation100}\nonumber
	&\sum_{t = 0}^{T-1} \eta^t\mathbb{E}\left[\left \|\nabla  F(\hat{\bm \theta}^{t})\right \|_2^2\right]\\\nonumber
	\leq& \frac{4\left[F(\bm{\theta}^{(0,0)}) - F^{\inf}\right] }{E} + \frac{4C_1}{E} \sum_{t = 0}^{T-1} (\eta^t)^3+ \frac{4C_2}{E} \sum_{t = 0}^{T-1} (\eta^t)^2\\
	+&\frac{4C_3}{E} \sum_{t = 0}^{T-1}\frac{1}{\beta^t} .
\end{align}
Dividing \eqref{eq: expectation100} by $\sum_{t = 0}^{T-1} \eta^t$, we complete the proof. 

\subsection{Biased MAE}\label{sec: case3biased}
Obviously we have $\mathbb{E}\left[ \bm \varepsilon^t \right]\neq 0, \mathbb{E}\left[\left \| \bm \varepsilon^t\right \|_2^2 \right]  \neq 0$.
By applying the relationship between arithmetic and geometric mean to \eqref{eq: expouter}, i.e., $|\bm x^{\sf T} \bm y| \leq \frac{p\|\bm x\|_2^2}{2}+\frac{\|\bm y\|_2^2}{2p}$, it yields
\begin{align}\nonumber\label{eq: outererror}
	&\mathbb{E}\left[F(\hat{\bm{\theta}}^{t+1}) - F(\hat{\bm{\theta}}^{t})\right]\\\nonumber
	\leq& \frac{E\eta^t\left\| \nabla  F(\hat{\bm{\theta}}^{t}) \right\|_2^2}{4}+\frac{\left \| \mathbb{E}\left[\bm \varepsilon^{t}\right] \right \|_2^2}{E\eta^t}\\\nonumber
	-&E\eta^t \left [\nabla  F(\hat{\bm{\theta}}^{t})\right]^{\sf T} \mathbb{E}\left[\sum_{n=1}^N p_n \bm h_n^t\right] +\frac{L}{2}\mathbb{E}\left[\left\| \bm \varepsilon^{t}\right \| _2^2 \right]\\\nonumber
	+ & \frac{L(\eta^t)^2E^2}{2} \mathbb{E}\left[ \sum_{n=1}^N p_n \bm d_n^t \right]\\
	+ & L\eta^tE\left\{ 2L\mathbb{E}\left[\left\| \bm \varepsilon^{t}\right \| _2^2 \right]
	+\frac{\mathbb{E}\left[\left\| \sum_{n=1}^N p_n \bm d_n^t \right \| _2^2 \right]}{8L} \right\}.
\end{align}

Combining \eqref{eq: outererror} with \textbf{Lemma \ref{lem: Lemma 1}} and \eqref{eq: T11}, we obtain
\begin{align}\label{eq: expectation12}\nonumber
	&\mathbb{E}\left[F(\hat{\bm{\theta}}^{t+1}) - F(\hat{\bm{\theta}}^{t})\right]\\\nonumber
	\leq&-\frac{\eta^tE }{4}\left(1-\frac{2H^t\beta_1}{1-H^t}\right)\left \|\nabla  F( \hat{\bm \theta}^{t})\right \|_2^2\\\nonumber
	+&(\eta^t)^3\frac{L^2E(E-1) }{2(1-H^t)} \sum_{n=1}^{N}p_n\sigma_{n}^2+(\eta^t)^2LE\sum_{n=1}^N p_n^2 \sigma_n^2\\\nonumber
	+&\eta^t\left[\frac{H^tE\beta_2}{2(1-H^t)}+\frac{1}{4}\sum_{n=1}^N p_n^2 \sigma_n^2+2L^2E\mathbb{E}\left[\left\| \bm \varepsilon^{t}\right \| _2^2 \right]\right]\\
	+&\frac{L}{2}\mathbb{E}\left[\left\| \bm \varepsilon^{t}\right \| _2^2 \right]+(\eta^t)^{-1}\frac{\left \| \mathbb{E}\left[\bm \varepsilon^{t}\right] \right \|_2^2}{E}.
\end{align}
If $H^t\leq\frac{1}{4\beta_1+1}$, then it follows that $1-\frac{2H^t\beta_1}{1-H^t}\geq\frac{1}{2}$, which requires $\eta^t\leq\frac{1}{L\sqrt{2E(E-1)(4\beta_1+1)}}$. Then \eqref{eq: expectation12} can be simplified as 
\begin{align}\label{eq: expectation13}\nonumber
	&\mathbb{E}\left[F(\hat{\bm{\theta}}^{t+1}) - F(\hat{\bm{\theta}}^{t})\right]\\\nonumber
	\leq&-\frac{\eta^tE}{8}\left \|\nabla  F(\hat{\bm \theta}^{t})\right \|_2^2\\\nonumber
	+&(\eta^t)^3\frac{L^2E(E-1)(4\beta_1+1) }{8\beta_1} \left[\sum_{n=1}^{N}p_n\sigma_{n}^2 + 2E\beta_2\right] \\\nonumber
	+&(\eta^t)^2LE\sum_{n=1}^N p_n^2 \sigma_n^2
	+\eta^t\left[\frac{1}{4}\sum_{n=1}^N p_n^2 \sigma_n^2+2L^2E\mathbb{E}\left[\left\| \bm \varepsilon^{t}\right \| _2^2 \right]\right]\\
	+&\frac{L}{2}\mathbb{E}\left[\left\| \bm \varepsilon^{t}\right \| _2^2 \right]+(\eta^t)^{-1}\frac{\left \| \mathbb{E}\left[\bm \varepsilon^{t}\right] \right \|_2^2}{E}.
\end{align}

Finally we can obtain \textbf{Theorems \ref{thm: case3biasedconvex}} and \textbf{\ref{thm: case3biasednonconvex}} respectively with the same method in \textbf{Appendix \ref{sec: case3unbiased}}.

\section{Convergence Analysis for MAE in \airfedavgs}\label{sec: case1}
This section summarizes the method of convergence analysis for MAE in \airfedavgs.
\subsection{Preliminaries}
First we have
\begin{align}
	\hat{\bm{\theta}}^{t+1} = \hat{\bm{\theta}}^{t}-\eta^t\hat{\bm y}^t.
\end{align}
By applying \textbf{Assumption} \textbf{\ref{ass: smooth}} for global loss function, we obtain
\begin{align}\label{eq: primal}
	F(\hat{\bm{\theta}}^{t+1}) - F(\hat{\bm{\theta}}^{t})\leq -\eta^t\left<\hat{\bm y}^t, \nabla  F(\hat{\bm{\theta}}^{t})\right>+\frac{L(\eta^t)^2}{2}\left \|\hat{\bm y}^t \right \|_2^2.
\end{align}
Taking an expectation on the mini-batch stochastic gradient for both sides of \eqref{eq: primal}, and denoting $\bar{\bm y}^t=\sum_{n=1}^{N}p_n \bm g_n^{(t,0)} $, we obtain
\begin{align}\nonumber\label{eq: expectation}
	&\mathbb{E}\left[F(\hat{\bm{\theta}}^{t+1}) - F(\hat{\bm{\theta}}^{t})\right]\\\nonumber
	\leq& -\eta^t \mathbb{E} \left[\left <\hat{\bm y}^t, \nabla  F(\hat{\bm{\theta}}^{t})\right>\right]+\frac{L(\eta^t)^2}{2}\mathbb{E}\left[\left \|\hat{\bm y}^t \right \|_2^2\right]\\\nonumber
	=&-\eta^t \left \|\nabla  F(\hat{\bm{\theta}}^t)\right \|_2^2 - \eta^t \left[\nabla  F(\hat{\bm{\theta}}^t)\right]^{\sf T} \mathbb{E}\left[\bm \varepsilon^{t}\right] \\\nonumber
	+ &\frac{L(\eta^t)^2}{2}\mathbb{E}\left[\left \|\bar{\bm y}^t \right \|_2^2\right] + L(\eta^t)^2 \left[\nabla  F(\hat{\bm{\theta}}^t)\right]^{\sf T} \mathbb{E}\left[\bm \varepsilon^{t}\right]\\ 
	+ &\frac{L(\eta^t)^2}{2}\mathbb{E}\left[\left \|\bm\varepsilon^t \right \|_2^2\right],
\end{align}
where the equality holds by \textbf{Assumption} \textbf{\ref{ass: gradient variance}} and the fact that $\nabla  F(\hat{\bm{\theta}}^t)=
\sum_{n=1}^{N} p_n\nabla  F_{n}(\hat{\bm \theta}^t ) $.

\subsection{Unbiased MAE}\label{sec: case1unbiased}
Obviously we have $\mathbb{E}\left[ \bm \varepsilon^t \right]=0$. We suppose $\mathbb{E}\left[\left \| \bm \varepsilon^t\right \|_2^2 \right]  = \sigma_{w}^2/\beta^t \neq 0$, then \eqref{eq: expectation} can be reformulated as
\begin{align}\nonumber\label{eq: expunbiased}
	 &\mathbb{E}\left[F(\hat{\bm{\theta}}^{t+1}) - F(\hat{\bm{\theta}}^{t})\right]\\
	\leq&-\eta^t \left \|\nabla  F(\hat{\bm{\theta}}^t)\right \|_2^2
	+ \frac{L(\eta^t)^2}{2}\mathbb{E}\left[\left \|\bar{\bm y}^t \right \|_2^2\right] 
	+ \frac{L(\eta^t)^2}{2}\frac{\sigma_{w}^2}{\beta^t},
\end{align}

Next, to further give a bound on $\mathbb{E}\left[\left \|\bar{\bm y}^t  \right \|_2^2\right]$, we only need to give a bound on $\mathbb{E}\left[\left \|\bar{\bm g}^t \right \|_2^2\right]$
\begin{align}\nonumber\label{eq: Ebargradient}
	\mathbb{E}\left[\left \|\bar{\bm y}^t  \right \|_2^2\right]&=\mathbb{E}\left[\left \| \sum_{n=1}^N p_n\bm g_n^t - \nabla  F(\hat{\bm \theta}^t) + \nabla  F(\hat{\bm \theta}^t)\right \|_2^2\right]\\\nonumber
	& = \mathbb{E}\left[\left \| \sum_{n=1}^N p_n\bm g_n^t - \nabla  F(\hat{\bm \theta}^t)\right \|_2^2\right] + \left \|\nabla  F(\hat{\bm \theta}^t)\right \|_2^2\\\nonumber
	& = \mathbb{E}\left[\left \| \sum_{n=1}^N p_n\left[\bm g_n^t - \nabla  F_n(\hat{\bm \theta}^t)\right]\right \|_2^2\right] + \left \|\nabla  F(\hat{\bm \theta}^t)\right \|_2^2\\\nonumber
	& = \sum_{n=1}^{N}p_n^2\mathbb{E}\left[\left \| \bm g_n^t - \nabla  F_n(\hat{\bm \theta}^t) \right \|_2^2\right] + \left \|\nabla  F(\hat{\bm \theta}^t)\right \|_2^2\\
	&\leq \sum_{n=1}^{N}p_n^2\sigma_{n}^2 + \left \|\nabla  F(\hat{\bm \theta}^t)\right \|_2^2,
\end{align}
where the second equality and the last inequality hold by \textbf{Assumption} \ref{ass: gradient variance}.

By substituting \eqref{eq: Ebargradient} into \eqref{eq: expunbiased}, we can obtain
\begin{align}\nonumber\label{eq: F-F}
	\mathbb{E}\left[F(\hat{\bm \theta}^{t+1}) - F(\hat{\bm \theta}^t)\right] &\leq -\eta^t\left(1-\frac{L\eta^t}{2}\right)\left \|\nabla  F(\hat{\bm \theta}^t)\right \|_2^2\\
	&+\frac{L(\eta^t)^2}{2}\left( \sum_{n=1}^{N}p_n^2\sigma_{n}^2 +\frac{ \sigma_{w}^2 }{\beta^t}\right).
\end{align}
\subsubsection{Strongly Convex Case}
When \textbf{Assumption} \textbf{\ref{ass: strong convex}} holds, i.e., the strongly convex case, we have the following results.

\begin{align}\nonumber\label{eq: F-Fconvex}
	\mathbb{E}\left[F(\hat{\bm \theta}^{t+1}) - F(\hat{\bm \theta}^t)\right]
	&\leq -2\mu\eta^t\left(1-\frac{L\eta^t}{2}\right)\left[F(\hat{\bm \theta}^t)-F^{\star}\right] \\
	 &+ \frac{L(\eta^t)^2}{2}\left( \sum_{n=1}^{N}p_n^2\sigma_{n}^2 + \frac{ \sigma_{w}^2 }{\beta^t} \right),
\end{align}
where the inequality holds by \textbf{Assumption} \textbf{\ref{ass: strong convex}}. By rearranging \eqref{eq: F-Fconvex}, we have
\begin{align}\nonumber
	&\mathbb{E}\left[F(\hat{\bm \theta}^{t+1}) - F^{\star}\right]\\\nonumber
	\leq& \left[1-2\mu\eta^t\left(1-\frac{L\eta^t}{2}\right)\right]\mathbb{E}\left[F(\hat{\bm \theta}^t) - F^{\star}\right] \\\nonumber
	+ &\frac{L(\eta^t)^2}{2}\left( \sum_{n=1}^{N}p_n^2\sigma_{n}^2 + \frac{ \sigma_{w}^2 }{\beta^t} \right)\\
	\leq& (1-\mu\eta^t)\mathbb{E}\left[F(\hat{\bm \theta}^t) - F^{\star}\right] + \frac{L(\eta^t)^2}{2}\left( \sum_{n=1}^{N}p_n^2\sigma_{n}^2 + \frac{ \sigma_{w}^2 }{\beta^t} \right),
\end{align}
where the last inequality holds by $\eta^t\leq\frac{1}{L}$.

Applying this recursively for $t$ from $0$ to $T-1$, and we can finish the proof.

\subsubsection{Non-Convex Case} When \textbf{Assumption} \textbf{\ref{ass: strong convex}} does not hold, i.e., the non-convex case, we have the corresponding results.

By rearranging \eqref{eq: F-F}, we have
\begin{align}\label{eq: case1nonconvex}
	\eta^t\mathbb{E}\left[ \left \|\nabla  F(\hat{\bm \theta}^{t})\right \|_2^2\right]\leq 2 \mathbb{E}\left[F(\hat{\bm \theta}^{t}) - F(\hat{\bm \theta}^{t+1})\right]+ 2 C (\eta^t)^2.
\end{align}
Then, by summing both sides of \eqref{eq: case1nonconvex} for $t$ from $0$ to $T-1$ and dividing \eqref{eq: case1nonconvex2} by $\sum_{t = 0}^{T-1} \eta^t$, we finally arrive at the weighted sum of the global gradient
\begin{align}\label{eq: case1nonconvex2}\nonumber
	\frac{\sum_{t = 0}^{T-1} \eta^t\mathbb{E}\left[\left \|\nabla  F(\hat{\bm \theta}^{t})\right \|_2^2\right]}{\sum_{t = 0}^{T-1} \eta^t}
	&\leq \frac{2 \left[F(\hat{\bm{\theta}}^{0}) - F^{\inf}\right]}{\sum_{t = 0}^{T-1} \eta^t}\\
	&+ \frac{2 C \sum_{t = 0}^{T-1} (\eta^t)^2}{\sum_{t = 0}^{T-1} \eta^t}.
\end{align}

\subsection{Biased MAE}\label{sec: case1biased}
In this case, we have $\mathbb{E}\left[ \bm \varepsilon^t \right]\neq 0, \mathbb{E}\left[\left \| \bm \varepsilon^t\right \|_2^2 \right]  \neq 0$.
By applying the relationship between arithmetic and geometric mean to \eqref{eq: expectation}, i.e., $|(\sqrt{\eta}\bm x)^{\sf T} (\sqrt{\eta}\bm y)| \leq \frac{\eta\|\bm x\|_2^2}{2}+\frac{\eta\|\bm y\|_2^2}{2}$ and combining with \eqref{eq: Ebargradient}, it can be reformulated as
\begin{align}\nonumber\label{eq: biasedexp}
	&\mathbb{E}\left[F(\hat{\bm{\theta}}^{t+1}) - F(\hat{\bm{\theta}}^{t})\right]\\\nonumber
	\leq&-\eta^t \left \|\nabla  F(\hat{\bm{\theta}}^t)\right \|_2^2 
	+ \frac{\eta^t}{2} \left \|\nabla  F(\hat{\bm{\theta}}^t)\right \|_2^2 + \frac{\eta^t}{2}\left\| \mathbb{E}\left[\bm \varepsilon^{t}\right]\right\|_2^2\\\nonumber
	+ &\frac{L(\eta^t)^2}{2} \left[\left \|\nabla  F(\hat{\bm{\theta}}^t)\right \|_2^2+\left\| \mathbb{E}\left[\bm \varepsilon^{t}\right]\right\|_2^2\right]+\frac{L(\eta^t)^2}{2}\mathbb{E}\left[\left \|\bm\varepsilon^t \right \|_2^2\right]\\\nonumber
	+ &\frac{L(\eta^t)^2}{2} \left[\sum_{n=1}^{N}p_n^2\sigma_{n}^2+ \left \|\nabla  F(\hat{\bm{\theta}}^t)\right \|_2^2\right]\\\nonumber
	= &-\frac{\eta^t}{2}\left(1-2L\eta^t\right) \left \|\nabla  F(\hat{\bm{\theta}}^t)\right \|_2^2+ \frac{\eta^t}{2} \left\| \mathbb{E}\left[\bm \varepsilon^{t}\right]\right\|_2^2\\\nonumber
	+& (\eta^t)^2 \frac{L}{2}\left\{ \mathbb{E}\left[\left \|\bm\varepsilon^t \right \|_2^2\right] + \left\| \mathbb{E}\left[\bm \varepsilon^{t}\right]\right\|_2^2 + \sum_{n=1}^{N}p_n^2\sigma_{n}^2 \right\}\\\nonumber
	\leq & -\frac{\eta^t}{4} \left \|\nabla  F(\hat{\bm{\theta}}^t)\right \|_2^2+ \frac{\eta^t}{2} \left\| \mathbb{E}\left[\bm \varepsilon^{t}\right]\right\|_2^2\\
	+& (\eta^t)^2 \frac{L}{2}\left\{ \mathbb{E}\left[\left \|\bm\varepsilon^t \right \|_2^2\right] + \left\| \mathbb{E}\left[\bm \varepsilon^{t}\right]\right\|_2^2 + \sum_{n=1}^{N}p_n^2\sigma_{n}^2 \right\},
\end{align}
where the second inequality holds by the constraint on the learning rate $\eta^t\leq\frac{1}{4L}$.

Finally we can obtain \textbf{Theorem \ref{thm: case1biasedconvex}} and \textbf{\ref{thm: case1biasednonconvex}} respectively with the same method in \textbf{Appendix \ref{sec: case1unbiased}}.

\section{Proof of Corollary \ref{cor: case3convex}}\label{sec: multiple}
For communication round $t\rightarrow t+1$, we have
\begin{align}\nonumber
	&\mathbb{E}\left[F(\hat{\bm \theta}^{t+1}) - F^{\star}\right]\\
	\leq&\left(1-\frac{\mu\eta^tE}{2}\right)\mathbb{E}\left[F(\hat{\bm \theta}^{t}) - F^{\star}\right]
	+(\eta^t)^3A	+(\eta^t)^2B.
\end{align}
where $A=\frac{L^2E(E-1)(2\beta_1+1) }{4\beta_1} \left[\sum_{n=1}^{N}p_n\sigma_{n}^2+2E\beta_2\right]$ and $B = LE\sum_{n=1}^N p_n^2 \sigma_n^2 + L\tilde{G}^2E/dN^2{\sf SNR}$.
We want to prove by induction that there exist $\tau$ and $\eta^t$, such that
\begin{align}\nonumber
	\boldsymbol{\Delta}^t\leq (\eta^t)^2\gamma A	+\eta^t\gamma B.
\end{align}
where $\boldsymbol{\Delta}^t = \mathbb{E}\left[F(\hat{\bm \theta}^{t})\right] - F^{\star}$.
Suppose that this holds for $t>0$. For round $t+1$, we have
\begin{align}\nonumber
	\boldsymbol{\Delta}^{t+1}\leq& \left(1-\frac{\mu\eta^tE}{2}\right)\boldsymbol{\Delta}^t +  (\eta^t)^3A	+(\eta^t)^2B\\\nonumber
	\leq&\left(1-\frac{\mu\eta^tE}{2}\right)\left[(\eta^t)^2\gamma A	+\eta^t\gamma B\right]+(\eta^t)^3A	+(\eta^t)^2B\\
	 = &\left[\gamma \left(1-\frac{\mu\eta^tE}{2}\right)+\eta^t\right]\left[(\eta^t)^2A	+\eta^tB\right].
\end{align}
We want to prove that 
\begin{equation}
	\left[\gamma \left(1-\frac{\mu\eta^tE}{2}\right)+\eta^t\right]\eta^t\leq\tau\eta^{t+1},
\end{equation}
\begin{equation}
	\left[\gamma \left(1-\frac{\mu\eta^tE}{2}\right)+\eta^t\right] (\eta^t)^2 \leq\tau(\eta^{t+1})^2.
\end{equation}

By setting
\begin{align}
	\left\{
	\begin{aligned}
		\eta^t &= \frac{6}{E\mu(t+\tau)}, \\
		\gamma &= \frac{6}{E\mu},
	\end{aligned}
	\right.
\end{align}
we can show that it is equivalent to proving
\begin{align}
	\left\{
	\begin{aligned}
		(t+\tau-2)(t+\tau-1)&\leq (t+\tau)^2, \\
		(t+\tau+1)^2(t+\tau-2)&\leq (t+\tau)^3,
	\end{aligned}
	\right.
\end{align}
which are obvious for $\tau  = \frac{3L}{\mu}$.

In this case, we have
\begin{align}\nonumber
	\boldsymbol{\Delta}^t&\leq \left[ \frac{6}{E\mu(t+\tau)}\right]^2\frac{6}{E\mu}A	+\frac{6}{E\mu(t+\tau)}\frac{6}{E\mu} B\\
	 &= \frac{216A}{E^3\mu^3(t+\tau)^2}+\frac{36B}{E^2\mu^2(t+\tau)}.
\end{align}

For $t = 0$, it follows that
\begin{align}
	\boldsymbol{\Delta}^0 \leq \frac{216A}{E^3\mu^3\tau^2}+\frac{36B}{E^2\mu^2\tau}.
\end{align}
Then the constraint on $E$ can be obtained by
\begin{equation}
	E\leq \frac{6(2\beta_1+1)\bar{\sigma}^2+12\beta_1(\sigma^2+\tilde{G}^2/dN^2{\sf SNR})}{\beta_1\mu (F(\hat{\bm \theta}^{0}) - F^{\star})-12\beta_2(2\beta_1+1)},
\end{equation}
where $\bar{\sigma}^2 = \sum_{n=1}^{N}p_n\sigma_{n}^2$ and $\sigma^2 = \sum_{n=1}^{N}p_n^2\sigma_{n}^2$.

The upper bound after $T$ communication rounds is then given by
\begin{align}\nonumber
	\boldsymbol{\Delta}^T &\leq \frac{216A}{E^3\mu^3(T+\tau)^2}+\frac{36B}{E^2\mu^2(T+\tau)}\\
	&=\mathcal{O}\left(\frac{A_1}{E^2\mu^3T^2}\right)+\mathcal{O}\left(\frac{B_1}{E\mu^2T}\right)
\end{align}
where $A_1 = L^2(E-1)(2+1/\beta_1)[\bar{\sigma}^2+E\beta_2]$, $B_1 = L\sigma^2+L/dN^2{\sf SNR}$.

\section{Proof of Corollary \ref{cor: case1convex}}\label{sec: single}
Similarly, we prove this by induction. Let $\boldsymbol{\delta}^t = \mathbb{E}\left[F(\hat{\bm \theta}^t)\right] - F^{\star}$, $C= \frac{L}{2}\left( \Sigma/N + G^2/dN^2{\sf SNR} \right)$, and define $u=\max\left\{\frac{\beta^2C}{\beta\mu-1},\boldsymbol{\delta}^0\tau\right\},\tau>0$. If the learning rate is set to be $\eta^t = \frac{\beta}{\tau+t}$ such that $\eta^{0} \leq \min \{ \frac{1}{\mu}, \frac{1}{L} \} = \frac{1}{L}$. Then by applying induction method, first it is clearly that $\boldsymbol{\delta}^0\leq\frac{u}{\tau}$. Second, assume that when $n=t$, we have
\begin{align}
	\boldsymbol{\delta}^t\leq \frac{u}{\tau+t}.
\end{align}
At last, when $n=t+1$,
\begin{align}\nonumber
	\boldsymbol{\delta}^{t+1}&\leq \left(1-\frac{\beta\mu}{\tau+t}\right)\frac{u^t}{\tau+t}+\frac{\beta^2C}{(\tau+t)^2}\\\nonumber
	& = \frac{\tau+t-1}{(\tau+t)^2}u - \frac{\beta\mu-1}{(\tau+t)^2}\left(u-\frac{\beta^2C}{\beta\mu-1}\right)\\\nonumber
	&\leq \frac{\tau+t-1}{(\tau+t)^2}u\\
	& \leq \frac{\tau+t-1}{(\tau+t)^2-1}u = \frac{u}{\tau+t+1},
\end{align}
where the second inequality holds by $\beta>\frac{1}{\mu}$ and $u=\max\left\{\frac{\beta^2C}{\beta\mu-1},\boldsymbol{\delta}^0\tau\right\}$.

Specifically, if we choose $\beta = \frac{2}{\mu}$ and $\tau = \frac{2L}{\mu}$, then $\eta^{t} = \frac{2}{\mu(\tau + t)}$ and
\begin{align}
	\mathbb{E}[F(\hat{\bm \theta}^t)] - F^{\star} \leq \frac{ \max\{4C, \mu^2 \tau \boldsymbol{\delta}^{0} \}}{\mu^2(\tau + t)}.
\end{align}

After $T$ total communication rounds, we have
\begin{align}\nonumber
	\mathbb{E}[F(\hat{\bm \theta}^T)] - F^{\star} &\leq \frac{ \max\{4C, \mu^2 \tau \boldsymbol{\delta}^{0} \}}{\mu^2(\tau + T)}\\
	 &= \mathcal{O}\left( \frac{\max\{4C, \mu^2 \tau \boldsymbol{\delta}^{0} \}}{\mu^2T}\right).
\end{align}

\bibliographystyle{IEEEtran}
\bibliography{IEEEabrv,reference0}
\end{document}